\documentclass[12pt]{article}
\usepackage{amsmath}
\usepackage[utf8]{inputenc}
\usepackage{graphicx}
\usepackage{natbib}
\usepackage{url} 
\usepackage{xr}
\usepackage{amssymb,amsthm}
 \usepackage{setspace}
\usepackage[plain,noend]{algorithm2e}
\usepackage{times}
\usepackage{bm}
\usepackage{caption,multirow,threeparttable,url}
\usepackage{subcaption}
\usepackage{color}
\usepackage{algorithmic}

\newtheorem{definition}{Definition}
\newtheorem{theorem}{Theorem}
\newtheorem{remark}{Remark}
\newtheorem{proposition}{Proposition}

\newtheorem{lemma}{Lemma}

\makeatletter
\renewcommand{\algocf@captiontext}[2]{\quad #1\algocf@typo. \AlCapFnt{}#2} 
\def\@algocf@capt@plain{top}
\renewcommand{\algocf@makecaption}[2]{%
  \addtolength{\hsize}{\algomargin}%
  \sbox\@tempboxa{\algocf@captiontext{#1}{#2}}%
  \ifdim\wd\@tempboxa >\hsize
    \hskip .5\algomargin%
    \parbox[t]{\hsize}{\algocf@captiontext{#1}{#2}}
  \else%
    \global\@minipagefalse%
    \hbox to\hsize{\box\@tempboxa}
  \fi%
  \addtolength{\hsize}{-\algomargin}%
}
\makeatother

\newcommand{\blind}{0}

\newcommand{\beginsupplement}{%
        \setcounter{table}{0}
        \renewcommand{\thetable}{S\arabic{table}}%
        \setcounter{figure}{0}
        \renewcommand{\thefigure}{S\arabic{figure}}%
     }

\begin{document}

\def\spacingset#1{\renewcommand{\baselinestretch}%
{#1}\small\normalsize} \spacingset{1.2}


\if0\blind
{
  \title{\bf Spectral clustering via adaptive layer aggregation for multi-layer networks\footnote{Huang and Weng contribute equally to this paper. Feng is the corresponding author (Email: yang.feng@nyu.edu).}}
  \author{Sihan Huang\hspace{.2cm}\\
    Department of Statistics, Columbia University\\
  Haolei Weng \\
    Department of Statistics and Probability, Michigan State University\\
    Yang Feng\\
    Department of Biostatistics, New York University}
  \maketitle
} \fi

\if1\blind
{
  \bigskip
  \bigskip
  \bigskip
  \begin{center}
    {\LARGE\bf Title}
\end{center}
  \medskip
} \fi

\bigskip
\begin{abstract}
One of the fundamental problems in network analysis is detecting community structure
in multi-layer networks, of which each layer represents one type of edge information among the nodes. We propose integrative spectral clustering approaches based on effective convex layer aggregations. Our aggregation methods are strongly motivated by a delicate asymptotic analysis of the spectral embedding of weighted adjacency matrices and the downstream $k$-means clustering, in a challenging regime where community detection consistency is impossible. In fact, the methods are shown to estimate the optimal convex aggregation, which minimizes the mis-clustering error under some specialized multi-layer network models. Our analysis further suggests that clustering using Gaussian mixture models is generally superior to the commonly used $k$-means in spectral clustering. Extensive numerical studies demonstrate that our adaptive aggregation techniques, together with Gaussian mixture model clustering, make the new spectral clustering remarkably competitive compared to several popularly used methods. 
\end{abstract}

\noindent%
{\it Keywords:}  Asymptotic Mis-clustering Error; Community Detection; Convex Aggregation; Eigenvalue Ratio; Gaussian Mixture Distributions; $k$-means; Multi-layer Networks; Spectral Clustering
\vfill

\newpage
\spacingset{1.5} 

\section{Introduction}
\label{sec:intro}
Clustering network data is termed community detection, where communities are understood as groups of nodes that share more similarities with each other than with other nodes. Examples include social groups in a social network and papers on the same research topic in a citation network. Community detection is one of the fundamental problems in network analysis to understand the network structure and functionality \citep{newman2018networks}. With enormous effort from a broad spectrum of disciplines, a large set of methodologies have been proposed and can be roughly classified into algorithmic and model-based ones \citep{zhao2012consistency}. Examples  of algorithmic methods include divisive algorithms using edge betweenness \citep{girvan2002community}, network random walk \citep{zhou2003distance}, spectral method \citep{hagen1992new, shi2000normalized}, modularity optimization \citep{newman2006modularity}, and information-theoretic approaches \citep{rosvall2008maps}. We refer to \cite{fortunato2010community} for a thorough and in-depth review of algorithmic community detection techniques. The class of model-based methods relies on fitting probabilistic models and applying statistical inference tools. Several widely studied models include stochastic block model (SBM) \citep{holland1983stochastic}, degree-corrected stochastic block model (DSBM) \citep{karrer2011stochastic}, mixed membership stochastic block model \citep{airoldi2008mixed}, and latent variable models \citep{handcock2007model, hoff2008modeling}, among others.



In the past decade, there have been increasingly active researches towards understanding the theoretical performance of community detection methods 
under different types of models. The seminal work by \cite{bickel2009nonparametric} introduced an asymptotic framework for the study of community detection consistency, and developed a general theory for checking the consistency properties of a range of methods including modularity and profile likelihood maximization, under SBM. Their consistency framework and results have been generalized to DSBM \citep{zhao2012consistency},  to allow the number of communities to diverge with network size for maximum likelihood estimator under SBM \citep{choi2012stochastic}, and to scalable pseudo-likelihood method \citep{amini2013pseudo}. The consistency and asymptotic normality of maximum likelihood and variational estimators for other model parameters were also established under both SBM and DSBM \citep{celisse2012consistency, bickel2013asymptotic}. Another important line of research focuses on the analysis of spectral clustering. Consistency or mis-clustering error rate of different variants of spectral clustering approaches have been investigated under SBM \citep{rohe2011spectral, lei2015consistency, yu2015useful, le2017concentration}, DSBM \citep{qin2013regularized, jin2015fast},  mixed membership models \citep{jin2017estimating, mao2020estimating, zhang2020detecting}, and SBM with covariates \citep{zhang2016community, weng2022community}. A class of semidefinite optimization approaches with established strong performance guarantees has been developed \citep{cai2015robust, guedon2016community, amini2018semidefinite}. Moreover, there is one major stream of research focused on characterizing the fundamental limits of community detection, under the minimax framework \citep{zhang2016minimax, gao2017achieving, xu2020optimal}, and with respect to sharp information-theoretic and computational thresholds \citep{decelle2011asymptotic, krzakala2013spectral, abbe2015exact, mossel2015reconstruction, abbe2017community}. 


All the aforementioned methods perform community detection based on a single-layer network that only retains one specific type of edge information. However, multi-layer networks with multiple types of edge information are ubiquitous in the real world. For example, employees in a company can have various types of relations such as Facebook friendship and coworkership; genes in a cell can have both physical and coexpression interactions; stocks in the US financial market can have different levels of stock price correlations in different time periods. See \cite{kim2015community} for a comprehensive report of multi-layer network data. Each type of relationship among the nodes forms one layer of the network. Each layer can carry potentially useful information about the underlying communities. Integrating the edge information from all the layers to obtain a more accurate community detection is of great importance. This problem of community detection for multi-layer networks has received significant interest over the last decade. The majority of existing algorithmic methods base on either spectral clustering \citep{long2006spectral, zhou2007spectral, kumar2011co, dong2012clustering} or low-rank matrix factorization \citep{singh2008relational, tang2009clustering, nickel2011three, liu2013multi}, and combine information from different layers via some form of regularization. An alternative category of methods relies on fitting probabilistic generative models. Various extensions of single-layer network models to multi-layer settings have been proposed, including multi-layer stochastic block models \citep{han2015consistent, stanley2016clustering, valles2016multilayer, paul2016consistent}, multi-layer mixed-membership stochastic block model \citep{de2017community}, Poisson Tucker decomposition models \citep{schein2015bayesian, schein2016bayesian}, and Bayesian latent factor models \citep{jenatton2012latent}, among others. However, the theoretical understanding of multi-layer community detection methods has been rather limited. Under multi-layer stochastic block models, \cite{han2015consistent} proved the consistency of maximum likelihood estimation (MLE) as the number of layers goes to infinity and the number of nodes is fixed. \cite{paul2016consistent} established the consistency of MLE under much more general conditions that allow both the numbers of nodes and layers to grow. The authors further obtained the minimax rate over a large parameter space. \cite{paul2020spectral} derived the asymptotic results for several spectral and matrix factorization based methods in the high-dimensional setting where the numbers of layers, nodes, and communities can all diverge. \cite{bhattacharyya2018spectral} proposed a spectral clustering method that can consistently detect communities even if each layer of the network is extremely sparse. They then generalized their method and results to multi-layer degree-corrected block models. 

In this paper, we consider spectral clustering after convex layer aggregation, a two-step framework for multi-layer network community detection. Our contribution is two-fold. Firstly, with a sharp asymptotic characterization of the mis-clustering error, we reveal the impact of a given convex aggregation on the community detection performance. This motivates us to develop two adaptive aggregation methods that can effectively utilize community structure information from different layers. Secondly, our study of the spectral embedding suggests using clustering with Gaussian mixture models as a substitute for the commonly adopted $k$-means in spectral clustering. Together, these two proposed recipes strengthen the two-step procedure to be an efficient community detection approach that outperforms several popular methods, especially for networks with heterogeneous layers. \textcolor{black}{Throughout the paper, our treatment will be mainly focused on networks with assortative community structures in all layers. We discuss the application of our proposed methods to networks of a mixed community structure (with both assortative and dis-assortative structures) in Section \ref{dis:sess}. We refer the reader to \cite{bhattacharyya2020general, paul2020spectral, lei2020consistent, lei2022bias} for recent developments towards effectively combining both assortative and dis-assortative community structures in multi-layer networks. We should also point out that our asymptotic analysis considers the partial recovery regime \citep{abbe2017community} where the node degrees diverge to infinity sufficiently fast while the gap between the within and between community probabilities remains small. Such asymptotics yields precise error characterization that motivates the proposed approaches. We discuss this asymptotic regime in detail in Section \ref{asymptotic:error}.}

\section{Detecting communities in multi-layer networks using new spectral methods}\label{sec:def}

\subsection{Definitions and problem statement}\label{two:step}

We focus on undirected networks throughout the paper. The observed edge information of a single-layer network with $n$ nodes can be represented by the symmetric adjacency matrix  $A=(A_{ij})\in \{0,1\}^{n\times n}$, where $A_{ij}=A_{ji}=1$ if and only if there exists a connection between nodes $i$ and $j$. Suppose the network can be divided into $K$ non-overlapping communities, and let $\overrightarrow{c}=(c_1,\cdots,c_n)^T$ be latent community membership vector corresponding to nodes $1,\cdots, n$, taking values in $[K]:=\{1,\cdots,K\}$. Arguably, the most studied network model for community detection is the stochastic block model (SBM) \citep{holland1983stochastic}.

\begin{definition}[Stochastic Block Model]
        The latent community labels $\{c_i\}_{i=1}^n$ are independently sampled from  a multinomial distribution, i.e., for $i\in[n]$ and $k\in[K]$, ${\rm pr}(c_{i}=k)=\pi_k$ with constraint $\sum_{k=1}^K\pi_k=1$, $0<\pi_k<1$. Define a symmetric connectivity matrix $\Omega=(\Omega_{ab})\in (0,1)^{K\times K}$. Conditioning on $\overrightarrow{c}$, the adjacency matrix $A$ has independent entries with $A_{ij}\sim ${\rm Bernoulli}$(\Omega_{c_ic_j})$ for all $i\leq j$. We denote the model by $A\sim$ \textsc{SBM}$(\Omega,\overrightarrow{\pi})$.  
\end{definition}  

Under SBM, the distribution of the edge between nodes $i$ and $j$ only depends on their community assignments $c_i$ and $c_j$. Nodes from the same community are stochastically equivalent. The parameters $\overrightarrow{\pi}=(\pi_1,\cdots,\pi_K)^T$ control sizes of the $K$ communities. We call the network balanced when $\pi_1=\cdots=\pi_K=1/K$. The symmetric matrix $\Omega$ represents the connectivity probabilities among the communities. A special case of SBM that has been widely studied in theoretical computer science is called the planted partition model \citep{bui1987graph, dyer1989solution}, where the values of the probability matrix $\Omega$ are one constant  on the diagonal and another constant off the diagonal.

\begin{definition}[Planted Partition Model]
	A \emph{planted partition model (PPM)} is a special homogeneous SBM, of which the connectivity matrix is $\Omega=(p-q)I_K+qJ_K    \in(0,1)^{K\times K},$, where $I_K$ is the identity matrix and $J_K$ is the matrix of ones. This means the within-community connectivity probabilities of PPM are all $p$ while the between-community probabilities are all $q$. \textcolor{black}{Given that our focus is on assortative networks}, we assume $p>q$ throughout this paper. The model is written as $A\sim$ \textsc{PPM}$(p,q, \overrightarrow{\pi})$.
\end{definition}

In this paper, we consider a multi-layer network of $L$ layers to be a collection of $L$ single-layer networks that share the same nodes but with different edges. For each $\ell\in [L]$, the adjacency matrix $A^{(\ell)}=(A^{(\ell)}_{ij})\in \{0, 1\}^{n\times n}$ represents edge information from the $\ell^{th}$ layer. The multi-layer stochastic block model (MSBM) \citep{han2015consistent, paul2016consistent, bhattacharyya2018spectral} is a natural extension of the standard SBM to the multi-layer case.

\begin{definition}[Multi-layer Stochastic Block Model]\label{Def:Mul}
The layers of a multi-layer stochastic block model share the common community assignments $\overrightarrow{c}=(c_1,\cdots,c_n)^T \in [K]^n$ which are independently sampled from  a multinomial distribution with parameters $\overrightarrow{\pi}=(\pi_1,\cdots,\pi_K)^T$. Conditioning on $\overrightarrow{c}$, all the adjacency matrices have independent entries with $A^{(\ell)}_{ij}\sim ${\rm Bernoulli}$(\Omega^{(\ell)}_{c_ic_j})$ for all $\ell \in [L]$ and $i\leq j$. We write the model as $A^{[L]}\sim\textsc{MSBM}(\Omega^{[L]},\overrightarrow{\pi})$ for short.
\end{definition}

Under MSBM, each layer follows an SBM with consensus community assignments $\overrightarrow{c}$, but with possibly different connectivity patterns as characterized by the set of parameters $\Omega^{[L]}=\{\Omega^{(\ell)}\}_{\ell=1}^L$. A multi-layer planted partition model is a special type of MSBM when each layer follows a planted partition model.

\begin{definition}[Multi-layer Planted Partition Model]\label{Multi:ppm}
         A multi-layer planted partition model (MPPM) is a special MSBM with $A^{(\ell)}\sim$ \textsc{PPM}$(p^{(\ell)},q^{(\ell)}, \overrightarrow{\pi})$ for each $\ell\in [L]$, i.e., $\Omega^{(\ell)}=(p^{(\ell)}-q^{(\ell)})I_K+q^{(\ell)}J_K    \in(0,1)^{K\times K}, ~\ell \in [L].$ The model is written as $A^{[L]}\sim$ \textsc{MPPM}$(p^{[L]},q^{[L]}, \overrightarrow{\pi})$.
\end{definition}

We consider the following two-step framework for multi-layer community detection:
\begin{itemize}
\item[(1)] \emph{Convex layer aggregation}. Form the weighted adjacency matrix
$A^{\overrightarrow{w}}=\sum_{\ell=1}^Lw_{\ell}A^{(\ell)}$,
for some $\overrightarrow{w} \in \mathcal{W}=\big\{\overrightarrow{w}:\sum_{\ell=1}^Lw_{\ell}=1, w_{\ell}\geq 0, \ell\in [L] \big \}$.

\item[(2)] \emph{Spectral clustering}. Suppose the spectral decomposition of $A^{\overrightarrow{w}}$ is given by 
$A^{\overrightarrow{w}}=\sum_{i=1}^n\lambda_i u_i u^T_i$,
where $\{\lambda_i\}_{i=1}^n$ are the eigenvalues and $\{u_i\}_{i=1}^n$ are the corresponding orthonormal eigenvectors. The eigenvalues are ordered in magnitude so that $|\lambda_1|\geq |\lambda_2|\geq \cdots \geq |\lambda_n|$. Form the eigenvector matrix $U=(u_1,u_2,\cdots, u_K)\in \mathcal{R}^{n\times K}$. Treat each row of $U$ as a data point in $\mathcal{R}^K$ and run the $k$-means clustering on the $n$ data points. The cluster label outputs from the $k$-means are the community membership estimates for the $n$ nodes. 
\end{itemize}

In the first step, the community structure information from different layers is integrated by a simple convex aggregation of the adjacency matrices. The second step runs spectral clustering on the weighted adjacency matrix. A similar two-step spectral method has been considered in \citet{chen2017multilayer}, where the authors presented a phase transition analysis of clustering reliability. With a distinctly different focus, we will provide novel solutions to refine the framework towards a better community detection approach. In the case when $\overrightarrow{w}=(1/L, \cdots, 1/L)$, the two-step procedure is considered as a generally effective baseline method in the literature \citep{tang2009clustering, kumar2011co, dong2012clustering, bhattacharyya2018spectral, paul2020spectral}. However, it is common that different levels of signal-to-noise ratios exist across the layers; hence aggregation with equal weights may not be the optimal choice. Thusly motivated, we propose two approaches that can adaptively choose a favorable weight vector $\overrightarrow{w}\in{\mathcal{W}}$ leading to superior community detection results. The two methods are simple and intuitive while having strong theoretical motivations. We provide analytical calculations to reveal that the two approaches, in fact, are estimating the optimal weight vector that minimizes the mis-clustering error (see formal definition of the error in Section \ref{asymptotic:error}) under balanced multi-layer planted partition models. Moreover, we present convincing arguments to show that, the $k$-means clustering in the spectral clustering step should be replaced by clustering using Gaussian mixture models to effectively capture the shape of spectral embedded data thus yielding improved community detection performances. In a nutshell, the proposed weight selection and change of the clustering method make the two-step framework a highly competitive multi-layer community detection procedure, as will be demonstrated by extensive numerical experiments in Sections \ref{sec:sim} and \ref{sec:real}.

\subsection{Asymptotic mis-clustering error under balanced MPPM} \label{asymptotic:error}

We first provide some asymptotic characterization of the two-step framework introduced in Section \ref{two:step}. The asymptotic analysis paves the way to develop two adaptive layer aggregation methods in Sections \ref{multi:opt_close} and \ref{multi:opt_ratio}. Towards this end, let $\delta: [K]\rightarrow [K]$ denote a permutation of $[K]$, and $\overrightarrow{c}=(c_1,\cdots, c_n)^T\in [K]^n$ be the latent community assignments. The mis-clustering error for a given community estimator $\hat{\overrightarrow{c}}=(\hat{c}_1,\cdots, \hat{c}_n)^T$ is defined as $r(\hat{\overrightarrow{c}})=\inf_{\delta}\sum_{i=1}^n1(\delta(\hat{c}_i)\neq c_i)/n$, which is the proportion of nodes that are mis-clustered, modulo permutations of the community labels. We consider a sequence of balanced multi-layer planted partition models indexed by the network size $n$: $A^{[L]}_n\sim$ \textsc{MPPM}$(p_n^{[L]},q_n^{[L]},\overrightarrow{\pi})$ with both $K$ and $L$ fixed. For a given weight $\overrightarrow{w}\in \mathcal{W}$, we introduce the following quantity that plays a critical role in the error characterization:
\begin{align}
\label{finite:snr}
\tau^{\overrightarrow{w}}_n = \frac{n\big[\sum_{\ell=1}^Lw_{\ell}(p^{(\ell)}_n-q^{(\ell)}_n)\big]^2}{\sum_{\ell=1}^Lw_{\ell}^2\big[p^{(\ell)}_n(1-p^{(\ell)}_n)+(K-1)q^{(\ell)}_n(1-q^{(\ell)}_n)\big]}.
\end{align}

\begin{theorem}\label{thm:error}
Recall the two-step procedure in Section \ref{two:step}. For a given weight $\overrightarrow{w}\in \mathcal{W}$, let $\hat{\overrightarrow{c}}_{\overrightarrow{w}}$ be the corresponding community estimator. Assume the following conditions
\begin{enumerate}
\item[(i)] $\sum_{\ell=1}^Lw_{\ell}^2\big[p^{(\ell)}_n(1-p^{(\ell)}_n)+(K-1)q^{(\ell)}_n(1-q^{(\ell)}_n)\big]=\Omega(n^{-1}\log^4n)$,
\item[(ii)] $[\sum_{\ell=1}^Lw_{\ell}^2p^{(\ell)}_n(1-p^{(\ell)}_n)]\cdot [\sum_{\ell=1}^Lw_{\ell}^2q^{(\ell)}_n(1-q^{(\ell)}_n)]^{-1}\rightarrow 1$,
\item[(iii)] $\tau^{\overrightarrow{w}}_{\infty}\equiv \lim_{n\rightarrow \infty}\tau^{\overrightarrow{w}}_n<\infty$,
\end{enumerate}
where for two sequences $a_n$ and $b_n$, $a_n = \Omega(b_n)$ means that $\lim\sup_{n\to \infty} |a_n/b_n|>0$. 
It holds that for $\tau^{\overrightarrow{w}}_{\infty}\in (K,\infty)$, as $n \rightarrow \infty$,
\begin{align}
\label{error:limit}
E[r(\hat{\overrightarrow{c}}_{\overrightarrow{w}})] \rightarrow 1-{\rm pr}(a_i \geq 0, i=1,2,\cdots, K-1),
\end{align}
where $\overrightarrow{a}=(a_1,\cdots, a_{K-1})^T \sim \mathcal{N}(\overrightarrow{\mu}, \Sigma)$ with $\overrightarrow{\mu}=\sqrt{\tau^{\overrightarrow{w}}_{\infty}-K}\cdot (1,1,\cdots, 1)^T, \Sigma=I_{K-1}+J_{K-1}$. Moreover, the asymptotic error $1-{\rm pr}(a_i \geq 0, i=1,2,\cdots, K-1)$ is a strictly monotonically decreasing function of $\tau^{\overrightarrow{w}}_{\infty}$ over $(K,\infty)$.
\end{theorem}

\begin{remark}
\label{rem:one}
In light of the asymptotic error characterization in Theorem \ref{thm:error}, $\tau^{\overrightarrow{w}}_{\infty}$ can be interpreted as the signal-to-noise ratio (SNR) for the two-step procedure with a given weight $\overrightarrow{w}$. The assumption $\tau^{\overrightarrow{w}}_{\infty}>K$ is critical. As will be clear from Proposition \ref{thm:main2}, if the SNR is below the threshold $K$, the informative eigenvalues and eigenvectors of the weighted adjacency matrix $A^{\overrightarrow{w}}$ cannot be separated from the noisy ones. As a result, the spectral clustering in the second step will completely fail. Some supporting simulations are shown in Figure \ref{fig:err}. As the left plot demonstrates, the two-step method performs like random guess when $\tau_n^{\overrightarrow{w}}$ is smaller than the cut-off point $K=2$.
\end{remark}

\begin{remark}
\label{asym:one}
A general scenario where the conditions in Theorem \ref{thm:error} will hold for all $\overrightarrow{w}\in \mathcal{W}$ is $p^{(\ell)}_n=\Omega(n^{-1}\log^4 n), n(p^{(\ell)}_n-q^{(\ell)}_n)^2(p^{(\ell)}_n)^{-1}=\Theta(1)$. This implies that $p^{(\ell)}_n-q^{(\ell)}_n \ll q_n^{(\ell)}\propto p_n^{(\ell)}$, i.e., the gap between within-community and between-community connectivity probabilities is of smaller order compared to the connectivity probabilities themselves. This is in marked contrast to the equal order assumption that is typically made in the statistical community detection literature \citep{zhao2017survey, abbe2017community}. \textcolor{black}{The minimax rate result in \cite{zhang2016minimax} shows that for finite number of communities, the sufficient and necessary condition for consistent single-layer community detection is $n(p^{(\ell)}_n-q^{(\ell)}_n)^2(p^{(\ell)}_n)^{-1}\rightarrow \infty$}. As a result, under our conditions the mis-clustering error will not vanish asymptotically even though the sequence of networks are sufficiently dense $p_n^{(\ell)}=\Omega(n^{-1}\log^4 n)$. We believe this asymptotic set-up where community detection consistency is unattainable, is a more appropriate analytical platform to understand real large-scale networks which are dense but not necessarily have strong community structure signals. Furthermore, it enables us to obtain the \emph{asymptotically} exact error formula that reveals the precise impact of the weight $\overrightarrow{w}$ on community detection. 
\end{remark}

\begin{remark}
\label{asym:two}
Because of the distinct asymptotic regime as explained in the last paragraph, existing works on the asymptotic properties of eigenvectors of random matrices \citep{tang2018limit, cape2019two, fan2019asymptotic, abbe2020entrywise} cannot be directly applied or adapted to the current setting. Our asymptotic analysis is motivated by the study of eigenvectors of large Wigner matrices in \citet{bai2012limiting}. A similar asymptotic setting to ours was adopted in \citet{deshpande2017asymptotic, deshpande2018contextual}. However, notably different from our study, these two works focused on characterizing the information-theoretical limit of community detection for single-layer networks via message passing or belief propagation algorithms. We should also point out that community detection inconsistency can also arise for sparse networks where the degree of nodes remains bounded (see \citet{decelle2011asymptotic, krzakala2013spectral, abbe2017community} and references therein). Nevertheless, to our best knowledge, the \emph{asymptotically} exact error characterization in this regime is largely unknown. 
\end{remark}

\begin{figure}[htb!]
\centering
\begin{tabular}{cc}
\hspace{-0.cm} \includegraphics[scale=0.5]{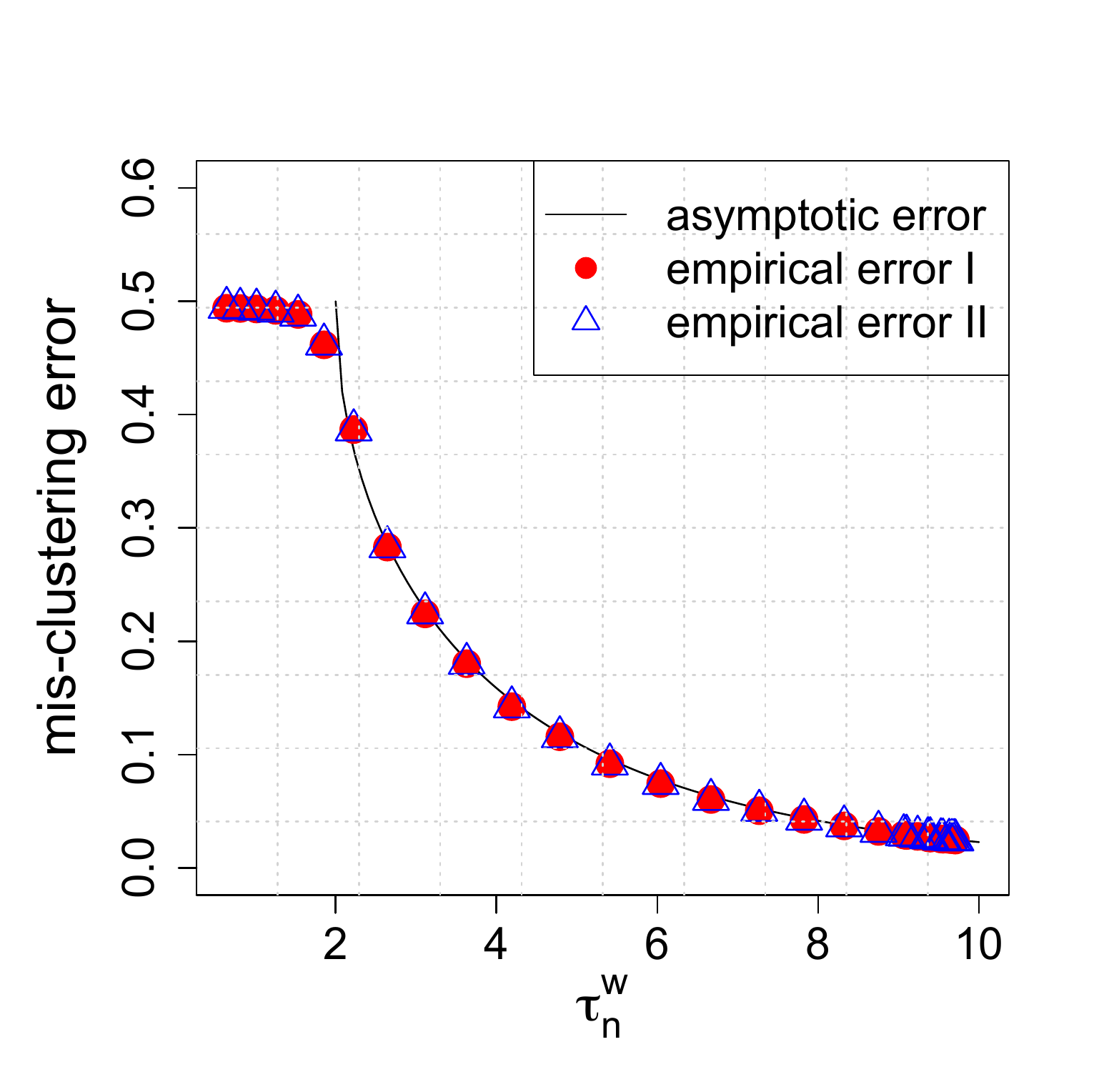} &
 \includegraphics[scale=0.53]{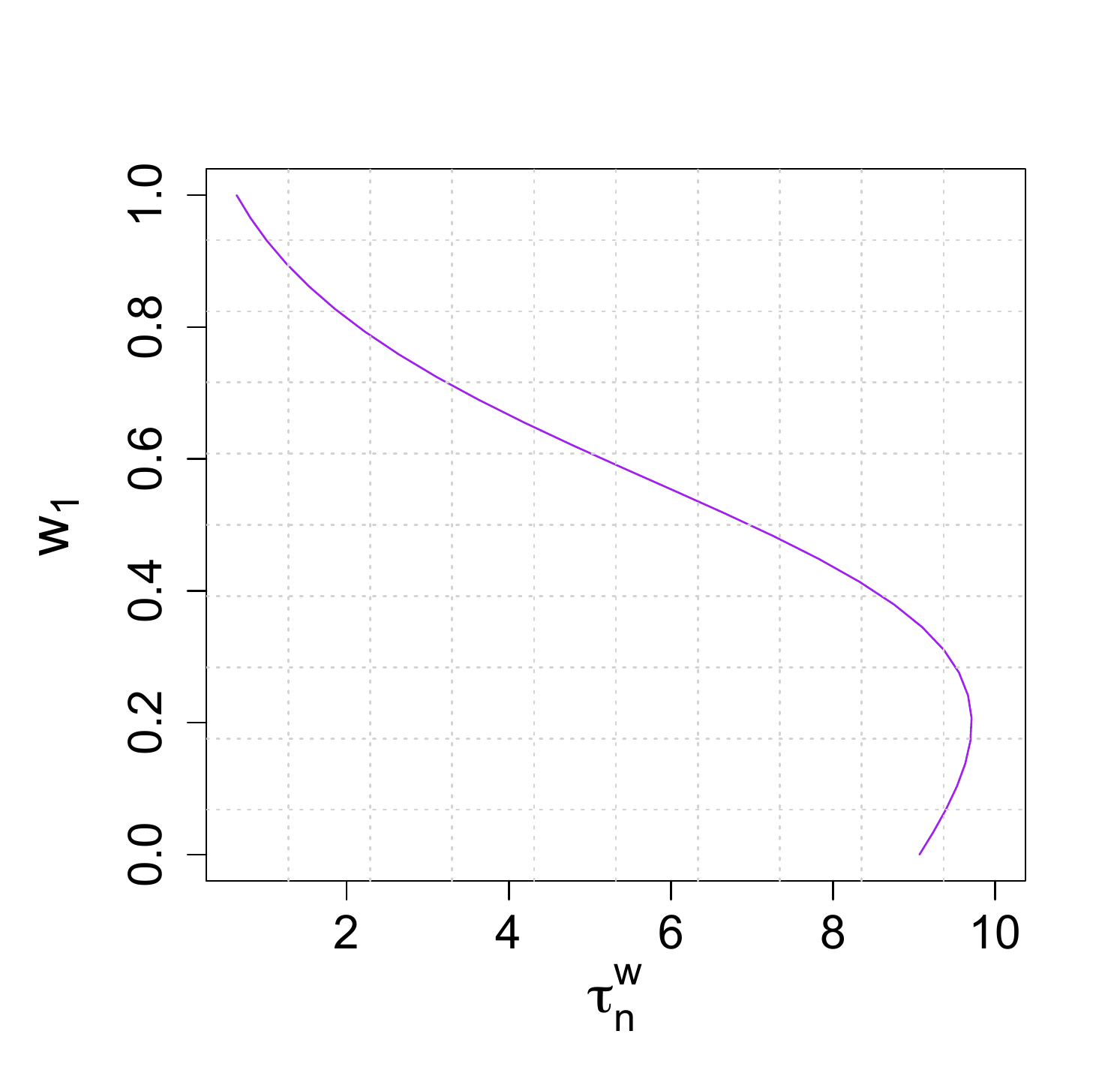} 
\end{tabular}
\vspace{-0.5cm}
\caption{Consider a balanced multi-layer planted partition model with $K=2, L=2, n=6000, p^{(1)}=0.02, q^{(1)}=0.018, p^{(2)}=0.02, q^{(2)}=0.013$. Left panel: ``asymptotic error" is calculated according to the limit formula \eqref{error:limit} in Theorem \ref{thm:error}; ``empirical error I" denotes the finite-sample error of the two-step procedure described in Section \ref{two:step}; ``empirical error II" represents the finite-sample error of the modified two-step procedure with $k$-means replaced by clustering using Gaussian mixture models; \textcolor{black}{both empirical errors are calculated for the procedures without making use of the oracle information of parameters}. Finite-sample errors are the averages over 5 repetitions. Right panel: the y-axis refers to the first component of the weight vector.}
\label{fig:err}
\end{figure}

\begin{remark}
We conduct a small simulation to evaluate the results of Theorem \ref{thm:error}. As is clear from the left panel of Figure \ref{fig:err}, the asymptotic error is a rather accurate prediction of the finite-sample error over a wide range of SNR when the network size is large. Moreover, \textcolor{black}{using the parameter values specified in the caption of Figure \ref{fig:err}}, it is straightforward to compute $\tau_n^{(1,0)^T}=0.64, \tau_n^{(0,1)^T}=9.07$. Hence, the second layer is much more informative than the first one. The right panel reveals that the optimal weight is located around $\overrightarrow{w}=(0.2, 0.8)^T$. It is interesting to observe that although the first layer alone acts like random noise to spectral clustering (since $\tau_n^{(1,0)^T}<K=2$; see also the left panel), appropriately combined with the second layer it provides useful community structure information that contributes to a decent performance boost compared to the result solely based on the second layer. We, therefore, see the importance of weight tuning in the two-step procedure. 
\end{remark}

\subsection{Iterative spectral clustering}\label{multi:opt_close}

We are in the position to introduce the first adaptive layer aggregation method. Theorem \ref{thm:error} reveals that the asymptotic mis-clustering error depends on $\tau^{\overrightarrow{w}}_{\infty}$ in a strictly monotonically decreasing fashion. Hence the optimal $\overrightarrow{w}$ that will minimize the asymptotic error can be found by solving $\max_{\overrightarrow{w} \in \mathcal{W}} \tau^{\overrightarrow{w}}_{\infty}$. Naturally, in the realistic finite-sample scenario where a multi-layer network with $n$ nodes is given, we would like to use the weight vector that maximize $ \tau^{\overrightarrow{w}}_n$.

\begin{proposition}\label{prop:one}
Recall $\tau^{\overrightarrow{w}}_n$ defined in \eqref{finite:snr}. Denote $\overrightarrow{w}^*=(w^*_1,\cdots, w_L^*)^T=\arg\max_{\overrightarrow{w} \in \mathcal{W}} \tau^{\overrightarrow{w}}_n$. Then $\overrightarrow{w}^*$ exists and is unique, admitting the explicit expression:
	\begin{align}\label{eq:1}
		w_l^*&\propto\frac{p^{(l)}_n-q^{(l)}_n}{p^{(l)}_n(1-p^{(l)}_n)+(K-1)q^{(l)}_n(1-q^{(l)}_n)}, \quad \ {\rm for}\ l\in[L],
	\end{align}
where a normalization constant is taken to ensure $\overrightarrow{w}^* \in \mathcal{W}$.
\end{proposition}

Formula \eqref{eq:1} has a simple and intuitive interpretation. Each $w^*_{\ell}$ is determined by the layer's own parameters. It measures the standardized difference between the within and between community connectivity probability, where the standardization (ignoring the smaller order terms $(p_n^{(\ell)})^2, (q_n^{(\ell)})^2$ in the denominator) is essentially the averaged connectivity probability over all pairs of nodes. A larger standardized difference implies a stronger community structure signal, deserving a larger weight on the corresponding layer. 

The weight vector $\overrightarrow{w}^*$ cannot be directly used in the two-step procedure as it depends on unknown parameters $\{(p_n^{(\ell)},q_n^{(\ell)})\}_{\ell=1}^L$. To address this issue, we propose an iterative spectral clustering (ISC) method. For each $\ell \in [L]$, we first run spectral clustering on the layer's own adjacency matrix $A^{(l)}$ to obtain an initial community estimator $\hat{\overrightarrow{c}}^{(\ell)}$ and compute the weight estimates:
\begin{align}
\label{weight:update}
\hat{w}^*_{\ell} \propto  \frac{\hat{p}^{(\ell)}_n-\hat{q}^{(\ell)}_n}{\hat{p}^{(\ell)}_n(1-\hat{p}^{(\ell)}_n)+(K-1)\hat{q}^{(\ell)}_n(1-\hat{q}^{(\ell)}_n)},
\end{align}
where 
\begin{align}
\label{pq:estimate}
\hat{p}^{(\ell)}_n=\frac{\sum_{ij}A^{(\ell)}_{ij}1(\hat{c}^{(\ell)}_i=\hat{c}^{(\ell)}_j)}{\sum_{ij}1(\hat{c}^{(\ell)}_i=\hat{c}^{(\ell)}_j)}, \quad \hat{q}^{(\ell)}_n=\frac{\sum_{ij}A^{(\ell)}_{ij}1(\hat{c}^{(\ell)}_i\neq \hat{c}^{(\ell)}_j)}{\sum_{ij}1(\hat{c}^{(\ell)}_i \neq \hat{c}^{(\ell)}_j)}.
\end{align}
We then run spectral clustering on the weighted adjacency matrix $A^{\hat{\overrightarrow{w}}^*}=\sum_{l=1}^L\hat{w}^*_lA^{(l)}$ to obtain a refined community estimate $\hat{\overrightarrow{c}}$. Such refinement can be repeatedly applied until convergence. We summarize the outlined method as Algorithm \ref{alg:formula}.

\begin{algorithm}[htb] 
\caption{Iterative spectral clustering [ISC]} \label{alg:formula}
\begin{algorithmic}[1] 
\REQUIRE 
$L$ layers of adjacency matrices $[A^{(l)}, l=1,\cdots,L]$, the number of communities  $K$, and the precision parameter $\epsilon_0$.
\ENSURE 
$\hat{\overrightarrow{w}}_{\rm new}$ and the community estimate $\hat{\overrightarrow{c}}$ by apply spectral clustering on $A^{\hat{\overrightarrow{w}}_{\rm  new}}$.
\STATE Initialization:  Apply spectral clustering on every single $A^{(l)}$, and compute the initial weight estimates
    according to \eqref{weight:update} and \eqref{pq:estimate}. Denote it by $\hat{\overrightarrow{w}}_{\rm old}$.  Set $\epsilon = \epsilon_0+1$.
 \WHILE {$\epsilon > \epsilon_0$}
 \STATE Apply spectral clustering on $A^{\hat{\overrightarrow{w}}_{\rm old}}$ and compute updated weights $\hat{\overrightarrow{w}}_{\rm new}$ as in \eqref{weight:update} and \eqref{pq:estimate}.
 \STATE Assign $\epsilon \leftarrow \|\hat{\overrightarrow{w}}_{\rm old}-\hat{\overrightarrow{w}}_{\rm new}\|$ and $\hat{\overrightarrow{w}}_{\rm old}\leftarrow \hat{\overrightarrow{w}}_{\rm new}$. 
 \ENDWHILE
\end{algorithmic}
\end{algorithm}

\begin{remark}
Algorithm \ref{alg:formula} is motivated by the asymptotic analysis of the mis-clustering error under balanced multi-layer planted partition models. However, as will be demonstrated by extensive numerical experiments in Sections \ref{sec:sim} and \ref{sec:real}, it works well for a much larger family of multi-layer stochastic block models. \textcolor{black}{An intuitive explanation is that for general stochastic block models, \eqref{pq:estimate} is estimating the averaged within and between community probabilities; and the weight formula \eqref{weight:update} represents a certain normalized gap between the aforementioned two probabilities which can be considered as a measure of the community signal strength, thus providing useful aggregation information. That being said, the weight formula \eqref{weight:update} is not necessarily estimating the optimal weights for general multi-layer stochastic block models. Deriving the optimal weight formulas in such a general setting is an interesting and important future research. On a related note, since our focus is on the partial recovery regime (see Section \ref{asymptotic:error} for detailed discussions), it would be interesting to investigate how well the optimal weight is estimated under MPPM cases. To our best knowledge, in the partial recovery regime when the node degrees diverge with $n$, the fundamental limits for parameter estimation have not been established in the literature.} We defer a thorough evaluation and discussion of Algorithm \ref{alg:formula} to Sections \ref{sec:sim} and \ref{sec:real}.
\end{remark}

\subsection{Spectral clustering with maximal eigenratio}\label{multi:opt_ratio}

We now present the second method to select the weight $\overrightarrow{w}$. Let $\lambda^{\overrightarrow{w}}_i$ be the $i^{th}$ largest (in magnitude) eigenvalue of the weighted adjacency matrix $A^{\overrightarrow{w}}$. The spectral clustering in the two-step procedure is implemented using the eigenvectors corresponding to the first $K$  eigenvalues $\{\lambda^{\overrightarrow{w}}_i, i=1,\cdots,K\}$. We will show that in addition to these eigenvectors, the eigenvalues of $A^{\overrightarrow{w}}$ can be utilized for community detection. In particular, the eigenvalue ratio $\lambda^{\overrightarrow{w}}_{K}/\lambda^{\overrightarrow{w}}_{K+1}$ holds critical information about how $\overrightarrow{w}$ affects mis-clustering error, as shown in following proposition. 

\begin{proposition}\label{thm:main2}
Under the same conditions of Theorem \ref{thm:error}, it holds that 
\begin{align*}
\frac{|\lambda_K^{{\overrightarrow{w}}}|}{|\lambda_{K+1}^{{\overrightarrow{w}}}|} \overset{a.s.}{\rightarrow}
\begin{cases}
\frac{1}{2}\Big(\sqrt{\frac{\tau_{\infty}^{\overrightarrow{w}}}{K}}+\sqrt{\frac{K}{\tau_{\infty}^{\overrightarrow{w}}}} \Big), & \text{if}\ \tau_\infty^{{\overrightarrow{w}}}>K,\\
1, & \text{if}\ \tau_\infty^{{\overrightarrow{w}}}\leq K.
\end{cases}
\end{align*}
Here, we have suppressed the dependence of $\lambda^{\overrightarrow{w}}_{K}$ and $\lambda^{\overrightarrow{w}}_{K+1}$ on $n$ to simplify the notation. 
\end{proposition}

Proposition \ref{thm:main2} reveals that the absolute eigenratio $|\lambda_K^{{\overrightarrow{w}}}|/|\lambda_{K+1}^{{\overrightarrow{w}}}|$ undergoes a phase transition: it remains a constant one when the SNR $\tau_\infty^{{\overrightarrow{w}}}$ is smaller than $K$; the ratio will be strictly increasing in $\tau_\infty^{{\overrightarrow{w}}}$ once $\tau_\infty^{{\overrightarrow{w}}}>K$. This phenomenon is consistent with Theorem \ref{thm:error}. Indeed, when $\tau_\infty^{{\overrightarrow{w}}}$ is below the threshold $K$, the informative eigenvalue $\lambda_K^{{\overrightarrow{w}}}$ is indistinguishable from the noisy one $\lambda_{K+1}^{{\overrightarrow{w}}}$. Spectral clustering on $A^{\overrightarrow{w}}$ will thus fail. See Remark \ref{rem:one} for more details. On the other hand, as the SNR $\tau_\infty^{{\overrightarrow{w}}}$ increases over the range $(K,\infty)$, the eigenratio becomes larger so that $\lambda_K^{{\overrightarrow{w}}}$ is better separated from $\lambda_{K+1}^{{\overrightarrow{w}}}$ and the mis-clustering error of spectral clustering on $A^{\overrightarrow{w}}$ is decreased. Figure \ref{fig:ratio} depicts both the finite-sample and asymptotic values of the eigenratio from the same simulation study as described in Figure \ref{fig:err}. Clearly the asymptotic values are fine predictions of the empirical ones. 

\begin{figure}[!htb]
	\centering
	\begin{subfigure}[b]{0.45\textwidth}
	\includegraphics[scale=0.5]{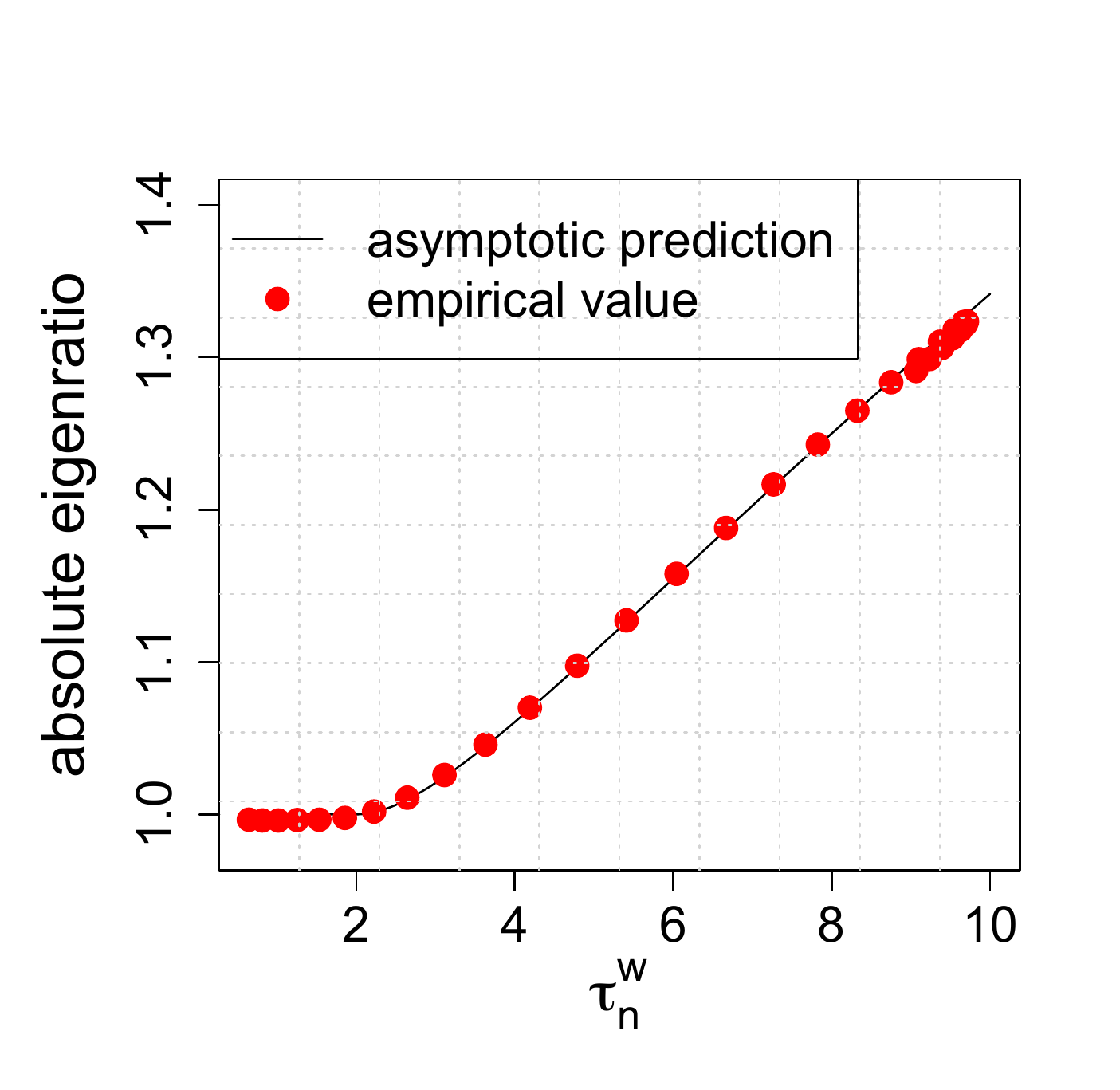}
	\caption{Absolute Eigenratio\label{fig:ratio}}

	\end{subfigure}
	\begin{subfigure}[b]{0.45\textwidth}
	\includegraphics[scale=0.5]{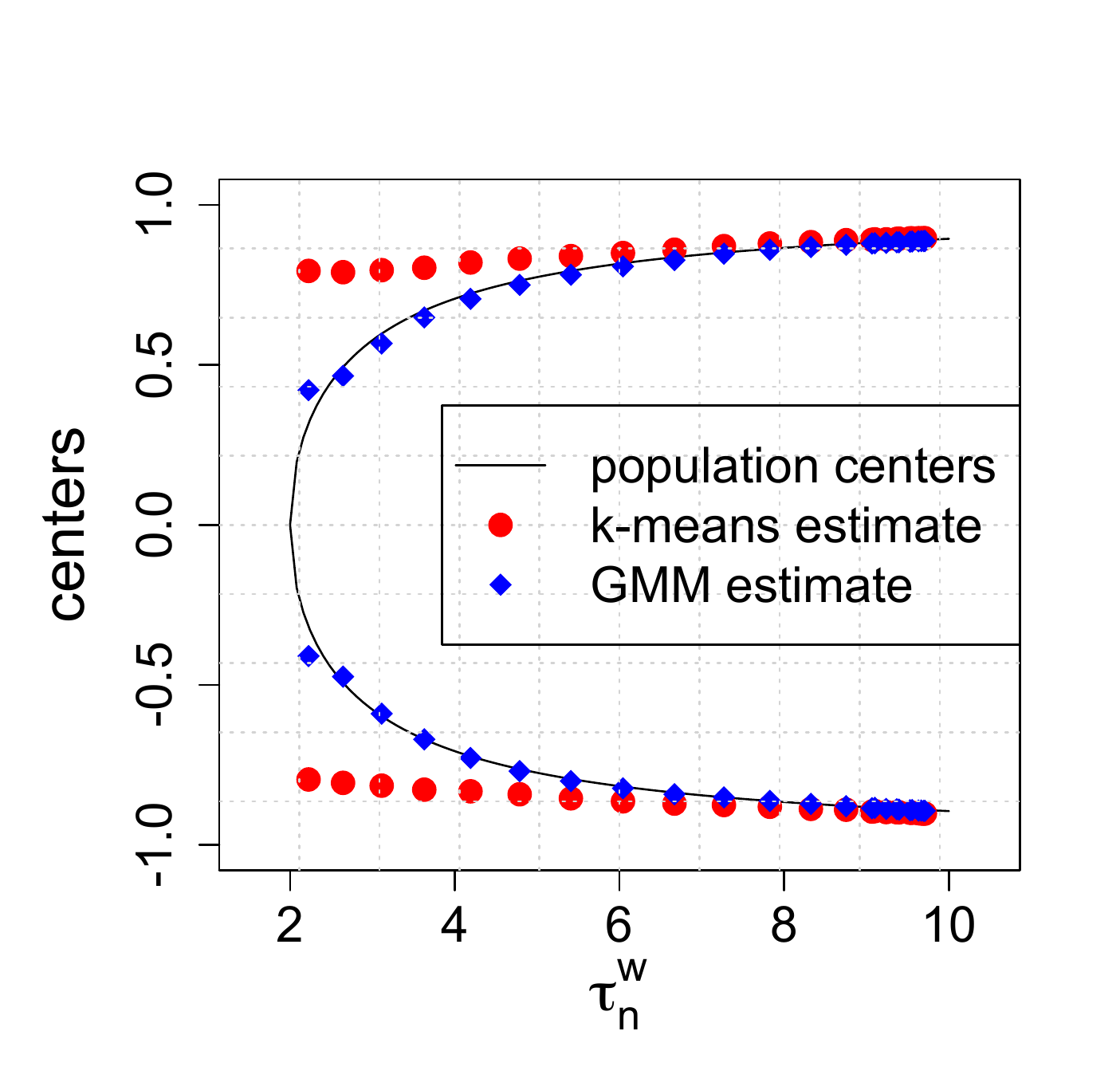}	
	\caption{Centers\label{fig:center}}
	\end{subfigure}
	\vspace{-0.3cm}
	\caption{The same  setting as in Figure \ref{fig:err}. (i). The asymptotic prediction is calculated via the limit in Proposition \ref{thm:main2}. The empirical value is computed over one repetition \textcolor{black}{without using any oracle information of the parameters}. (ii). Proposition \ref{gaussian:embed} shows that the first eigenvector is uninformative for clustering. We thus focus on the second one. There are two population centers that are calculated according to Proposition \ref{gaussian:embed}. The center estimates from $k$-means (``$k$-means estimate") and MLE under Gaussian mixture distribution (``GMM estimate") are computed over 5 repetitions.}
\end{figure}

Proposition \ref{thm:main2} together with Theorem \ref{thm:error} tells us that asymptotically the optimal weight achieving the minimum mis-clustering error maximizes the absolute eigenratio $|\lambda_K^{{\overrightarrow{w}}}|/|\lambda_{K+1}^{{\overrightarrow{w}}}|$. This motivates us to maximize the eigenratio to obtain the weight. Define the objective function $g({\overrightarrow{w}})\equiv(\lambda_K^{{\overrightarrow{w}}})^2/(\lambda_{K+1}^{{\overrightarrow{w}}})^2$. We aim to solve the optimization problem: $\max_{\overrightarrow{w} \in \mathcal{W}} g({\overrightarrow{w}})$.

When the eigenvalues $\lambda_K^{{\overrightarrow{w}}}$ and $\lambda_{K+1}^{{\overrightarrow{w}}}$ are simple (in magnitude), it is well known that they are differentiable at $A^{\overrightarrow{w}}$ and admit closed-form gradients \citep{magnus1985differentiating}. Using the chain rule,  we can derive the gradient of $g({\overrightarrow{w}})$: for each $\ell \in [L]$
\begin{align*}
\frac{\partial g({\overrightarrow{w}})}{\partial w_{\ell}}&=2\lambda_K^{{\overrightarrow{w}}}({\lambda_{K+1}^{{\overrightarrow{w}}}})^{-3}(\lambda_{K+1}^{{\overrightarrow{w}}}\nabla\lambda_K^{{\overrightarrow{w}}}-\lambda_K^{{\overrightarrow{w}}}\nabla\lambda_{K+1}^{{\overrightarrow{w}}}) \\
&=2\lambda_K^{{\overrightarrow{w}}}({\lambda_{K+1}^{{\overrightarrow{w}}}})^{-3}(\lambda_{K+1}^{{\overrightarrow{w}}}u_K^TA^{(\ell)}u_K-\lambda_{K}^{{\overrightarrow{w}}}u_{K+1}^TA^{(\ell)}u_{K+1}),
\end{align*}
where $u_K$ and $u_{K+1}$ are the eigenvectors associated with the eigenvalues $\lambda_K^{{\overrightarrow{w}}}, \lambda_{K+1}^{{\overrightarrow{w}}}$ respectively. Hence, we can perform projected gradient descent to update ${\overrightarrow{w}}$ using 
\begin{equation}
\label{gradient:form}
{\overrightarrow{w}}_{t+1}=\mathcal{P}_{\mathcal{W}}\left({\overrightarrow{w}}_{t}+\gamma_t\nabla g({\overrightarrow{w}}_{t})\right), 
\end{equation}
where $\mathcal{P}_{\mathcal{W}}(\cdot)$ denotes the projection onto unit simplex; $\gamma_t=\gamma_0/(1+rt)$ is the learning rate that decays with time, $r$ is the decay rate and $t$ is the number of iterations. The above update is only feasible when $\lambda_K^{{\overrightarrow{w}}}$ and $\lambda_{K+1}^{{\overrightarrow{w}}}$ are simple. Nevertheless, we found empirical evidence that $A^{\overrightarrow{w}}$ does not have repeated eigenvalues (in magnitude) for a wide range of $\overrightarrow{w}$. In fact, it has been proved that certain types of random matrices have simple spectrum with high probability \citep{tao2017random, luh2018sparse}. For completeness, to handle the rare scenario when $\lambda_{K}^{{\overrightarrow{w}}}$ or $\lambda_{K+1}^{{\overrightarrow{w}}}$ is not simple, we resort to coordinate descent update with one-dimensional line search. Whenever $\lambda_{K}^{{\overrightarrow{w}}}$ and $\lambda_{K+1}^{{\overrightarrow{w}}}$ become simple at the current update, the projected gradient descent is resumed. We implement the algorithm with random initialization. To achieve a better convergence, we run it independently multiple times and choose the output weight that gives the largest value of $g({\overrightarrow{w}})$. The method is summarized as Algorithm \ref{alg:ratio}.

\begin{algorithm}[htb] 
\caption{Spectral clustering with maximal eigenratio [SCME].}\label{alg:ratio}
\begin{algorithmic}[1] 
\REQUIRE 
$L$ layers of adjacency matrices $[A^{(l)}, l=1,\cdots,L]$, the number of communities $K$, initial learning rate $\gamma_0$, decay rate $r$,  maximum number of iterations $T$, number of random initializations $M$, and the precision parameter $\epsilon_0$.
\ENSURE 
$\hat{\overrightarrow{w}}$ and the community estimate $\hat{\overrightarrow{c}}$ by apply spectral clustering on $A^{\hat{\overrightarrow{w}}}$.
\STATE Initialize $m=1$.
\WHILE {$m\leq M$}
\STATE Obtain random initialization and denote it by $\hat{\overrightarrow{w}}_{\rm old}$. Set $\epsilon = \epsilon_0+1$ and $t=1$.
 \WHILE {$\epsilon>\epsilon_0$ and $t\leq T$,}
 \STATE Compute the update $\hat{\overrightarrow{w}}_{\rm new}$ 
   $
   \begin{cases}
   \text{using}~ \eqref{gradient:form}, & \text{if}~ \lambda_K^{{\overrightarrow{w}_{\rm old}}} ~\text{and}~ \lambda_{K+1}^{{\overrightarrow{w}_{\rm old}}} ~\text{are simple}, \\
   \text{via coordinate descent}, & \text{otherwise.} \\
   \end{cases}
   $
   \STATE Assign $\epsilon \leftarrow \|\hat{\overrightarrow{w}}_{\rm old}-\hat{\overrightarrow{w}}_{\rm new}\|$, $\hat{\overrightarrow{w}}_{\rm old}\leftarrow \hat{\overrightarrow{w}}_{\rm new}$, and $t \leftarrow t+1$. 
 \ENDWHILE
 \STATE Set $\hat{\overrightarrow{w}}_{\rm new}^m = \hat{\overrightarrow{w}}_{\rm new}$ and $m \leftarrow m+1$.
 \ENDWHILE
 \STATE Set $\hat{\overrightarrow{w}}\leftarrow \hat{\overrightarrow{w}}_{\rm new}^{m^*}$ where 
 $m^* = \arg\max_m g(\hat{\overrightarrow{w}}_{\rm new}^m)$.
\end{algorithmic}
\end{algorithm}

\begin{remark}
This second adaptive layer aggregation method has an intuitive explanation as well. For networks with $K$ communities, the ratio between the last informative eigenvalue $\lambda_K$ and the first noisy eigenvalue $\lambda_{K+1}$ is a reasonable measure of the community structure signal strength. Such intuition is well pronounced under the balanced MPPM: the maximization of the eigen-ratio leads to the optimal layer aggregation. Extensive numerical studies in Section \ref{numeric:exp} will further demonstrate the robustness and effectiveness of the method working beyond balanced MPPM. 
\end{remark}

\subsection{$k$-means clustering versus Gaussian mixture model clustering}\label{two:vs}

We have presented two novel methods tailored for adaptive layer aggregation. We now turn to discuss the spectral clustering step of the two-step framework that is introduced in Section \ref{two:step}. Specifically, we will provide convincing evidence to support clustering using Gaussian mixture models (GMM) as a substitute for $k$-means in spectral clustering. Towards this goal, let $U\in \mathcal{R}^{n\times K}$ be the eigenvector matrix of which the $i^{th}$ column is the eigenvector of $A^{\overrightarrow{w}}$ associated with the $i^{th}$ largest (in magnitude) eigenvalue.
\begin{proposition}
\label{gaussian:embed}
Under the same conditions as in Theorem \ref{thm:error}, there exists an orthogonal matrix $\mathcal{O} \in \mathcal{R}^{K \times K}$ such that for any pair $i, j \in [n]$ conditioning on the community labels $c_i, c_j$, it holds that as $n\rightarrow \infty$
\begin{align*}
\begin{pmatrix}
\sqrt{n}\mathcal{O}U^Te_i \\
\sqrt{n}\mathcal{O}U^Te_j
\end{pmatrix}
\overset{d}{\rightarrow} \mathcal{N}(\mu, \Sigma), \quad 
\mu=
\begin{pmatrix}
\mu^{(c_i)} \\
\mu^{(c_j)}
\end{pmatrix}
, \quad \Sigma=
\begin{pmatrix}
\Theta & 0 \\
0 & \Theta
\end{pmatrix}
,
\end{align*}
where $\{e_i\}_{i=1}^n$ is the standard basis in $\mathcal{R}^n$. We have omitted the dependence of $\mathcal{O},U$ on $n$ for simplicity. The mean and covariance matrix in the multivariate Gaussian distribution take the following expressions:
\begin{align*}
\mu^{(c_i)}=
\begin{pmatrix}
1 \\
\sqrt{\frac{K(\tau^{\overrightarrow{w}}_{\infty}-K)}{\tau^{\overrightarrow{w}}_{\infty}}}\nu_{c_i}
\end{pmatrix}
, ~~\mu^{(c_j)}=
\begin{pmatrix}
1 \\
\sqrt{\frac{K(\tau^{\overrightarrow{w}}_{\infty}-K)}{\tau^{\overrightarrow{w}}_{\infty}}}\nu_{c_j}
\end{pmatrix}
,~~ \Theta=\frac{K}{\tau^{\overrightarrow{w}}_{\infty}}\cdot
\begin{pmatrix}
0 & 0 \\
0 &  I_{K-1}
\end{pmatrix}
.
\end{align*}
Here, the $\mathcal{V}_{K\times K-1}=[\nu_1,\nu_2,\cdots,\nu_K]^T$ such that the  $K\times K$ matrix
$
[1/\sqrt{K} 1_{K} , \mathcal{V}_{K\times K-1}]
$
is orthogonal, where $1_K$ is a length-$K$ vector of 1.
\end{proposition}

\begin{remark}
\label{weak:dep}
We believe that it is possible to have a nontrivial generalization of our analysis to obtain similar convergence results under more general multi-layer stochastic block models. However, the limiting Gaussian mixture distribution will be much more complicated. Some limiting results have been obtained under single-layer stochastic block models in \citet{tang2018limit}. \textcolor{black}{\cite{levin2017central} derives the central limiting theorem for an omnibus embedding of multiple random graphs for graph comparison inference.} However, as explained in Remarks \ref{asym:one} and \ref{asym:two}, we are considering a more challenging asymptotic regime, which requires notably different analytical techniques. Given that $k$-means is invariant under scaling and orthogonal transformation, the vectors $\{\sqrt{n}\mathcal{O}U^Te_i\}_{i=1}^n$ can be considered as the input data points for $k$-means in spectral clustering. Proposition \ref{gaussian:embed} establishes that asymptotically all the spectral embedded data points follow a Gaussian mixture distribution, and they are pairwise independent. These properties turn out to be sufficient for the convergence result of $k$-means under classical i.i.d. settings in \citet{pollard1981strong} to carry over to the current case. With the convergence of estimated centers of $k$-means, the mis-clustering error in Theorem \ref{thm:error} can be derived in a direct way. See the proof of Theorem 1 for details. 
\end{remark}

\begin{remark}
\label{kmean:argue}
Given a sample of independent observations from a mixture distribution: $x \mid c=k \sim f(x;\theta_k)$, it is known that the optimal \emph{asymptotic} mis-clustering error is 
\begin{align*}
\min_{g}{\rm pr}(g(x)\neq c)={\rm pr}(g^*(x)\neq c), ~\mbox{where}~g^*(x)= \underset{k}{\operatorname{argmax}}~{\rm pr}(c=k \mid x).
\end{align*}
It is straightforward to verify that the \emph{asymptotic} error in Theorem \ref{thm:error} coincides with the above optimal error when the mixture distribution is the Gaussian mixture distribution given in Proposition \ref{gaussian:embed}. Combining this fact with Remark \ref{weak:dep}, we can conclude that $k$-means achieves the best possible outcome and is the proper clustering method to use in spectral clustering (if the conditions of Theorem \ref{thm:error} hold). However, we would like to emphasize that the $k$-means in spectral clustering, in general, does not consistently estimate the population centers of the Gaussian mixture distribution. Such a phenomenon is well recognized in the classical i.i.d. scenario (see for example \citet{bryant1978asymptotic}). It continues to occur in the present case. A formal justification of the statement can be found in Lemma 1 of the appendix. Some supportive simulation results are shown in Figure \ref{fig:center}.  Given the inconsistency of $k$-means for estimating the population centers, it is intriguing to ask why $k$-means is able to obtain the optimal mis-clustering error. This is uniquely due to population centers' symmetry and the simple structure of the covariance matrix in this specific Gaussian mixture distribution. Simple calculations show that the $k$-means error will be optimal as long as the estimated centers (ignoring the irrelevant first coordinate) after a common scaling converge to the population centers. This is verified numerically in Figure \ref{fig:center}: at each $\tau^{\overrightarrow{w}}_n$, the two center estimates from the $k$-means will be compatible with the population centers after a common shrinkage. 
\end{remark}

\begin{figure}[t]
\centering
\begin{subfigure}{.33\textwidth}
  \includegraphics[width=.85\linewidth]{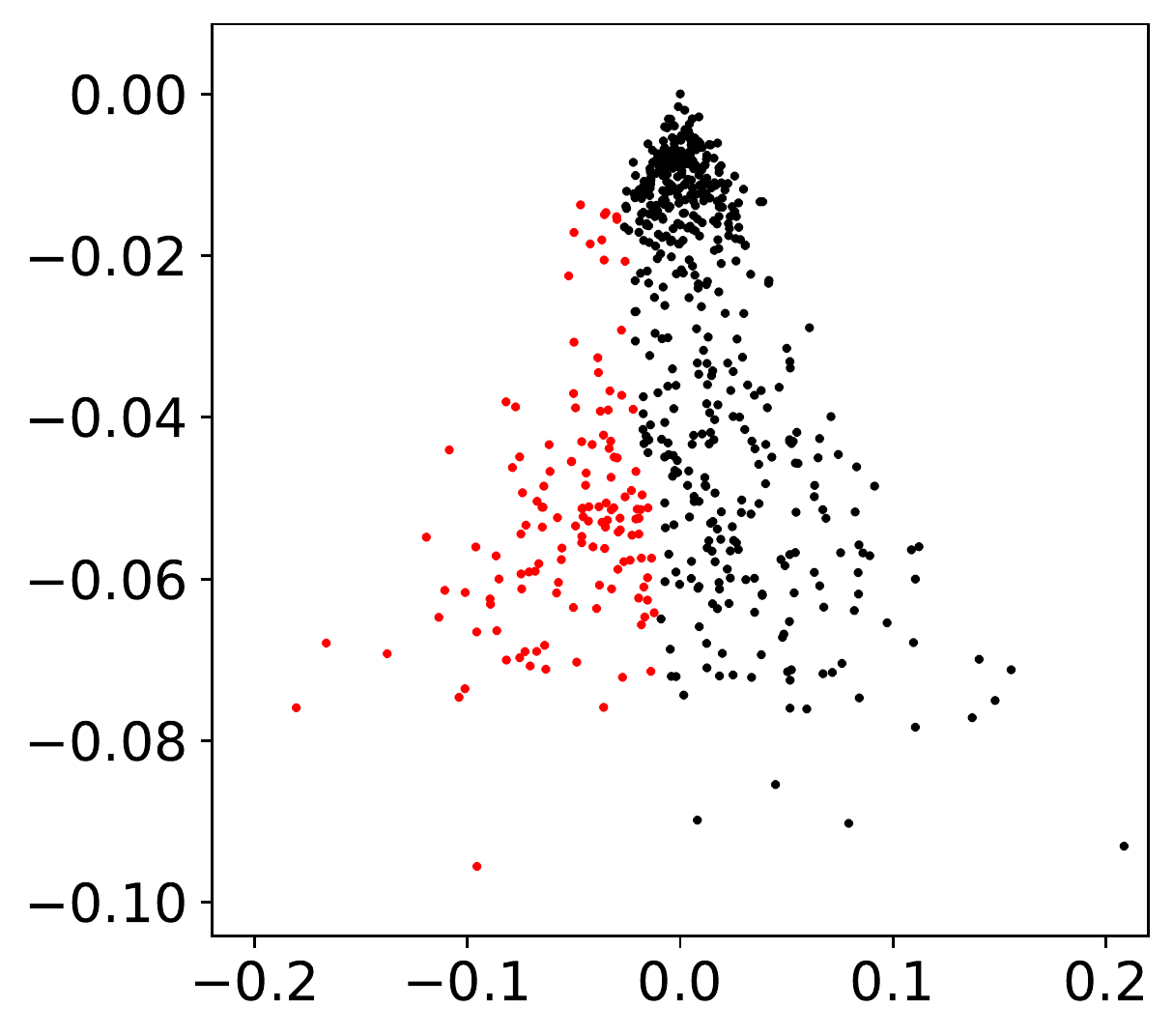}  
  \vspace{-0.5\baselineskip} 
  \caption{$k$-means}
  \label{fig:sub-km}
\end{subfigure}
\hspace{-2em}%
\begin{subfigure}{.33\textwidth}
  \includegraphics[width=.85\linewidth]{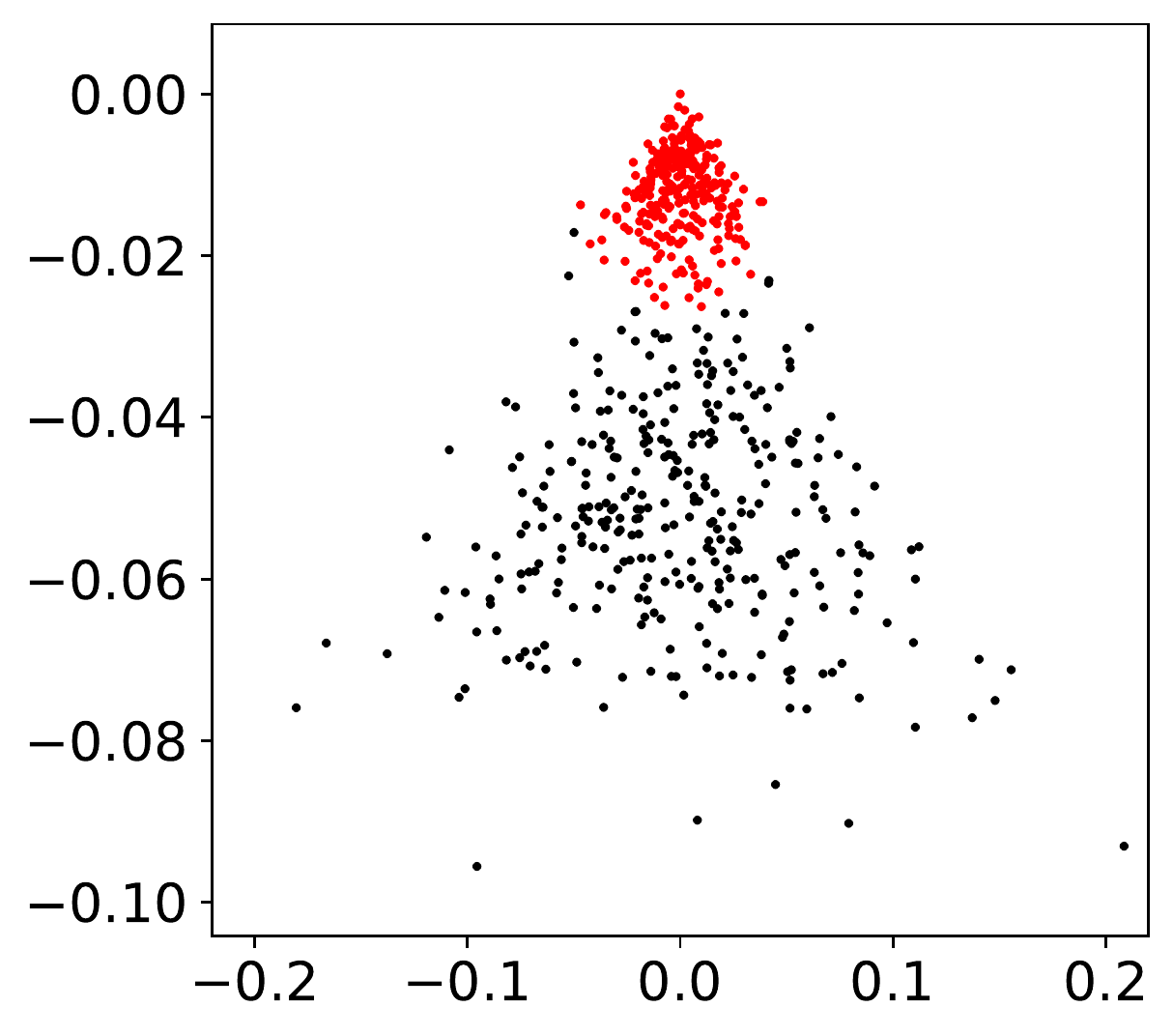}  
  \vspace{-0.5\baselineskip} 
  \caption{GMM} 
  \label{fig:sub-gm}
\end{subfigure}
\hspace{-2em}%
\begin{subfigure}{.33\textwidth}
  \includegraphics[width=.85\linewidth]{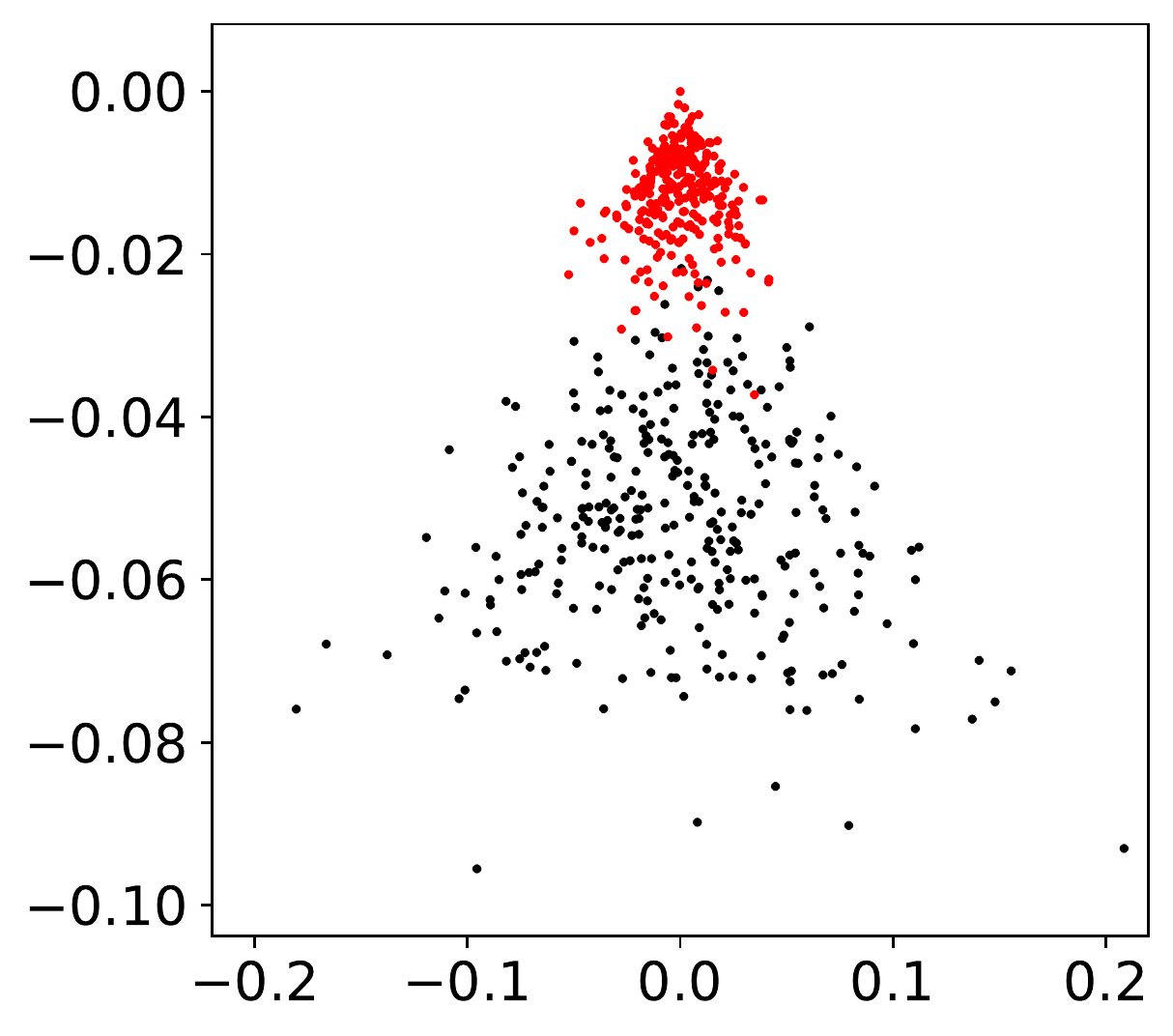}  
  \vspace{-0.5\baselineskip} 
  \caption{Ground truth} 
  \label{fig:sub-true}
\end{subfigure}
\vspace{-0.6\baselineskip} 
\caption{Consider a balanced multi-layer stochastic block model with $K=2, L=1, n=600, \Omega^{(1)}_{11}=0.053, \Omega^{(1)}_{12}=0.011, \Omega^{(1)}_{22}=0.016$. The spectral embedding in $\mathcal{R}^2$ along with the clustering results of $k$-means, GMM, and ground truth is shown. }
\label{fig:eigen}
\end{figure}

In light of Proposition \ref{gaussian:embed}, it is appealing to consider using Gaussian mixture model (GMM) clustering as an alternative in the spectral clustering, since the spectral embedded data follows a Gaussian mixture distribution asymptotically. Indeed, Figures \ref{fig:center} and \ref{fig:err} from a simulation study have demonstrated that GMM clustering is able to recover the population centers and obtain the optimal mis-clustering error. Therefore, both $k$-means and GMM clustering are optimal clustering methods that can be implemented in spectral clustering under balanced multi-layer planted partition models. However, as discussed in Remark \ref{kmean:argue}, it is the special parameter structures in the Gaussian mixture distribution that make $k$-means attain optimal error even though it is inconsistent for estimating population centers. Under a general multi-layer stochastic block model, it is likely that the limiting Gaussian mixture distribution (if exists) will have different parameter configurations so that the inconsistency of $k$-means eventually degrades the clustering performance. In contrast, GMM clustering specifies the asymptotically correct model and is able to better capture various shapes of embedded data. We  therefore expect it to outperform $k$-means in general settings. A quick numerical comparison is shown in Figure \ref{fig:eigen}. We observe that as the model deviates from balanced MPPM, the shapes of mixture components can be fairly different. The $k$-means fails to identify the shapes, resulting in poor clustering performance, while GMM provides an excellent fit thus gaining significant improvement in clustering. We provide more numerical experiments regarding the comparison in Sections \ref{sec:sim} and \ref{sec:real}.

\section{Numerical Experiments}
\label{numeric:exp}

We have proposed iterative spectral clustering (ISC) and spectral clustering with maximal eigenratio (SCME) for community detection. With the use of $k$-means or GMM clustering in spectral clustering, we in fact have four different methods: ISC with $k$-means clustering [ISC\_km], ISC with GMM clustering [ISC\_gm], SCME with $k$-means clustering [SCME\_km], and SCME with GMM clustering [SCME\_gm]. In this section, we present a systematic numerical study of the community detection performance of the four methods. Moreover, we compare our methods with three standard spectral methods in the literature, which are based on mean adjacency matrix [Mean adj.] \citep{han2015consistent}, aggregate spectral kernel [SpecK] \citep{paul2017consistency} and module allegiance matrix [Module alleg.] \citep{braun2015dynamic}. We use the adjusted rand index (ARI) \citep{hubert1985comparing} to evaluate the community detection performance of all the methods. ARI is a common measure of the similarity between two data clusterings. It is bounded by $1$, with the value of 1 indicating perfect recovery while 0 implying the estimation is no better than a random guess \citep{steinhaeuser2010identifying}. \textcolor{black}{Finally, to verify that our algorithms can reach (nearly) optimal layer aggregation under balanced MPPM, we include the oracle two-step procedure [Grid Search] which uses the weight that minimizes the empirical ARI (over a grid).}

\subsection{Simulation results}\label{sec:sim}

In the simulations, we consider four different cases corresponding to balanced MPPM, imbalanced MPPM, balanced MSBM, and imbalanced MSBM, respectively. Our methods are motivated by the asymptotic analysis under balanced MPPM. We thus would like to verify their effectiveness under the assumed models. In addition, we aim to study to what extent our methods work well for more general models. Each experiment is repeated 100 times. 

\begin{enumerate}

\item \emph{Balanced MPPM case}. For a balanced model $A^{[L]}\sim {\rm MPPM}(p^{[L]},q^{[L]},  \overrightarrow{\pi})$, we consider $p^{[L]}\equiv (p^{(1)},\cdots, p^{(L)})^T=c_\rho\frac{\log n}{n} \overrightarrow{p}, q^{[L]}\equiv (q^{(1)},\cdots, q^{(L)})^T=c_\rho\frac{\log n}{n} \overrightarrow{q}$. We vary values of the five parameters $n, K, L, c_\rho, \overrightarrow{q}$ in the model to investigate settings with different sample sizes, numbers of communities, numbers of layers, sparsity levels, and connectivity probabilities. Values of model parameters are summarized in Table~\ref{tab:ppm}.

\begin{table}[!h]
\begin{center}
	
\begin{threeparttable}[t]

\caption{Balanced \textsc{MPPM}. The notation ``s-t" denotes that a parameter is changed over the range $[s,t]$.}{%
\begin{tabular}{ccccccc}
\
	Case & $n$ & $K$ & $L$ & $c_\rho$ & $\overrightarrow{p}$ & $\overrightarrow{q}$\\ 
	1a 	&	600	&	2	&	2	&	1.5	&	(4, 4) & (2, 0-4) \\ 
	1b	&	600	&	2-6	&	2	&	1.5	&	(4, 4) & (0, 3) \\ 
	1c	&	600	&	2	&	1-5	&	1.5	&	(4, $\cdots$, 4) & (0, 4, $\cdots$, 4)\\ 
	1d	&	600	&	2	&	2	&	.4-1.2	&	(4, 4) & (.5, 2.5)	\\ 
	1e	&	200-1000	&2	&2	& .6	&	(4, 4)	&	(1, 3)\\ 
\end{tabular}}
\label{tab:ppm}

\end{threeparttable}

\end{center}
\end{table}

\begin{figure}[!htb]
\centering
\vspace{-0.5em}%
\begin{subfigure}{.45\textwidth}
  \includegraphics[width=.9\linewidth]{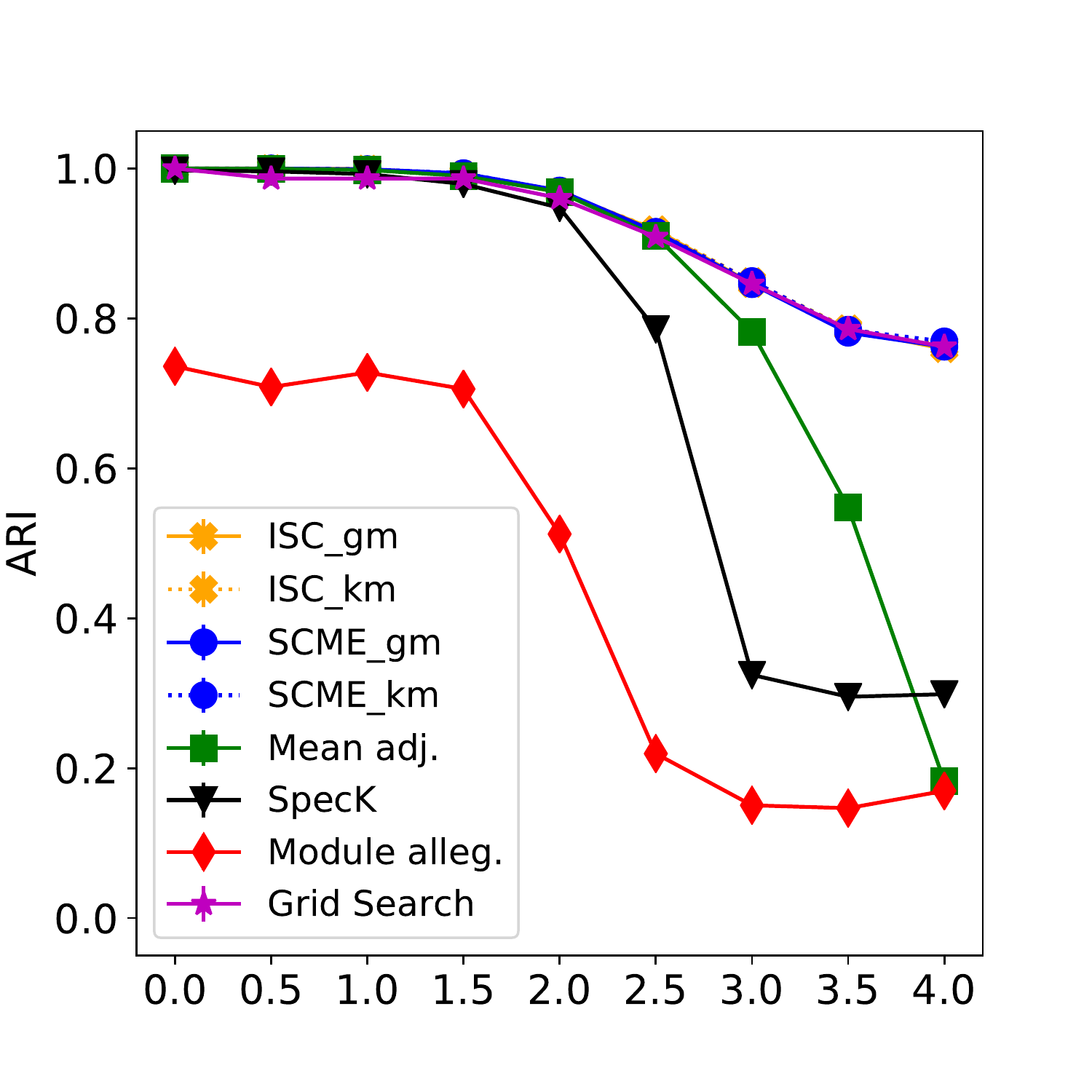} 
  \vspace{-0.8\baselineskip} 
  \caption{Case 1a: $\overrightarrow{q}$}
  \label{fig:sub-Omega}
\end{subfigure}
\begin{subfigure}{.45\textwidth}
  \includegraphics[width=.9\linewidth]{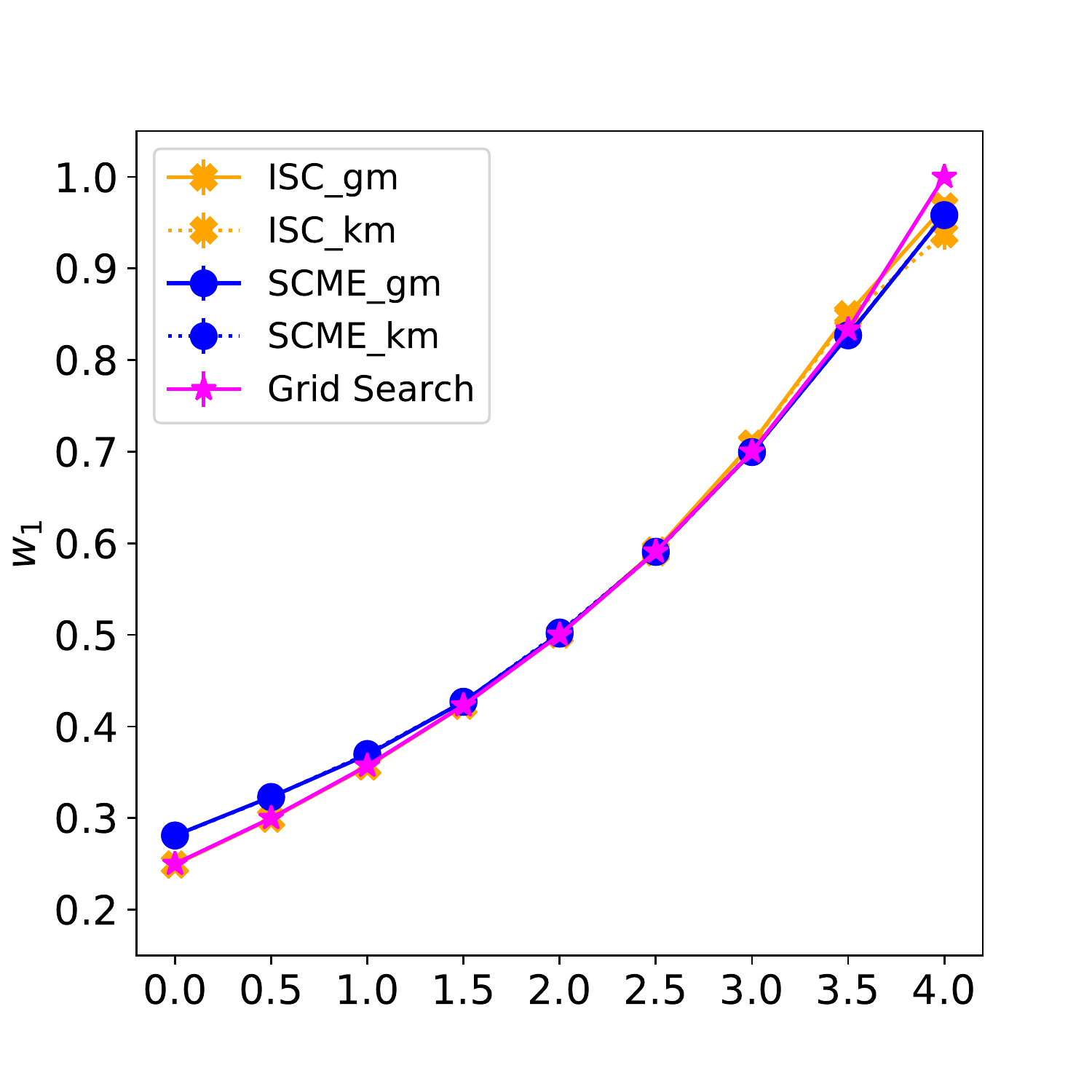}  
  \vspace{-0.3\baselineskip} 
  \caption{Case 1a: $\overrightarrow{q}$}
  \label{fig:sub-weight}
\end{subfigure}
\vspace{-0.3em}%
\begin{subfigure}{.45\textwidth}
  \includegraphics[width=.9\linewidth]{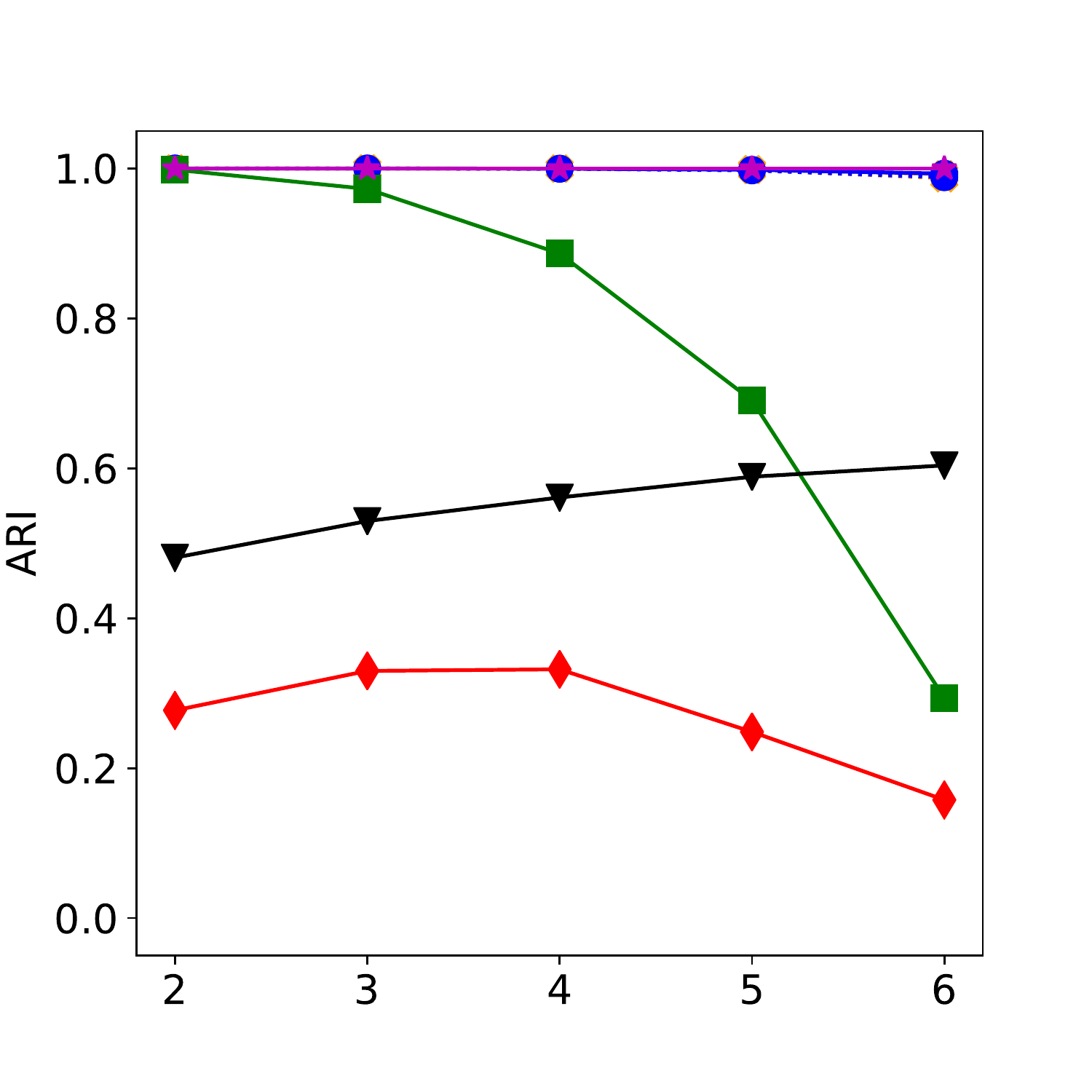}  
  \vspace{-0.3\baselineskip} 
  \caption{Case 1b: $K$}
  \label{fig:sub-K}
\end{subfigure}
\begin{subfigure}{.45\textwidth}
  \includegraphics[width=.9\linewidth]{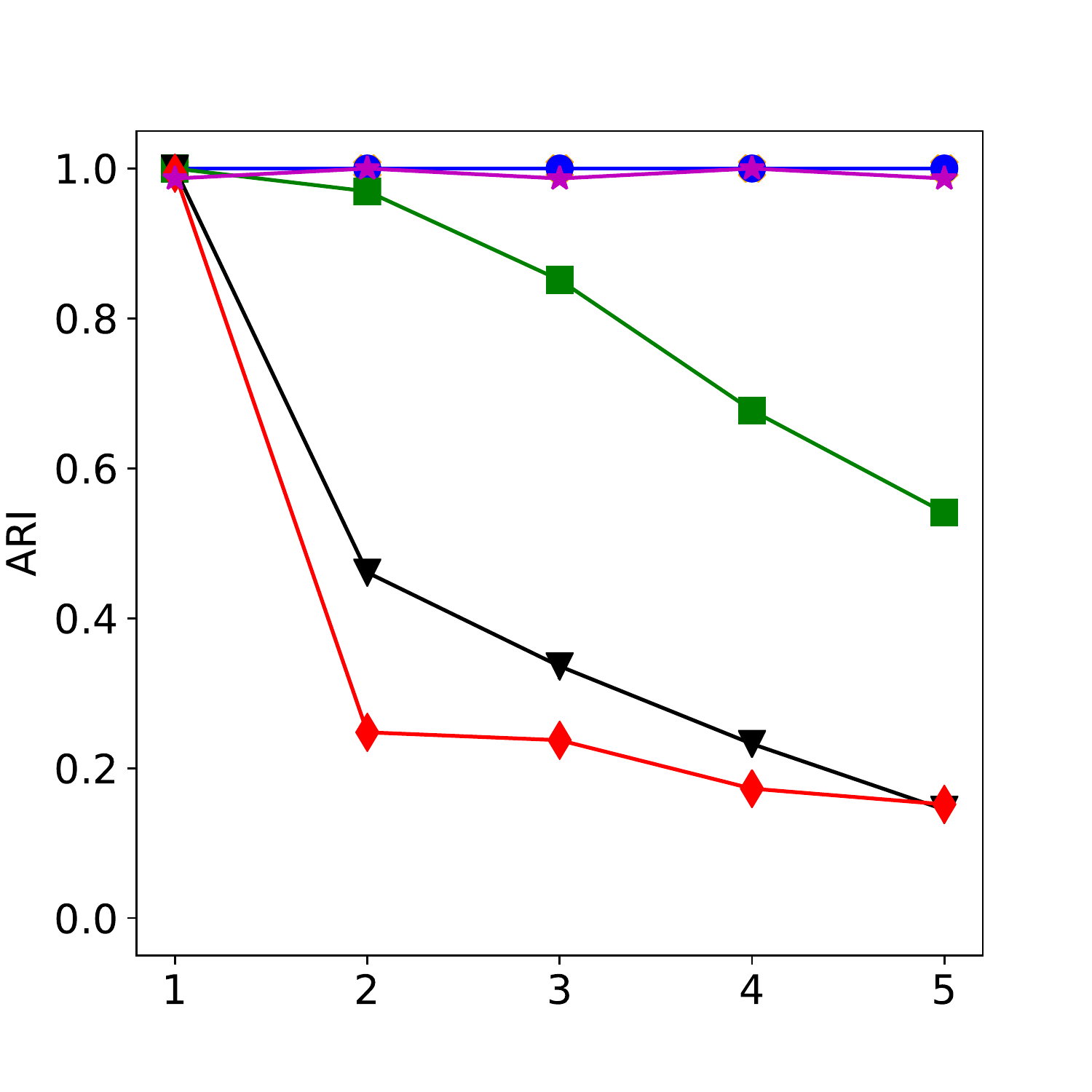}  
  \vspace{-0.3\baselineskip} 
  \caption{Case 1c: $L$}
  \label{fig:sub-L}
\end{subfigure}
\vspace{-0.3em}%
\begin{subfigure}{.45\textwidth}
  \includegraphics[width=.9\linewidth]{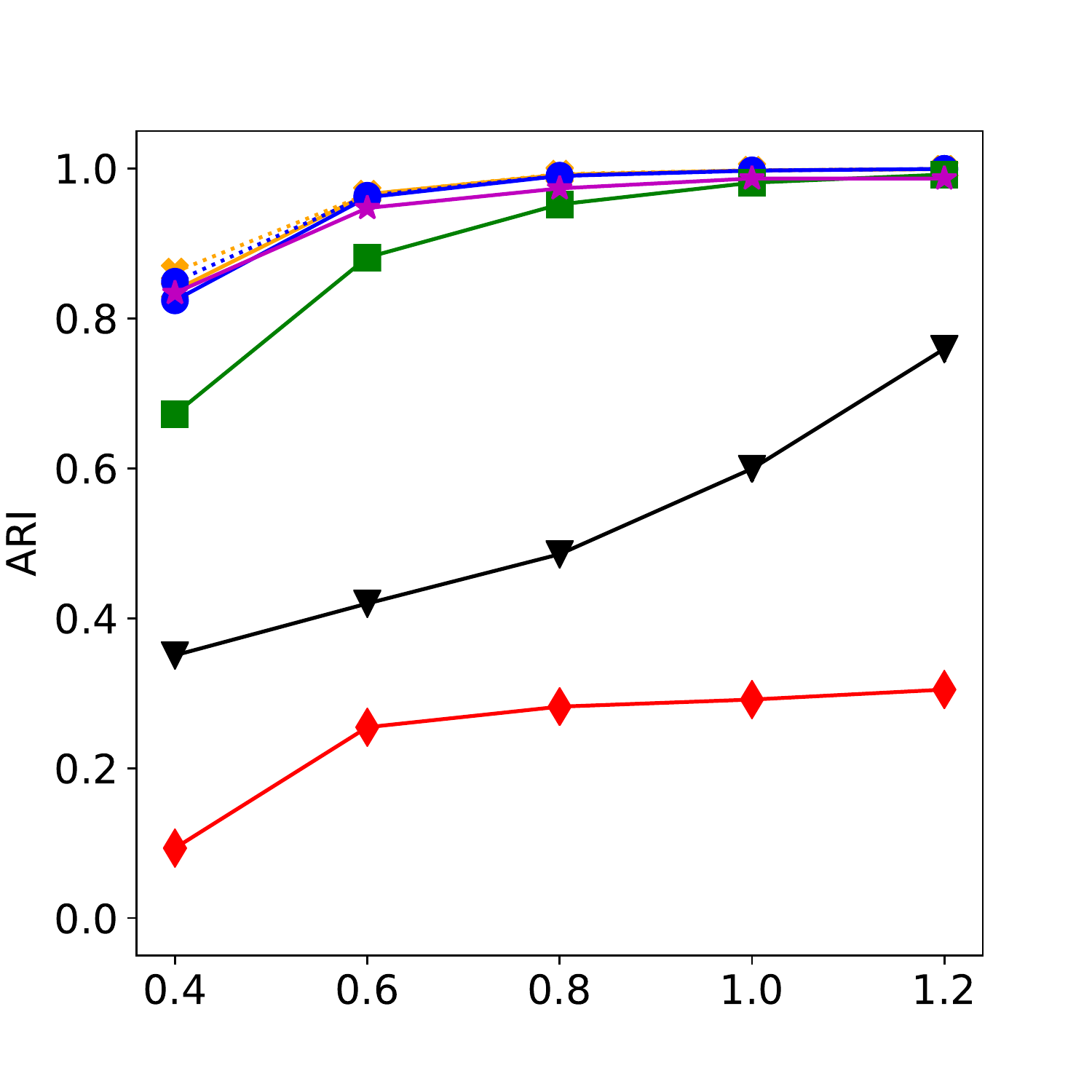} 
  \vspace{-0.3\baselineskip}  
  \caption{Case 1d: $c_\rho$}
  \label{fig:sub-rho}
\end{subfigure}
\begin{subfigure}{.45\textwidth}
  \includegraphics[width=.9\linewidth]{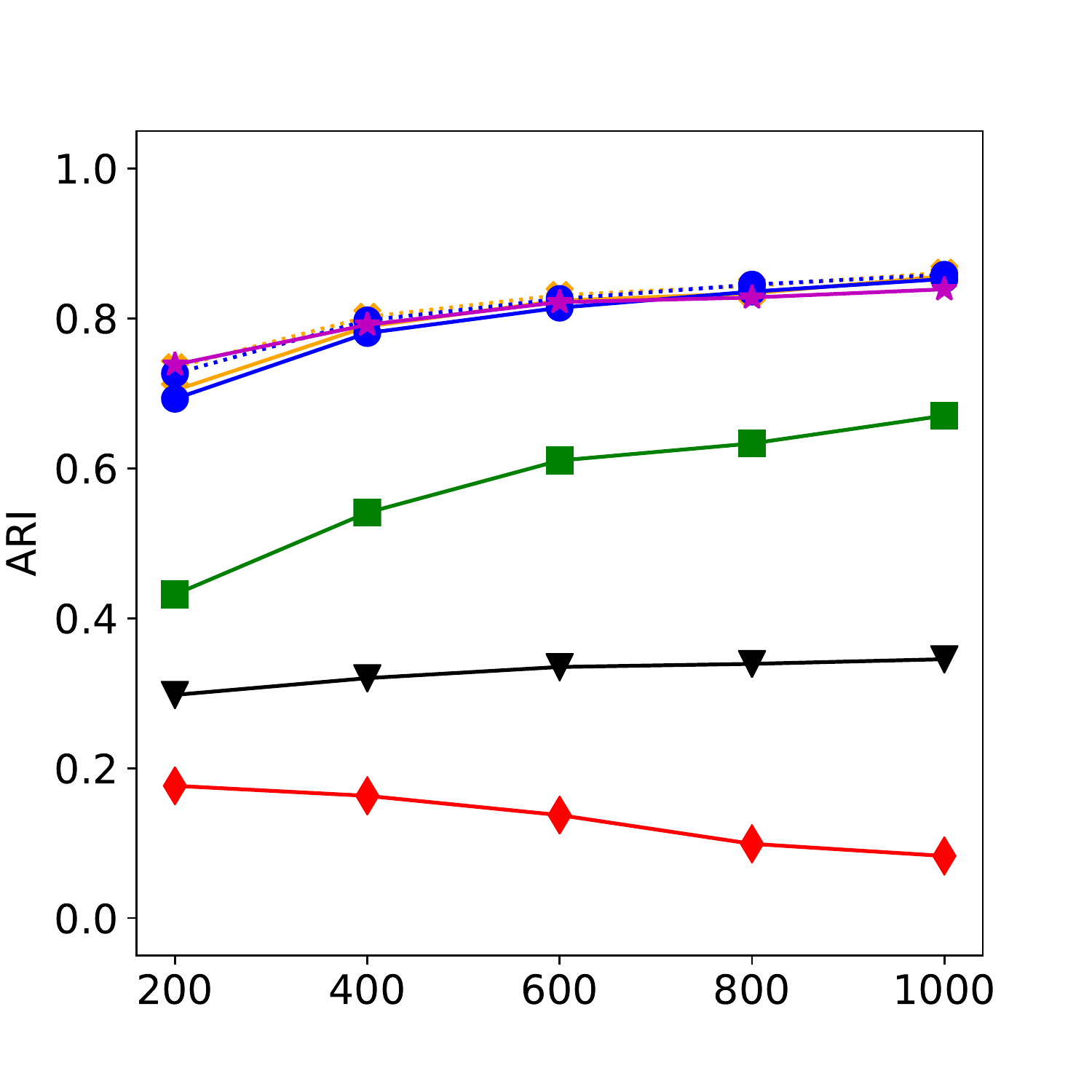} 
  \vspace{-0.3\baselineskip}  
  \caption{Case 1e: $n$}
  \label{fig:sub-n}
\end{subfigure}
\vspace{-0.2\baselineskip} 
\caption{Balanced \textsc{MPPM} case (maximum standard error in the figure is .022)}
\label{fig:ppm}
\end{figure}

Simulation results are presented in Figure~\ref{fig:ppm}. We observe that our four methods show competitive or superior performances compared with the other three consistently under Cases 1a-1e. This lends further strong support to the use of proposed methods. Moreover, the four methods yield uniformly comparable results across all the settings. Such phenomena are consistent with the discussion on the equivalence of $k$-means and GMM clustering under balanced MPPM in Section \ref{two:vs}. 

We now discuss each of the scenarios to shed more light on the performance of our methods. In Case 1a, we vary the between-community connectivity probability in the second layer while keeping all other parameters fixed. As this probability increases, the second layer becomes less informative for community detection, so the ARI decreases for all the methods. However, compared with the three existing methods, ISC and SCME are both robust to the increased noises in the second layer, thanks to the adaptive layer aggregation mechanism in the new methods. Figure \ref{fig:sub-weight} depicts the weight for the first layer selected by ISC\_gm, ISC\_km, SCME\_gm, and SCME\_km, along with the (empirical) optimal weights. It is clear that all the four methods successfully found the (nearly) optimal weights, thus being able to adaptively down-weigh the second layer when it becomes noisier. Case 1b changes the number of communities. As $K$ varies with everything else fixed, the effective information in each layer for every method tends to change. Take our methods for example. As can be verified from the optimal weight formula \eqref{eq:1}, a layer with a larger ratio of within-community to between-community connectivity probability should be weighted higher when the number of communities increases. Referring to Figure \ref{fig:sub-K}, we see that our methods provide stable and decent results due to adaptive layer aggregation, while the other three methods are comparatively sensitive to the change of $K$. In Case 1c, we fix the first layer as an informative layer and set other layers to be random noises.  Figure \ref{fig:sub-L} clearly shows that our methods detect communities effectively and are remarkably robust to added noise layers. In contrast, the three existing methods perform worse with more noise layers. Case 1d generates networks with different sparsity levels by simply scaling all the connectivity probabilities in the model. As the scaling $c_\rho$ is getting larger, the network becomes denser and more informative, so all the methods perform better as seen in Figure \ref{fig:sub-rho}. Our methods outperform Mead adj. for sparse networks, even though they are similar in dense settings. Case 1e examines the performance under different sample sizes. We notice from Figure \ref{fig:sub-n} that our methods can be advantageous for moderately small-scale networks.

\item \emph{Imbalanced {\rm MPPM} case}. For an imbalanced $A^{[L]}\sim {\rm MPPM}(p^{[L]},q^{[L]},  \overrightarrow{\pi})$, the parameter $ \overrightarrow{\pi}$ characterizing the size of each community is of primary interest. Adopting the same notation from the balanced MPPM case, we list the parameter values in Table \ref{tab:imppm}. The results are shown in Figure \ref{fig:imppm}. In all scenarios, our four methods outperform other methods. ISC\_km (ISC\_gm) and SCME\_km (SCME\_gm) are comparable. One notably different observation from the balanced MPPM case is that GMM clustering leads to superior performances than $k$-means in both ISC and SCME. This validates our arguments regarding the comparison between $k$-means and GMM in Section \ref{two:vs}: when the model deviates from balanced MPPM, GMM clustering is expected to better account for the heterogeneity of spectral embedded data. As illustrated in Figure \ref{fig:sub-2a}, the improvement of GMM over $k$-means becomes more significant as the communities get more imbalanced. In Case 2b, we keep the communities imbalanced at a given level and change the sparsity level by varying $c_\rho$. Figure \ref{fig:sub-2b} shows that SCME is slightly better than ISC in some settings. Our methods outperform Mean adj. by a larger margin for sparser networks.

\begin{table}[t]
\begin{center}

\def~{\hphantom{0}}
\caption{Imbalanced \textsc{MPPM}}{%
\begin{tabular}{cccccccc}
	Case & $n$ & $K$ & $L$ & $c_\rho$ & $\overrightarrow{p}$ & $\overrightarrow{q}$ & $\overrightarrow{\pi}$\\ 
2a & 600 & 2 & 2  & 2 & (4, 4) & (2, 3.5) & (.25-.5, .75-.5) \\
2b   &600	&2 &	2 &1.5-2.7	 &(4, 4)&(2, 3.5) & (.3, .7) \\
\end{tabular}}
\label{tab:imppm}
	
\end{center}

\end{table}

\begin{figure}[!htb]
\centering
\begin{subfigure}{.45\textwidth}
  \includegraphics[width=.9\linewidth]{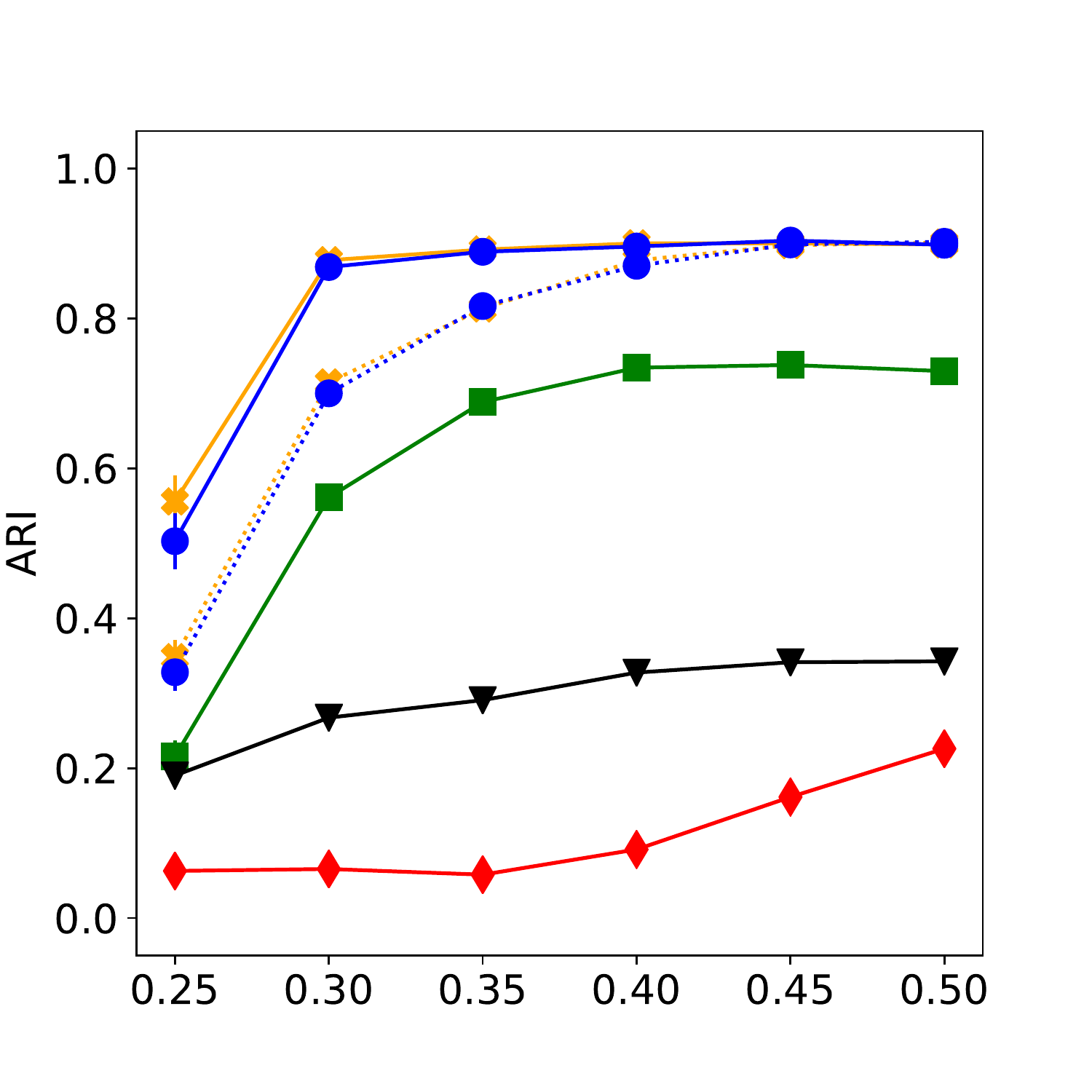}  
  \caption{Case 2a: $\pi_1$}
  \label{fig:sub-2a}
\end{subfigure}
\hspace{-2em}%
\begin{subfigure}{.45\textwidth}
  \includegraphics[width=.9\linewidth]{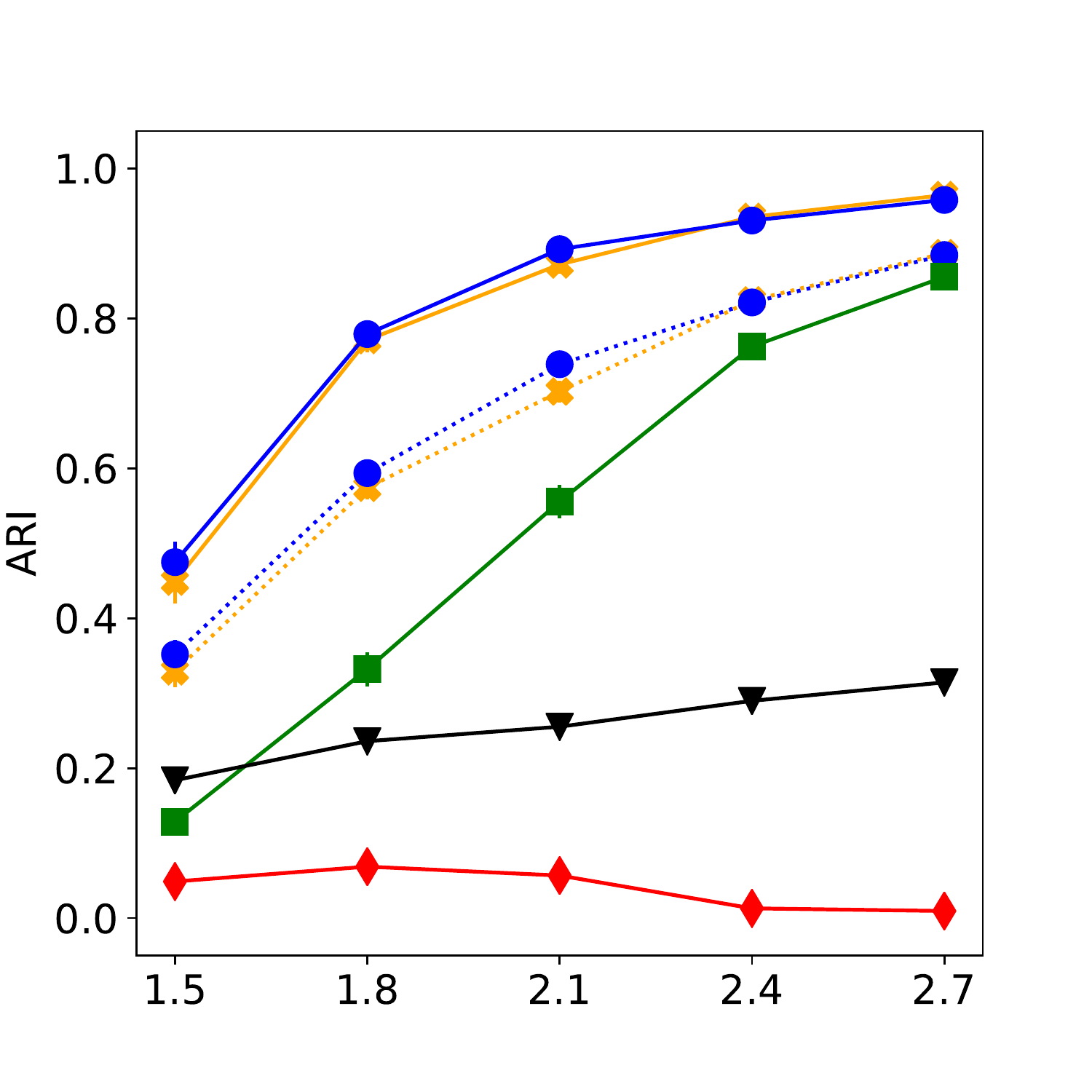}   
  \caption{Case 2b: $c_\rho$} 
  \label{fig:sub-2b}
\end{subfigure}
\vspace{-0.6\baselineskip} 
\caption{Imbalanced \textsc{MPPM} case (maximum standard error in the figure is .037)}
\label{fig:imppm}
\end{figure}

\item \emph{Balanced {\rm MSBM} case}. We set $\Omega^{[L]}=c_\rho\frac{\log n}{n}\bar{\Omega}$ in the balanced $A^{[L]}\sim\textsc{MSBM}(\Omega^{[L]},\overrightarrow{\pi})$. Table \ref{tab:sbm} summarizes the parameter values. We consider two different balanced MSBM scenarios, both varying the sparsity level by changing $c_\rho$.

\begin{table}[!htb]
\caption{Balanced \textsc{MSBM}}

\begin{center}
{%
\begin{tabular}{cccccc}
Case & $n$ & $K$ & $L$ & $c_\rho$ & $\bar{\Omega}$ \\  
\multirow{2}{*}{3a} & \multirow{2}{*}{600} & \multirow{2}{*}{2} & \multirow{2}{*}{2}
    & \multirow{2}{*}{.6-1.6} & \multirow{2}{*}{$\begin{pmatrix} 
5 & 2 \\
2 & 4 
\end{pmatrix},\begin{pmatrix} 
4 & 3.5 \\
3.5 & 5 
\end{pmatrix}$ }\\
&  & & & &  \\ \\
\multirow{3}{*}{3b} & \multirow{3}{*}{600} & \multirow{3}{*}{3} & \multirow{3}{*}{5}
    & \multirow{3}{*}{.5-3} & \multirow{3}{*}{$\begin{pmatrix} 
9 & 2 & 2 \\
2 & 2 & 2 \\
2 & 2 & 9 
\end{pmatrix},\begin{pmatrix} 
2 & 2 & 2 \\
2 & 4 & 2 \\
2 & 2 & 2
\end{pmatrix}, 2J_3, 2J_3, 2J_3$}\\ 
&  & & &  &  \\ 
\end{tabular}
}
	
\end{center}

\label{tab:sbm}

\end{table}

\begin{figure}[!htb]
\centering

\begin{subfigure}{.45\textwidth}
  \includegraphics[width=.9\linewidth]{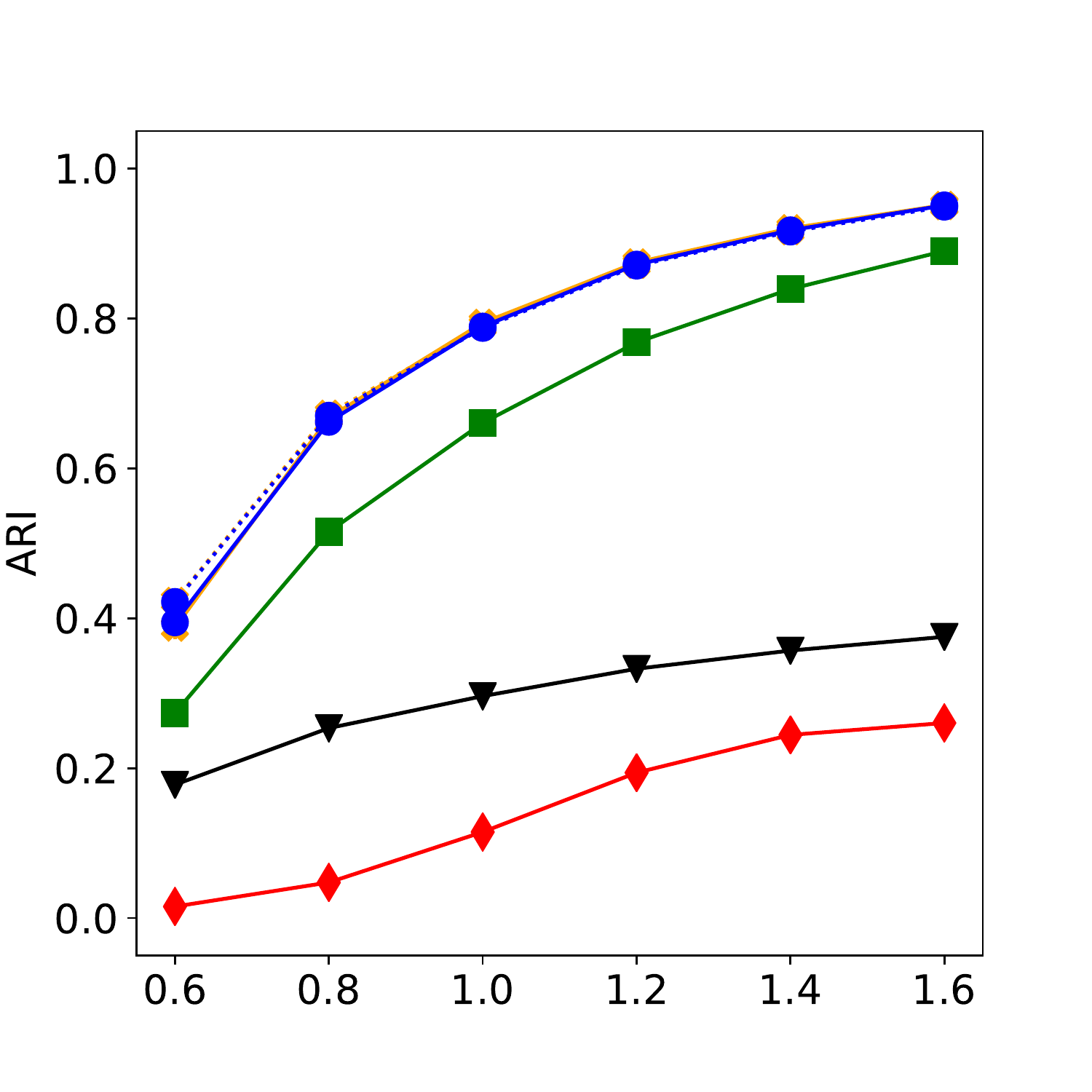}

  \caption{Case 3a: $c_\rho$}
  \label{fig:sub-3a}
\end{subfigure}
\begin{subfigure}{.45\textwidth}
  \includegraphics[width=.9\linewidth]{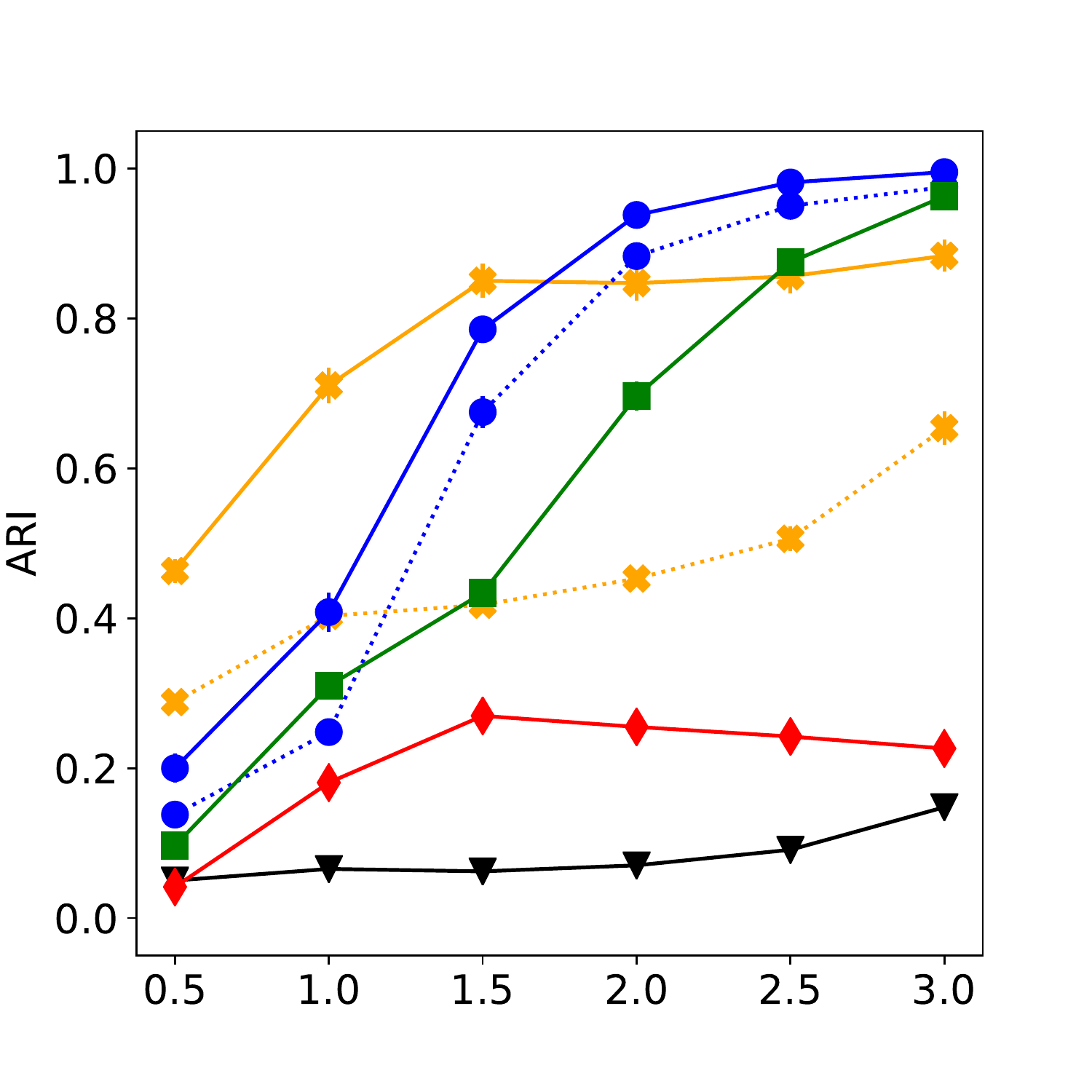}  
  \caption{Case 3b: $c_\rho$}
  \label{fig:sub-3b}
\end{subfigure}

\caption{Balanced \textsc{MSBM} case (maximum standard error in the figure is .027)}
\label{fig:sbm}
\end{figure}

In Case 3a, we consider a two-layer MSBM deviating moderately from MPPM.  As illustrated in Figure~\ref{fig:sub-3a}, the comparison results are similar to what we observe in the balanced MPPM case: our four methods perform alike and are uniformly better than the other three methods. In Case 3b, we make the MSBM adequately different from MPPM. Figure \ref{fig:sub-3b} shows the results. First, we see that for both ISC and SCME, GMM outperforms $k$-means. The same outcomes occur in the imbalanced MPPM case. We have well discussed the reason in Section \ref{two:vs}. Second, ISC\_gm (ISC\_km) works better than SCME\_gm (SCME\_km) in sparse settings while the latter wins in the dense cases. The result indicates that eigenratio is more robust to model misspecification as long as the network is sufficiently dense, while the weight update \eqref{weight:update} used in ISC is less variable when the network is sparse. We leave a thorough investigation of the distinct impact of sparsity level on ISC and SCME for future research. We also observe that SCME\_gm continues to outperform the three standard methods, and ISC\_gm does so in all the sparse settings.

\item \emph{Imbalanced {\rm MSBM} case}. We set $\Omega^{[L]}=c_\rho\frac{\log n}{n}\bar{\Omega}$ in the imbalanced $A^{[L]}\sim\textsc{MSBM}(\Omega^{[L]},\overrightarrow{\pi})$. We list the parameter values in Table \ref{tab:imsbm}, where for Case 4a we fix the second community's proportion. Results are shown in Figure \ref{fig:imsbm}. As expected, we observe again that GMM performs better than $k$-means for both ISC and SCME. Case 4a can be seen as an imbalanced continuation of Case 3b at a given $c_\rho$. Figure \ref{fig:sub-4a} illustrates the effectiveness of SCME\_gm and ISC\_gm compared to the three existing methods over a wide range of imbalanced community sizes. Case 4b is an imbalanced variant of Case 3b. We see from Figure \ref{fig:sub-4b} that the same comparison outcomes observed in Figure \ref{fig:sub-3b} for Case 3b remain valid in the imbalanced case. 

\begin{table}[!h]
\caption{Imbalanced \textsc{MSBM}.}
\begin{center}
{%
\begin{tabular}{ccccccc}
Case & $n$ & $K$ & $L$ & $c_\rho$ & $\bar{\Omega} $ & $\overrightarrow{\pi}$\\  
\multirow{2}{*}{4a} & \multirow{3}{*}{600} & \multirow{3}{*}{3} & \multirow{3}{*}{5}
    & \multirow{2}{*}{2}  & \multirow{3}{*}{$\begin{pmatrix} 
9 & 2 & 2 \\
2 & 2 & 2 \\
2 & 2 & 9 
\end{pmatrix},\begin{pmatrix} 
2 & 2 & 2 \\
2 & 4 & 2 \\
2 & 2 & 2
\end{pmatrix}, 2J_3, 2J_3, 2J_3$}&\multirow{1.5}{*}{(.17-.33, .33, .5-.34)}\\ 
\multirow{2}{*}{4b} &  & & &\multirow{2}{*}{.5-3}   &  &\multirow{2}{*}{(.25, .33, .42)} \\ &&&&&\\ \\ 
\end{tabular}
}
\end{center}
\label{tab:imsbm}
\end{table}

\begin{figure}[!htb]
\centering

\begin{subfigure}{.45\textwidth}
  \includegraphics[width=.9\linewidth]{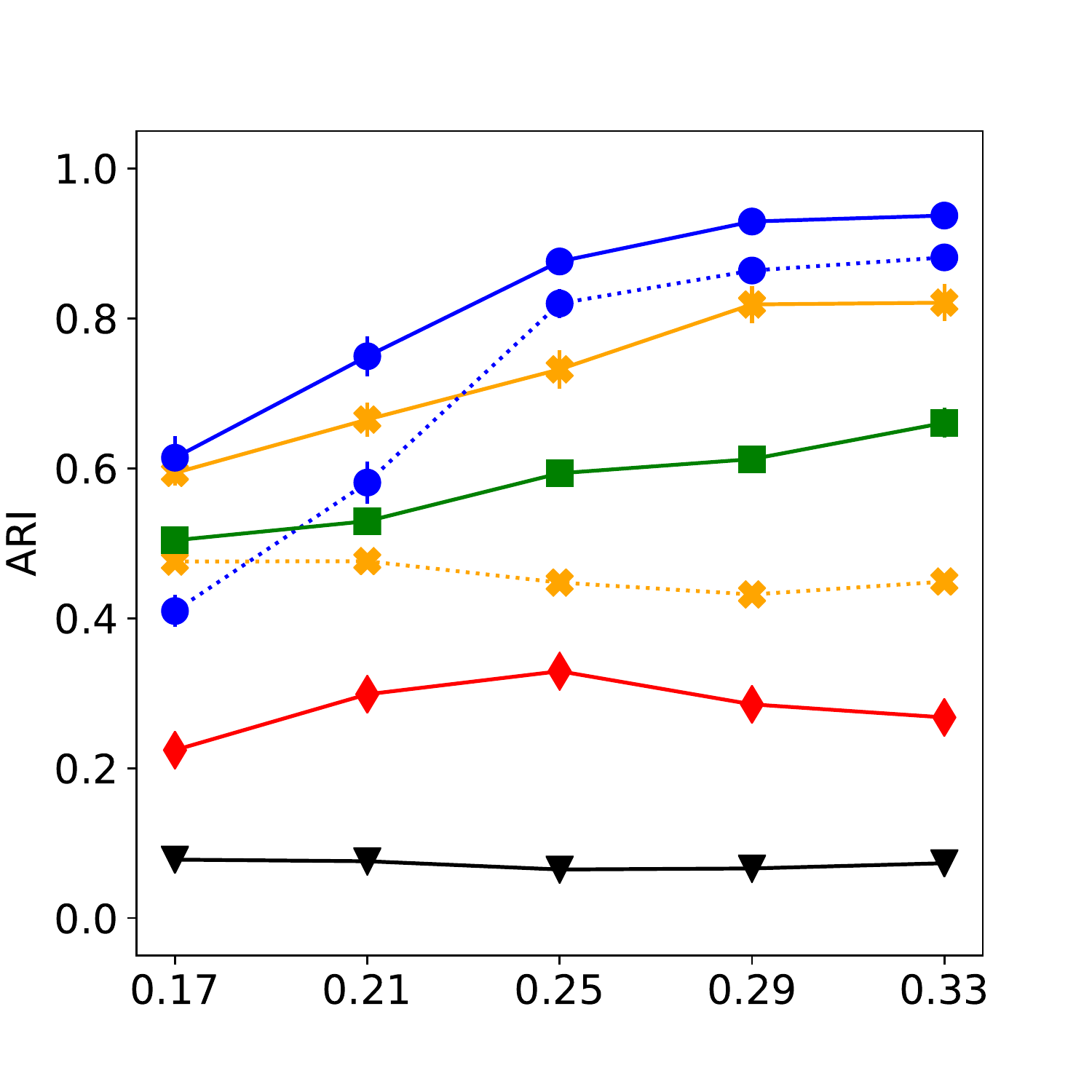}
  \vspace{-0.8\baselineskip}  
  \caption{Case 4a: $\pi_1$}
  \label{fig:sub-4a}
\end{subfigure}
\hspace{-2em}%
\begin{subfigure}{.45\textwidth}
  \includegraphics[width=.9\linewidth]{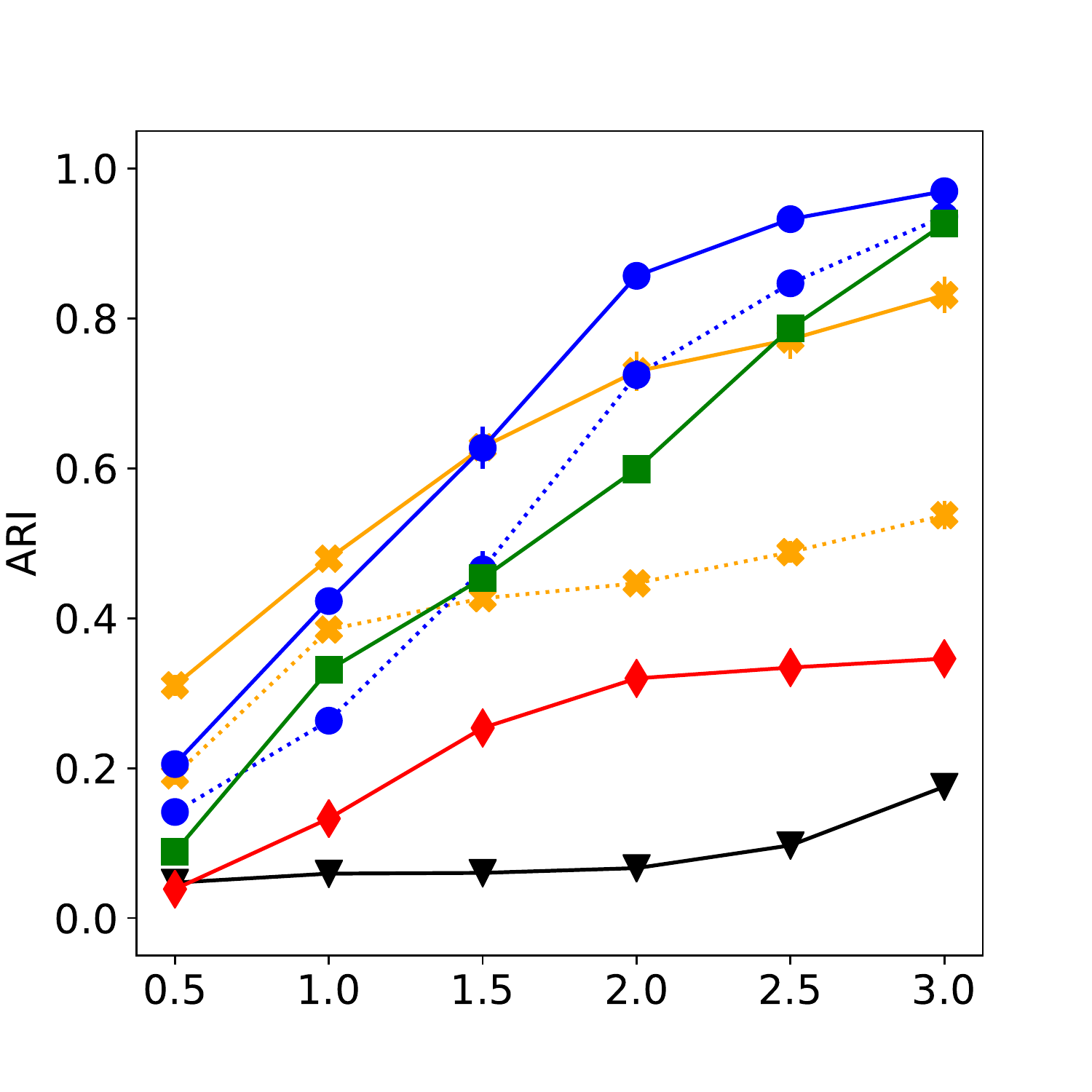}  
  \vspace{-0.8\baselineskip} 
  \caption{Case 4b: $c_\rho$}
  \label{fig:sub-4b}
\end{subfigure}
\vspace{-0.6\baselineskip} 
\caption{Imbalanced \textsc{MSBM} case (maximum standard error in the figure is .028)}
\label{fig:imsbm}
\end{figure}

\item \emph{Summary}. The simulations under balanced MPPM demonstrate that both ISC and SCME are able to select (nearly) optimal weight for layer aggregation and thus achieve superior community detection results, which is consistent with the asymptotic analysis presented in Section \ref{sec:def}. Additional simulations under general MSBM cases support our claim made in Section \ref{two:vs} that GMM is a generally better method than $k$-means in spectral clustering. Moreover, ISC\_gm and SCME\_gm are rather robust and continue to work well under general multi-layer stochastic block models. In practice, we recommend SCME\_gm for dense networks and ISC\_gm for sparse ones.

\end{enumerate}

\subsection{Real data example}\label{sec:real}

In this section, we use S\&P 1500 data to evaluate our methods further and compare them with the three existing spectral methods. S\&P 1500 index is a good representative of the US economy as it covers 90\% of the market value of US stocks. We obtain the daily adjusted close price of stocks from Yahoo! Finance \citep{ran2019} for the period from 2001-01-01 to 2019-06-30, including 4663 trading days in total. We keep only stocks with less than 50 days' missing data and forward fill the price. This leaves us with 1020 stocks. According to the newest Global Industry Classification Standard [GICS], there are in total 11 stock market sectors, of which each consists of a group of stocks that share some common economic characteristics. Therefore, we treat each sector as a community and use the sector information as the ground truth for community detection. We aim to discover the sectors or communities from stock prices. One more step of data pre-processing we did is to remove the sector ``Communication Services" due to its small size, and the sectors ``Industrials" and ``Materials" because of their similar performances during economic cycles. The final dataset contains 770 stocks from 8 sectors. 

We use the logarithmic return, a standard measure used in stock price analysis, to construct the network for stocks. Specifically, for a pair of stocks, we compute the Pearson correlation coefficient of log-returns (during a period of time) between the two stocks as the edge among them to measure their similarity. Note that the constructed network is a weighted graph whose adjacency matrix takes continuous values in $[-1,1]$. A binary network would have been easily obtained by thresholding. However, the conversion can result in a substantial loss of community information. Hence we will keep the weighted networks for community detection. We will see that our methods work well for this type of network as well.

The most straightforward idea is to use the whole time window to calculate the Pearson correlation among the stocks and create a single-layer network. However, since financial data is usually non-stationary, correlation tends to change over time. As a result, the within-community and between-community connectivity patterns may vary with time. To tackle such heterogeneity, we split the data into four time periods, according to the National Bureau of Economic Research [NBER], which are respectively recession I (2001/03-2001/11), expansion I (2001/12-2007/12), recession II (2008/01-2009/06) and expansion II (2009/07-2019/06). The intuition is that the economy cycle is a determinant of sector performance over the intermediate term. Different sectors tend to perform differently compared with the market in different phases of the economy. We then use the Pearson correlation computed within each time period to construct one layer. We end up having a four-layer network of stocks. 
\begin{table}[!h]
\caption{Community detection results for the stock market data}\begin{center}
\begin{tabular}{ccc|ccc}
	Method & ARI & Weights &Method & ARI & Weights\\  
	ISC\_gm 	&	.44 	&	[.52, .23, .07, .18]&Mean adj.	&	.37	& [.25, .25, .25, .25]\\ 
	ISC\_km	&	.35 	&	[.57, .21, .07, .15]&SpecK & .45 & - 	\\ 
	SCME\_gm	&	.65	&	[.08, .33, .01, .60]&Module alleg. & .29 & -\\ 
	SCME\_km	&	.43	&	[.08, .33, .01, .60]&&&	
\end{tabular}
\end{center}
\label{tab:sp}
\end{table}

Table~\ref{tab:sp} summarizes the community detection results of different methods. Again, GMM works better than $k$-means in both ISC and SCME. Our method SCME\_gm outperforms all the others by a large margin. ISC\_gm is also competitive compared with the three spectral methods. For a closer comparison, Figure~\ref{fig:conf} shows the confusion matrix from Mean adj. and SCME\_gm. We can clearly see the significant accuracy improvements for the consumer discretionary, financials, information technology, and real estate sectors. Table~\ref{tab:sp} also shows the weights for the four layers learned by our methods. In the appendix, we further explain the interesting implications of the weights learned by our best method SCME\_gm.

\begin{figure}[!htb]

\begin{subfigure}{.5\textwidth}
  \centering
  \includegraphics[width=1\linewidth]{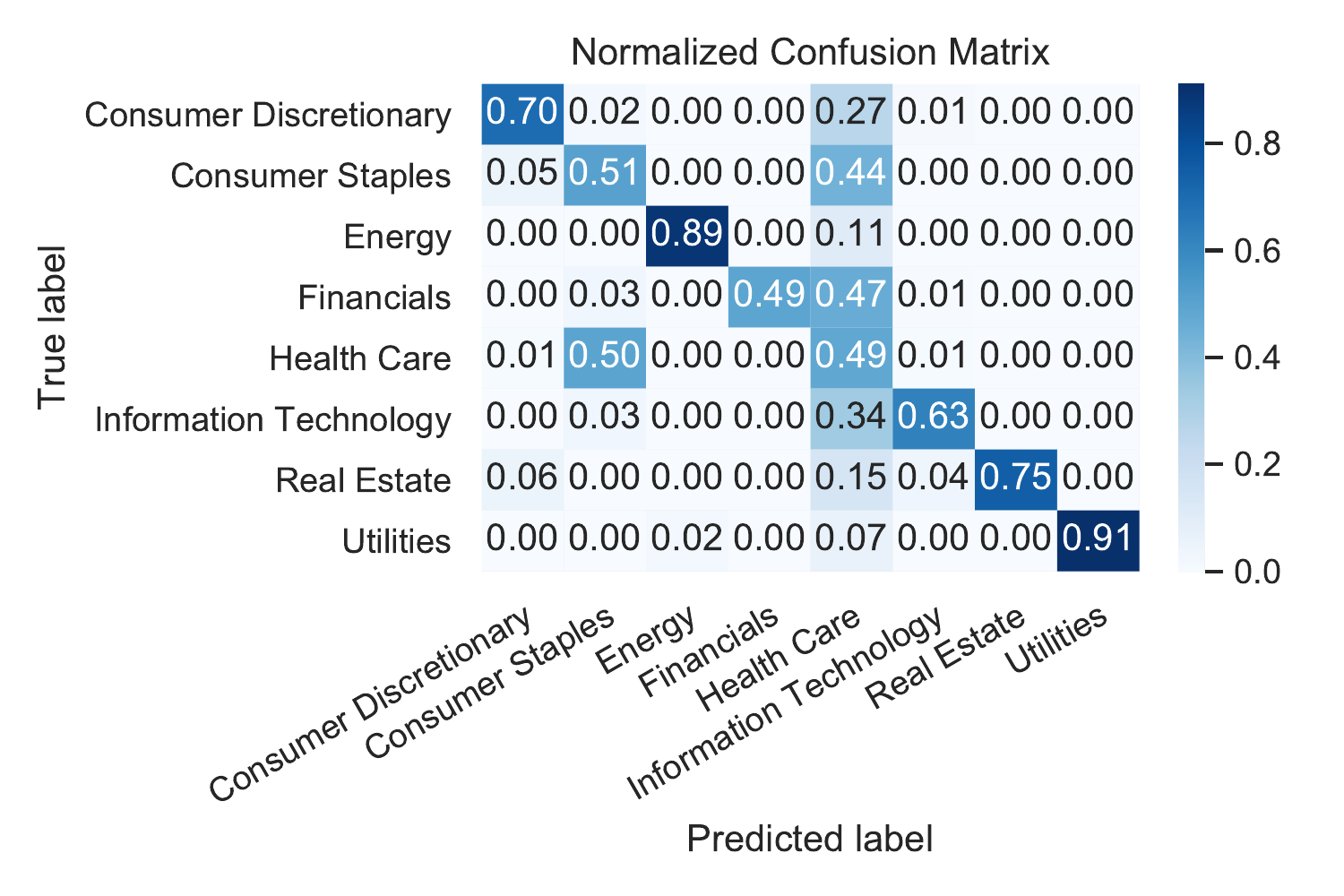}  
  \caption{Mean adj.}
  \label{fig:sub-all}
\end{subfigure}
\hfill
\begin{subfigure}{.5\textwidth}
  \centering
  \includegraphics[width=1\linewidth]{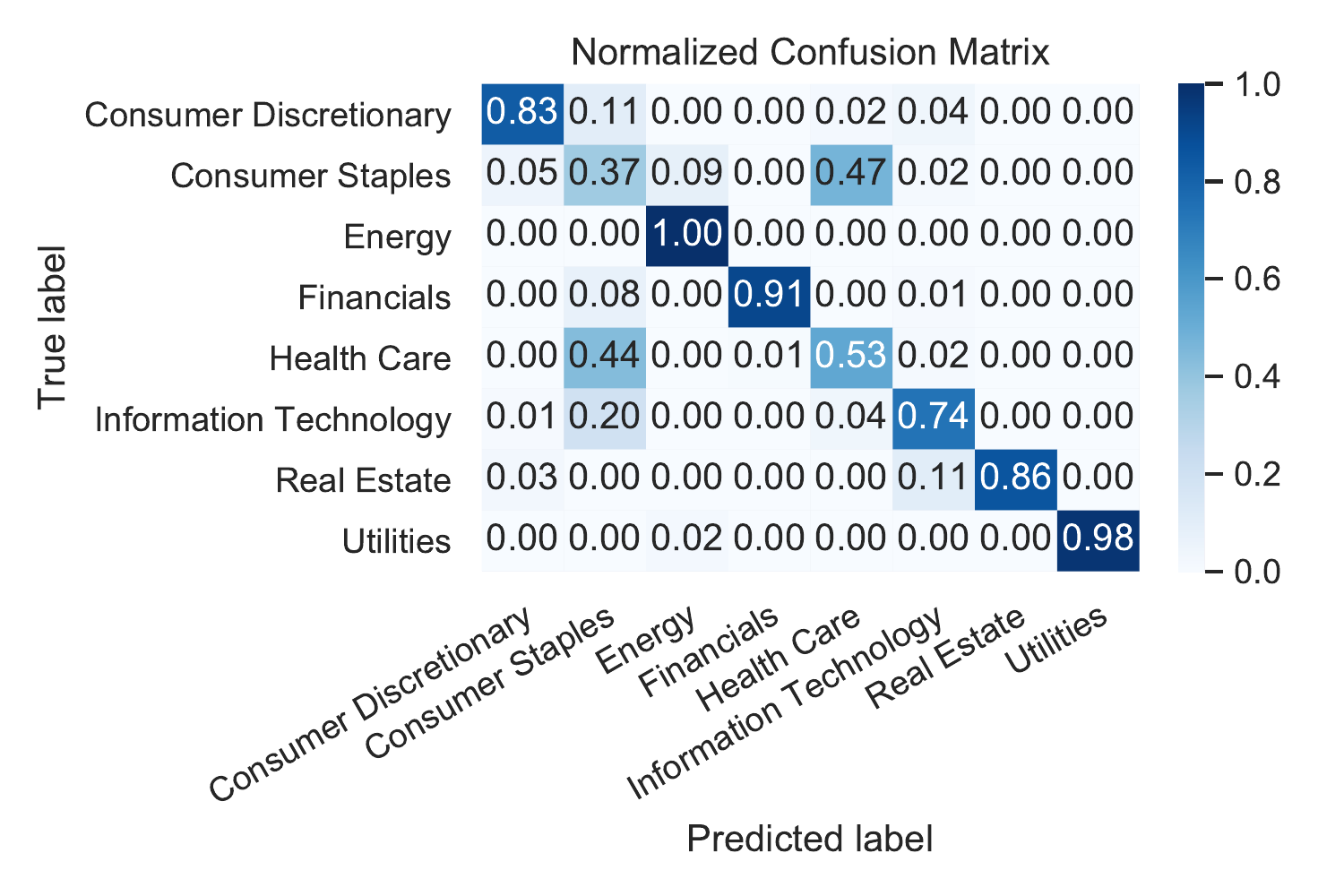}  
  \caption{SCME\_gm}
  \label{fig:sub-wam}
\end{subfigure}

\caption{Normalized confusion matrix of community detection results.}
\label{fig:conf}

\end{figure}

\section{Conclusions and Discussions}\label{dis:sess}

This paper presents a thorough study of the community detection problem for multi-layer networks via spectral clustering with adaptive layer aggregation. We develop two eigensystem-based methods for adaptive selection of the layer weight, and further provide a detailed discussion on the superiority of GMM clustering to $k$-means in the spectral clustering. Extensive numerical experiments demonstrate the impressive community detection performance of our algorithms. The proposed ISC and SCME algorithms are implemented on GitHub (\url{https://github.com/sihanhuang/Multi-layer-Network}).
Several important directions are left open for future research. 
\begin{itemize}
\item \textcolor{black}{The theory and methods are primarily developed for multi-layer networks with assortative communities structures. It is interesting to further study how our layer aggregation methods can be used to discover both assortative and dis-assortative community structures in multi-layer networks of a mixed structure, especially given the fact that simply adding the adjacency matrices from assortative layers and dis-assortative layers can potentially wash out the community signals in the data. We have applied the ISC algorithm (without any modification) to balanced MPPM of a mixed structure in some new simulations (see Section 2 in the appendix for details). The preliminary results show that like in assortative cases of Section \ref{sec:sim}, our algorithm is able to find the (nearly) optimal weight for layer aggregation. On the other hand, we expect that the SCME algorithm cannot be directly applied. This is because the algorithm is looking for maximum eigenratio between the $K$th and $(K+1)$th eigenvalues, while aggregating assortative layers (with positive weights) and dis-assortative layers (with negative weights) can potentially reduce the number of informative eigenvalues. Some modifications such as taking into account the first $K$ eigenratios $|\lambda_1|/|\lambda_2|,\ldots, |\lambda_K|/|\lambda_{K+1}|$ are needed to make it work. We leave a full investigation of detecting mixed community structures in a separate project.}
\item The current paper focuses on spectral clustering based on adjacency matrices. It is well known that variations of spectral clustering using matrices such as normalized or regularized graph Laplacian \citep{rohe2011spectral, qin2013regularized, amini2013pseudo, sarkar2015role, joseph2016impact, le2017concentration} can be useful for community detection under different scenarios. It is of great interest and feasible to generalize our framework to incorporate various spectral clustering forms. 
\item The effectiveness of our framework under general MSBM was illustrated by a wide array of synthetic and real datasets. A generalization of our asymptotic analysis to general MSBM will enable us to refine the proposed adaptive layer aggregation methods to achieve even higher accuracy. Such a generalization will rely on more sophisticated random matrix analysis and is considered as an important yet challenging future work. 
\item Our work assumes the number of communities $K$ is given, which may be unknown in certain applications. Some recent efforts for estimating $K$ in single-layer networks include \citet{le2015estimating, lei2016goodness, saldana2017many, wang2017likelihood}, among others. Extending our framework along these lines is interesting and doable. 
\end{itemize}

\section*{Supplementary Materials}
\begin{enumerate}
	\item Appendix: Additional simulation results under balanced MPPM, more real data analysis results, and all proofs. (pdf)
	\item Package: Python package for implementing the proposed algorithms. \url{https://github.com/sihanhuang/Multi-layer-Network}
	\item Code: Python codes for reproducing the simulation and real data analysis results. (zip)

\end{enumerate}
\section*{Acknowledgement}
We thank the editor, the AE, and the referees for their insightful comments which greatly improved the paper. Haolei Weng was partially supported by NSF Grant DMS-1915099. Yang Feng was partially supported by NSF CAREER Grant DMS-2013789 and NIH Grant 1R21AG074205-01. 
\begin{singlespace}\setstretch{0.75}


\bibliographystyle{JASA}

\bibliography{reference}
\end{singlespace}

\section{Appendix}
\beginsupplement
\subsection{Organization}

This appendix is organized as follows. Section \ref{simu:supp} shows the additional simulation results under balanced MPPM. Section \ref{rd:supp} provides more real data analysis results. Section \ref{sec:proof:prop1} proves Proposition \ref{prop:one}. Section \ref{sec:notpre} introduces some notations and preliminaries that will be extensively used in the latter proofs. Section \ref{sec:thm1} proves Theorem \ref{thm:error}.  Section \ref{sec:prop2} proves Proposition \ref{thm:main2}. Section \ref{sec:prop3} contains the proof of Proposition \ref{gaussian:embed}.

\subsection{Additional Simulations under Balanced MPPM}\label{simu:supp}
\label{more:bMPPM}

\begin{enumerate}

\item \emph{Accuracy under balanced {\rm MPPM}}. To further check the performance of our proposed algorithms, we computed accuracy, which is defined as one minus the mis-clustering error, instead of ARI under the same setting of Case 1a in Section \ref{sec:sim} of the main text. As is clear from Figure \ref{fig:1a-acc}, even with a relative small sample $n=600$, the accuracy of our algorithms matches well with the theoretical optimal values, as well as the empirical optimal values calculated by a grid search for the weights. These results indicate that for balanced MPPM, our algorithms are effective in aggregating the layers to obtain (nearly) optimal mis-clustering error when the network size is reasonably large.

\begin{figure}[!htb]
\centering
\hspace{-3.cm}
\begin{subfigure}{.45\textwidth}
  \includegraphics[width=1.3\linewidth]{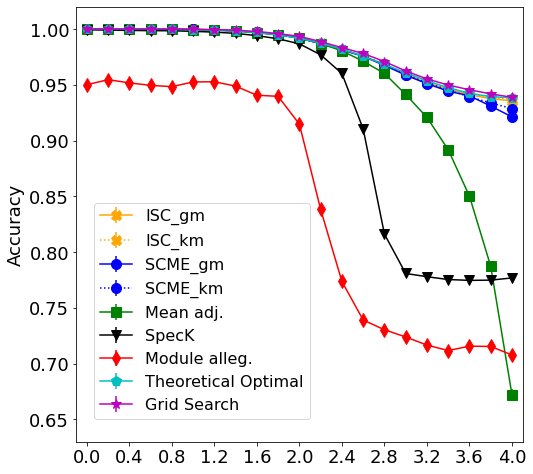}  
  \caption*{\hspace{1.5cm} Case 1a: $\overrightarrow{q}$}
  \label{fig:sub-1a_acc}
\end{subfigure}
\caption{``Theoretical Optimal" refers to the asymptotic accuracy calculated according to Theorem \ref{thm:error} in the main article, under the optimal weight derived in Proposition \ref{prop:one} of the main article.}
\label{fig:1a-acc}
\end{figure}

\item \emph{Balanced {\rm MPPM} with both assortative and dis-assortative community structures}. We conduct some preliminary simulations to check how our algorithms perform under balanced MPPM which has both assortative and dis-assortative community structures in the layers. The parameter values are specified in Table \ref{tab:dis-mppm}. Given that SCME may not be directly applicable to the mixed structure case, we focus on the performance of ISC. Referring to the results in Figure \ref{fig:mixed:new}, we can see that like in the assortative cases, our algorithm not only effectively integrates the information across assortative and dis-assortative layers for improved community detection results (in comparison to the ones from single layers), but also can allocate the (nearly) optimal aggregation weight to obtain the (nearly) optimal errors. It has been well known that simply adding the adjacency matrices from assortative layer and dis-assortative layer can potentially wash out the community signals in the data. The key success of our aggregation approach is to adaptively choose the appropriate weight for each layer to reveal the signals. In particular, as illustrated in Figure \ref{fig:mixed:new}, assigning negative weights to dis-assortative layers and positive weights to assortative layers can still effectively aggregate the signals across different layers. 

\begin{table}[htb]
\begin{center}

\def~{\hphantom{0}}
\caption{Balanced \textsc{MPPM} of a mixed community structure}{%
\begin{tabular}{cccccc}
$n$ & $K$ & $L$ & $c_\rho$ & $\overrightarrow{p}$ & $\overrightarrow{q}$ \\ 
6000 & 2 & 2  & 0.4 & (2-6, 1) & (1, 4) 
\end{tabular}}
\label{tab:dis-mppm}
	
\end{center}
\end{table}

\begin{figure}[!t]
\centering
\hspace{-3.cm}
\begin{subfigure}{.45\textwidth}
 \caption*{\hspace{2.5cm} ARI plot}
  \includegraphics[width=1.3\linewidth]{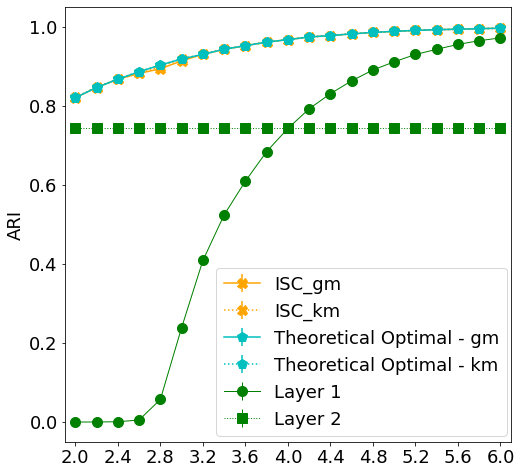}  
  \vspace{-1.1\baselineskip} 
  \label{fig:sub-km}
\end{subfigure}
\\
\vspace{0.5cm}
\begin{subfigure}{.4\textwidth}
  \caption*{\hspace{0.4cm} Weight plot for the first layer} 
  \includegraphics[scale=0.3]{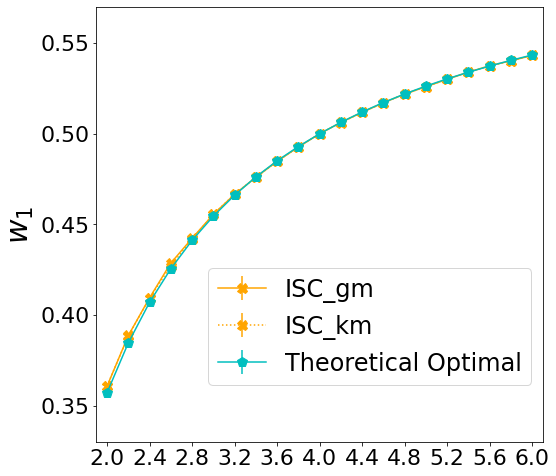}  
  \vspace{-0.5\baselineskip} 
  \label{fig:sub-gm}
\end{subfigure}
\begin{subfigure}{.4\textwidth}
  \caption*{\hspace{0.4cm} Weight plot for the second layer} 
  \includegraphics[scale=0.3]{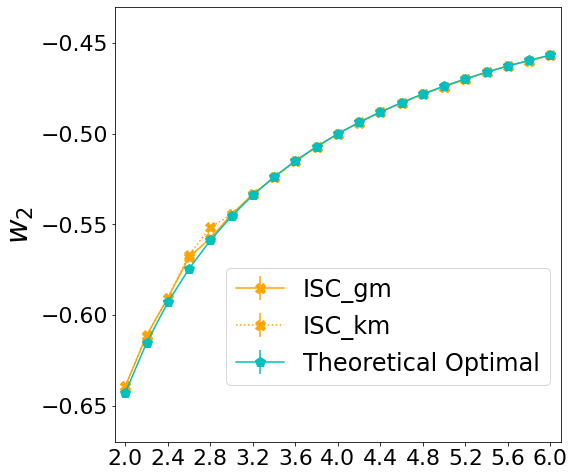}  
  \vspace{-0.5\baselineskip} 
  \label{fig:sub-true}
\end{subfigure}
\vspace{0.55\baselineskip} 
\caption{The x-axis in the three plots represents the varying parameter values specified in Table \ref{tab:dis-mppm}. Upper panel: ``Theoretical Optimal - gm (-km)" uses the optimal weight formula from Proposition \ref{prop:one} of the main article and GMM ($k$ means) clustering; ``Layer 1" only uses the first layer and $k$-means in spectral clustering; ``Layer 2" only uses the second layer and $k$-means in spectral clustering. Lower panel: ``Theoretical Optimal" denotes the optimal weight derived in Proposition \ref{prop:one} of the main article.}
\label{fig:mixed:new}
\end{figure}

\end{enumerate}


\subsection{Additional Real Data Analysis}\label{rd:supp}

For the real data example from the main text, Table~\ref{tab:sp} in Section \ref{sec:real} also shows the weights for the four layers learned by our methods. The best method SCME\_gm places most weights on the second and fourth layers corresponding to two periods of expansions. It indicates that expansions play a much more important role compared with recessions for distinguishing sectors. This is in accordance with our knowledge that in recessions, different sectors collapse in a comparatively similar way, while they will grow in different patterns during expansions. To illustrate this point further, we plot the adjacency matrices of correlations for the four-time periods in Figure \ref{fig:adj}. As is evident, there is a larger difference between within-sector and between-sector correlations in expansion I (II) than in recession I (II).

\begin{figure}[!htb]
	\centering
	\includegraphics[width=.65\linewidth]{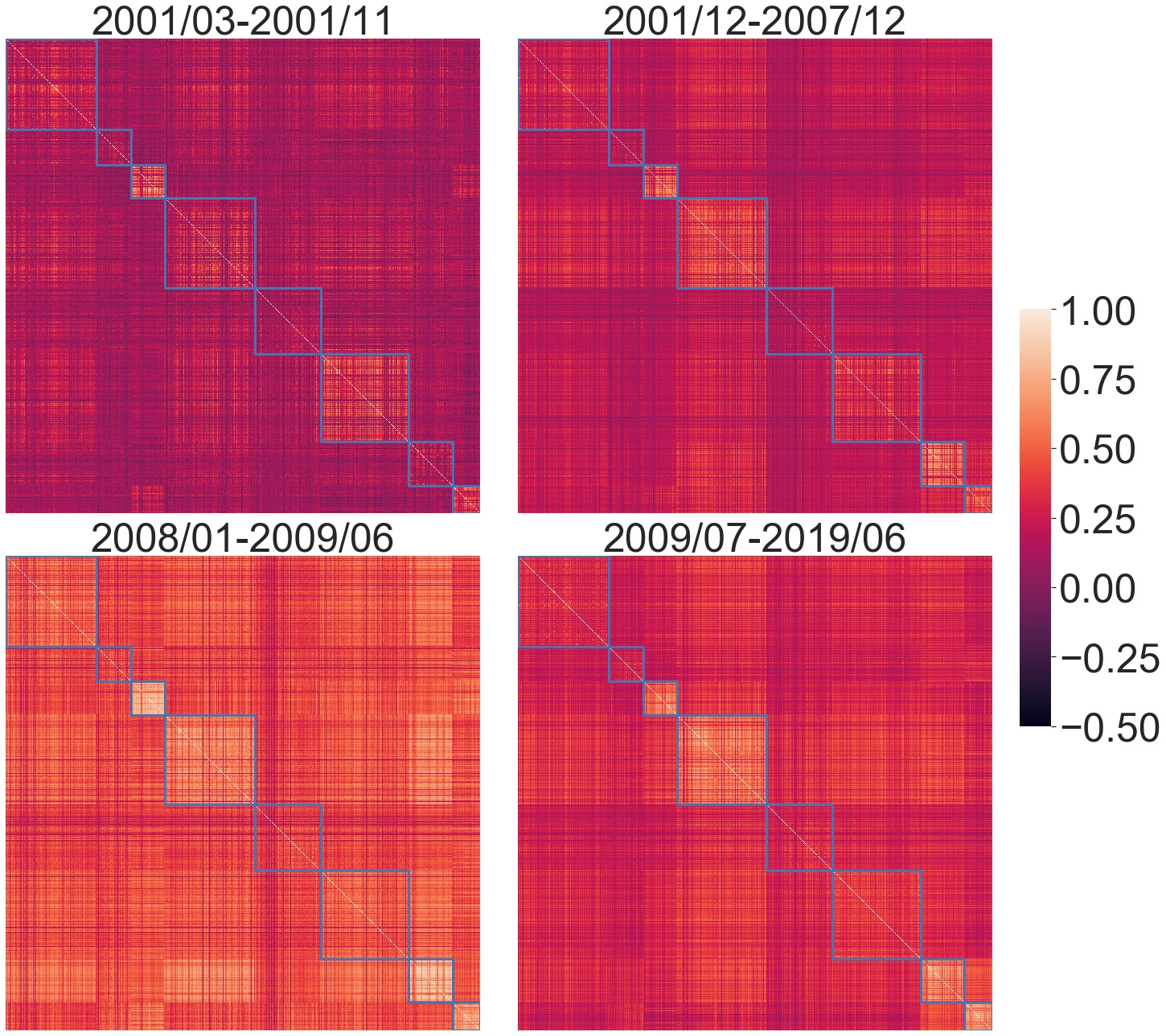}
	\caption{The adjacency matrix of correlations for different time periods. One marked block corresponds to one sector. The sectors are in the same order as in Figure~\ref{fig:conf} of the main text.}
	\label{fig:adj}
\end{figure}

\subsection{Proof of Proposition \ref{prop:one}} \label{sec:proof:prop1}

\begin{proof}

Denote $b_{\ell}=p^{(\ell)}_n-q^{(\ell)}_n, c_{\ell}=p^{(\ell)}_n(1-p^{(\ell)}_n)+(K-1)q^{(\ell)}_n(1-q^{(\ell)}_n), \ell\in [L]$. The proof is a direct application of Cauchy-Schwarz inequality:
\begin{align*}
\tau^{\overrightarrow{w}}_n= \frac{n(\sum_{\ell=1}^Lw_{\ell}\sqrt{c_{\ell}}\cdot \frac{b_{\ell}}{\sqrt{c_{\ell}}}\big)^2}{\sum_{\ell=1}^Lw_{\ell}^2c_{\ell}}\leq \frac{n\sum_{\ell=1}^Lw_{\ell}^2c_{\ell}\cdot \sum_{\ell=1}^L \frac{b_{\ell}^2}{c_{\ell}}}{\sum_{\ell=1}^Lw_{\ell}^2c_{\ell}}=n\sum_{\ell=1}^L \frac{b_{\ell}^2}{c_{\ell}},
\end{align*}
The above inequality becomes equality if and only if $w_{\ell}\sqrt{c_{\ell}} \propto  \frac{b_{\ell}}{\sqrt{c_{\ell}}}, \ell \in [L]$.

\end{proof}

\subsection{Notations and Preliminaries}\label{sec:notpre}

For a balanced multi-layer planted partition model $A^{[L]}\sim$\textsc{MPPM}$(p^{[L]},q^{[L]}, \overrightarrow{\pi})$, the weighted adjacency matrix is $A^{\overrightarrow{w}}=\sum_{\ell=1}^Lw_{\ell}A^{(\ell)}$. For simplicity, we will drop the subscript or superscript $\overrightarrow{w}$ in various notations such as $A^{\overrightarrow{w}}$ and $\hat{\overrightarrow{c}}_{\overrightarrow{w}}$ whenever it does not cause any confusion. We may also omit the dependency of notations on $n$ whenever it is clear from the context. Define $\mathcal{B}_r=\{x: \|x\|_2\leq r\}$; $(a)_+=\max\{a,0\}; (a)_{-}=\max\{0,-a\}; {\rm sign}(a)$ is the sign of $a$; $J_k$ denotes the $k\times k$ matrix with all entries being one; $\{e_i\}_{i=1}^n$ is the standard basis in $\mathcal{R}^n$; $[m]=\{1,2,\ldots, m\}$ for a positive integer $m$; $\|\cdot\|_F$ denotes the Frobenius norm; $\|\cdot\|_2$ is the spectral norm; $\alpha_n\lesssim \beta_n$ means $\alpha_n\leq C\beta_n$ for some constant $C>0$ and all $n$; $\alpha_n\approx \beta_n$ denotes $\lim_{n\rightarrow \infty}\alpha_n/\beta_n=1$; $\alpha_n \propto \beta_n$ means $\alpha_n,\beta_n$ are of the same order; ${\rm diag}(a_1,a_2,\ldots, a_n)$ denotes a diagonal matrix with diagonal entries $a_1,\ldots, a_n$; $\tau^{\overrightarrow{w}}_{n}$ is defined in \eqref{finite:snr} from the main article and $\tau^{\overrightarrow{w}}_{\infty}=\lim_{n\rightarrow \infty} \tau^{\overrightarrow{w}}_{n}$. Note that the size of each community concentrates tightly around $\frac{n}{K}$. It is sometimes more convenient to consider an asymptotically equivalent balanced planted partition model where the community labels $\overrightarrow{c}$ are uniformly drawn such that each community size equals $\frac{n}{K}$ \citep{avrachenkov2015spectral, abbe2015exact, mossel2015consistency}. To streamline the overall proving process, we will adopt such a version of planted partition model throughout the appendix.

\subsection{Proof of Theorem \ref{thm:error}}\label{sec:thm1}

The proof relies on the asymptotic normality of spectral embedding in Proposition \ref{gaussian:embed} of the main article. Recall that $U\in \mathcal{R}^{n\times K}$ is the eigenvector matrix and $\mathcal{O}\in \mathcal{R}^{K\times K}$ is the orthogonal matrix in Proposition \ref{gaussian:embed}. The rows of $U\in \mathcal{R}^{n\times K}$ are the input data points for the $k$-means in spectral clustering. Since $k$-means is invariant under scaling and orthogonal transformation, it is equivalent to solving the following optimization problem
\begin{align*}
\min_{\{c_i\}_{i=1}^n}\min_{\{\xi_k\}_{k=1}^K}\sum_{i=1}^n\sum_{k=1}^K \|x_i-\xi_{c_i}\|_2^2,
\end{align*}
where $x_i\equiv \sqrt{n}\mathcal{O}U^Te_i$ and $c_i\in [K]$. Denote the estimated community labels and centers by $\hat{\overrightarrow{c}}=(\hat{c}_1, \ldots, \hat{c}_n)^T$ and $\{\hat{\xi}_k\}_{k=1}^K$ respectively. For a permutation $\delta: [K]\rightarrow [K]$, define
\[
r^{\delta}(\hat{\overrightarrow{c}})\equiv \frac{1}{n}\sum_{i=1}^n 1(\delta(\hat{c}_i)\neq c_i).
\]
The mis-clustering error can be then written as $r(\hat{\overrightarrow{c}})=\inf_{\delta}r^{\delta}(\hat{\overrightarrow{c}})$. We first analyze $E[r^{\delta}(\hat{\overrightarrow{c}})]$ for a given permutation $\delta$. Conditioning on the community labels $\overrightarrow{c}$, we have that $\forall \epsilon>0$,
\begin{align*}
&\hspace{0.6cm} E[r^{\delta}(\hat{\overrightarrow{c}})] \\
&= 1-\frac{1}{n}\sum_{k=1}^K\sum_{c_i=k}{\rm pr}(\hat{c}_i=\delta^{-1}(k)) \leq 1-\frac{1}{n}\sum_{k=1}^K\sum_{c_i=k}{\rm pr}\Big(\|x_i-\hat{\xi}_{\delta^{-1}(k)}\|_2<\min_{j\neq  \delta^{-1}(k)}\|x_i-\hat{\xi}_j\|_2\Big) \\
& \leq 1-\frac{1}{n}\sum_{k=1}^K\sum_{c_i=k} {\rm pr}\Big(\|x_i-\vartheta_{\delta^{-1}(k)}\|_2<\min_{j\neq  \delta^{-1}(k)}\|x_i-\vartheta_j\|_2-3\epsilon,~ \max_{j\in [K]}\|\hat{\xi}_j-\vartheta_j\|_2\leq \epsilon\Big) \\
& \leq 1-\frac{1}{n}\sum_{k=1}^K\sum_{c_i=k} {\rm pr}\Big(\|x_i-\vartheta_{\delta^{-1}(k)}\|_2<\min_{j\neq  \delta^{-1}(k)}\|x_i-\vartheta_j\|_2-3\epsilon\Big)+ {\rm pr}\Big(\max_{j\in [K]}\|\hat{\xi}_j-\vartheta_j\|_2> \epsilon\Big) 
\end{align*}
Here, $\{\vartheta_j\}_{j=1}^K$ denote the limits that the estimated centers $\{\hat{\xi}_k\}_{k=1}^K$ converge to as shown in Lemma \ref{centerk}. With the exchangeability of $\{x_i\}_{c_i=k}$ we can proceed from the last upper bound to obtain that marginally
\begin{align*}
E[r^{\delta}(\hat{\overrightarrow{c}})]  \leq &1-\frac{1}{K}\sum_{k=1}^K {\rm pr}\Big(\|x_i-\vartheta_{\delta^{-1}(k)}\|_2<\min_{j\neq  \delta^{-1}(k)}\|x_i-\vartheta_j\|_2-3\epsilon\mid c_i=k\Big) \\
&\hspace{0.3cm} + {\rm pr}\Big(\max_{j\in [K]}\|\hat{\xi}_j-\vartheta_j\|_2> \epsilon\Big).
\end{align*}
According to Proposition \ref{gaussian:embed} and Lemma \ref{centerk}, the above result implies that $\forall \epsilon>0$
\begin{align*}
\limsup_{n\rightarrow \infty} E[r^{\delta}(\hat{\overrightarrow{c}})] \leq 1-\frac{1}{K}\sum_{k=1}^K  {\rm pr} \Big(\|z_k-\vartheta_{\delta^{-1}(k)}\|_2<\min_{j\neq \delta^{-1}(k)}\|z_k-\vartheta_j\|_2-3\epsilon\Big), 
\end{align*}
where $z_k\sim \mathcal{N}(\mu^{(k)},\Theta)$ is the Gaussian distribution specified in Proposition \ref{gaussian:embed}. Letting $\epsilon\rightarrow 0$ on both sides of the above inequality gives
\begin{align*}
\limsup_{n\rightarrow \infty} E[r^{\delta}(\hat{\overrightarrow{c}})] &\leq 1-\frac{1}{K}\sum_{k=1}^K {\rm pr}\Big(\|z_k-\vartheta_{\delta^{-1}(k)}\|_2 \leq \min_{j\neq \delta^{-1}(k)}\|z_k-\vartheta_j\|_2\Big) \\
&=1-\frac{1}{K}\sum_{k=1}^K {\rm pr}\Big(\cap_{j\neq \delta^{-1}(k)}\{z^T_k(\vartheta_{\delta^{-1}(k)}-\vartheta_j)\geq 0\}\Big)
\end{align*}
In a similar way, we can derive a lower bound. 
\begin{align*}
&\hspace{0.6cm} E[r^{\delta}(\hat{\overrightarrow{c}})] \\
&\geq 1-\frac{1}{n}\sum_{k=1}^K\sum_{c_i=k}  {\rm pr}\Big(\|x_i-\vartheta_{\delta^{-1}(k)}\|_2<\min_{j\neq  \delta^{-1}(k)}\|x_i-\vartheta_j\|_2+3\epsilon~\mbox{or}~\max_{j\in [K]}\|\hat{\xi}_j-\vartheta_j\|_2> \epsilon\Big) \\
& \geq 1-\frac{1}{K}\sum_{k=1}^K {\rm pr} \Big(\|x_{i}-\vartheta_{\delta^{-1}(k)}\|_2<\min_{j\neq  \delta^{-1}(k)}\|x_{i}-\vartheta_j\|_2+3\epsilon\mid c_i=k\Big) \\
&\hspace{1.cm} -  {\rm pr}\Big(\max_{j\in [K]}\|\hat{\xi}_j-\vartheta_j\|_2> \epsilon\Big).
\end{align*}
Letting $n\rightarrow \infty$ and then $\epsilon\rightarrow 0$ yields
\[
\liminf_{n\rightarrow \infty} E[r^{\delta}(\hat{\overrightarrow{c}})]  \geq 1-\frac{1}{K}\sum_{k=1}^K {\rm pr}\Big(\cap_{j\neq \delta^{-1}(k)}\{z^T_k(\vartheta_{\delta^{-1}(k)}-\vartheta_j)\geq 0\}\Big).
\]
We thus have proved that 
\begin{align*}
\lim_{n\rightarrow \infty} E[r^{\delta}(\hat{\overrightarrow{c}})] &=1-\frac{1}{K}\sum_{k=1}^K {\rm pr}\Big(\cap_{j\neq \delta^{-1}(k)}\{z^T_k(\vartheta_{\delta^{-1}(k)}-\vartheta_j)\geq 0\}\Big)
\end{align*}
Using the multivariate Gaussian distributions of $\{z_k\}_{k=1}^K$, it is straightforward to confirm that 
\begin{align*}
\lim_{n\rightarrow \infty} E[r^{\delta}(\hat{\overrightarrow{c}})]=
\begin{cases}
1-{\rm pr}(a_i \geq 0, i=1,2,\ldots, K-1) & \text{if}\ \delta \ \text{is the identity mapping}\ I   \\
1-\frac{1}{K}\sum_{k=1}^Kp^{\delta}_{k}& \text{otherwise}
\end{cases}
\end{align*}
Here, $(a_1,\ldots,a_{K-1})^T$ follows the multivariate Gaussian distribution described in Theorem 1. At least two of $\{p^{\delta}_k\}_{k=1}^K$ are equal to ${\rm pr}(a_i \geq 1, i=1,2,\ldots, K-2, a_{K-1}\geq 2)$, and the rest of $\{p^{\delta}_k\}_{k=1}^K$ are all equal to ${\rm pr}(a_i \geq 0, i=1,2,\ldots, K-1)$. It is clear that $\lim_{n\rightarrow \infty} E[r^{\delta}(\hat{\overrightarrow{c}})]$ is smallest when $\delta$ equals the identity mapping $I$. With a direct modification of the preceding arguments, it is not hard to further show that the convergence of $r^{\delta}(\hat{\overrightarrow{c}})$ holds in probability as well. As a result, as $n\rightarrow \infty$
\begin{align*}
{\rm pr}(\mathcal{A}_n)\rightarrow 1, \quad \mathcal{A}_n \equiv \Big\{r^{I}(\hat{\overrightarrow{c}}) \leq \min_{\delta \neq I} r^{\delta}(\hat{\overrightarrow{c}})\Big\}
\end{align*}
Accordingly, we have
\begin{align*}
0\leq E[r^I(\hat{\overrightarrow{c}})]-E[r(\hat{\overrightarrow{c}})]=E[r^I(\hat{\overrightarrow{c}})]-E[\inf_{\delta}r^{\delta}(\hat{\overrightarrow{c}})]\leq 2{\rm pr}(\mathcal{A}^c_n)\rightarrow 0,
\end{align*}
which leads to
\[
\lim_{n\rightarrow \infty}E[r(\hat{\overrightarrow{c}})] =\lim_{n\rightarrow \infty} E[r^{I}(\hat{\overrightarrow{c}})]= 1-{\rm pr}(a_i \geq 0, i=1,2,\ldots, K-1).
\]
Finally, the monotonicity of the limit follows from a simple change of variable:
\begin{align*}
&{\rm pr}(a_i \geq 0, i=1,2,\ldots, K-1)\\
=&(2\pi)^{-\frac{K-1}{2}}|I_{K-1}+J_{K-1}|^{-\frac{1}{2}}\int_{y_i\geq 0}e^{-\frac{1}{2}(y-\sqrt{\tau^{\overrightarrow{w}}_{\infty}-K}\overrightarrow{1})^T(I_{K-1}+J_{K-1})^{-1}(y-\sqrt{\tau^{\overrightarrow{w}}_{\infty}-K}\overrightarrow{1})}dy \\
=&(2\pi)^{-\frac{K-1}{2}}|I_{K-1}+J_{K-1}|^{-\frac{1}{2}}\int_{y_i\geq -\sqrt{\tau^{\overrightarrow{w}}_{\infty}-K}}e^{-\frac{1}{2}y^T(I_{K-1}+J_{K-1})^{-1}y}dy.
\end{align*}

\begin{lemma}\label{centerk}
The estimated centers from the $k$-means can be written as
\begin{align}
\label{obj:one}
\{\hat{\xi}_k\}_{k=1}^K= \underset{\{\xi_k\}_{k=1}^K}{\operatorname{argmin}} ~\frac{1}{n}\sum_{i=1}^n \min_{1\leq k\leq K}\|\sqrt{n}\mathcal{O}U^Te_i-\xi_k\|_2^2.
\end{align}
There exists a labeling of the $K$ centers and some constant $t>0$ such that for each 
$k\in [K]$, 
\begin{equation}
\hat{\xi}_k \overset{P}{\rightarrow}
\begin{pmatrix}
1 \\
t \sqrt{\frac{K(\tau^{\overrightarrow{w}}_{\infty}-K)}{\tau^{\overrightarrow{w}}_{\infty}}} \nu_k
\end{pmatrix}
,
\quad \mbox{as}~ n\rightarrow \infty,  
\label{center:limits}
\end{equation}
where the $(K-1)$-dimensional vectors $\nu_k$'s are defined in Proposition \ref{gaussian:embed}. 
\end{lemma}
\begin{proof}
The strong consistency of $k$-means in classical i.i.d. cases has been derived in \citet{pollard1981strong}. We adapt the arguments in \citet{pollard1981strong} to prove the convergence of $k$-means in probability under the spectral clustering settings. To avoid the ambiguities in the labeling of the centers, we treat the objective function in \eqref{obj:one} as a function of a set $\Xi$ containing $K$ (or fewer) points in $\mathcal{R}^K$, and denote it by 
\[
\ell_n(\Xi) \equiv \frac{1}{n}\sum_{i=1}^n \min_{\xi\in \Xi}\|\sqrt{n}\mathcal{O}U^Te_i-\xi\|_2^2.
\]
The distance between two subsets is measured by the Hausdorff metric $d_H(\cdot,\cdot)$. Accordingly, the estimated centers $\{\hat{\xi}_k\}_{k=1}^K$ can be represented by a set $\hat{\Xi}_n$ given as 
\[
\hat{\Xi}_n=\underset{\Xi \in \mathcal{C}_K}{\operatorname{argmin}} ~\ell_n(\Xi), \quad \mathcal{C}_k=\{\Xi\subseteq \mathcal{R}^K: \Xi~\mbox{contains}~K~\mbox{or fewer points}\}.
\]
In light of Lemma \ref{lln} (i), we define the `population' function as
\begin{align*}
\ell(\Xi)\equiv E\min_{\xi \in \Xi}\|z-\xi\|_2^2, 
\end{align*}
where $z$ is the Gaussian mixture distribution specified in Lemma \ref{lln}. We first aim to prove the following results: for some large constant $r>0$, as $n\rightarrow \infty$
\begin{enumerate}
\item[(a)] ${\rm pr}(\hat{\Xi}_n \subseteq \mathcal{B}_r) \rightarrow 1.$
\item[(b)] $\sup_{\Xi \in \mathcal{D}_K} |\ell_n(\Xi)-\ell(\Xi)|\overset{P}{\rightarrow} 0$, where $\mathcal{D}_K=\{\Xi \subseteq \mathcal{B}_r: \Xi~\mbox{has}~K~\mbox{or fewer points} \}$.
\item[(c)] $\ell(\Xi)$ is continuous on $\mathcal{D}_K$.
\end{enumerate}
Suppose the above results hold and $\ell(\Xi)$ has a unique minimizer $\Xi^*\in \mathcal{D}_K$, then $\forall \epsilon>0$ there exists $\delta_{\epsilon}>0$ such that
\begin{align*}
&{\rm pr}(d_H(\hat{\Xi}_n, \Xi^*)>\epsilon)= {\rm pr}(d_H(\hat{\Xi}_n, \Xi^*)>\epsilon, \hat{\Xi}_n  \subseteq \mathcal{B}_r) + {\rm pr}(d_H(\hat{\Xi}_n, \Xi^*)>\epsilon,\hat{\Xi}_n  \not\subseteq \mathcal{B}_r) \\
\leq & {\rm pr}\Big(\sup_{\Xi \in \mathcal{D}_K} |\ell_n(\Xi)-\ell(\Xi)| >\delta_{\epsilon}\Big)+  {\rm pr}(\hat{\Xi}_n  \not\subseteq \mathcal{B}_r)\rightarrow 0,
\end{align*}
thus $\hat{\Xi}_n\overset{P}{\rightarrow }\Xi^*$. The proof will be completed after a further analysis of $\Xi^*$. We closely follow the idea in \citet{pollard1981strong} to prove (a)(b)(c). For the sake of simplicity, we will describe several key steps and refer the readers to \citet{pollard1981strong} for more details. To prove Part (a), we first show ${\rm pr}(\hat{\Xi}_n \cap \mathcal{B}_r\neq \emptyset)\rightarrow 1$ when $r$ is large enough. This is because for any $0<\tilde{r}<r$,
\begin{align*}
&\frac{1}{n}\sum_{i=1}^n\|\sqrt{n}\mathcal{O}U^Te_i\|_2^2=\ell_n(\{0\})  \geq \ell_n(\hat{\Xi}_n) \\
\geq & 1(\hat{\Xi}_n \cap \mathcal{B}_r=\emptyset) \cdot  \frac{(r-\tilde{r})^2}{n}\sum_{i=1}^n 1(\sqrt{n}\mathcal{O}U^Te_i\in \mathcal{B}_{\tilde{r}}).
\end{align*}
Choosing $r$ sufficiently large that $E\|z\|_2^2<a<(r-\tilde{r})^2\cdot {\rm pr}(\|z\|_2\leq \tilde{r})$ for some $a>0$, we continue from the above bound to have
\begin{align*}
{\rm pr}\Big(\frac{1}{n}\sum_{i=1}^n\|\sqrt{n}\mathcal{O}U^Te_i\|_2^2>a\Big)&\geq {\rm pr}\Big(\frac{(r-\tilde{r})^2}{n}\sum_{i=1}^n 1(\sqrt{n}\mathcal{O}U^Te_i\in \mathcal{B}_{\tilde{r}})>a,\hat{\Xi}_n \cap \mathcal{B}_r=\emptyset\Big) \\
&\hspace{-0.9cm} \geq {\rm pr}(\hat{\Xi}_n \cap \mathcal{B}_r=\emptyset)-{\rm pr}\Big(\frac{(r-\tilde{r})^2}{n}\sum_{i=1}^n 1(\sqrt{n}\mathcal{O}U^Te_i\in \mathcal{B}_{\tilde{r}})\leq a\Big),
\end{align*}
which combined with Lemma \ref{lln} (ii)(iii) leads to ${\rm pr}(\hat{\Xi}_n \cap \mathcal{B}_r= \emptyset)\rightarrow 0$. Define $\tilde{\Xi}_n= \hat{\Xi}_n \cap \mathcal{B}_{5r} $. On the event $\mathcal{E}\equiv \{\hat{\Xi}_n \cap \mathcal{B}_r\neq \emptyset\}\cap \{\hat{\Xi}_n \not\subseteq \mathcal{B}_{5r}\}$,  $\tilde{\Xi}_n$ has at least one point $x$ from $\mathcal{B}_r$  and contains at most $K-1$ points. So it holds that on $\mathcal{E}$ 
\begin{align*}
\min_{\Xi \in \mathcal{C}_{K-1}}\ell_n(\Xi) & \leq \ell_n(\tilde{\Xi}_n) \leq \ell_n(\hat{\Xi}_n) + \frac{1}{n}\sum_{i=1}^n\|\sqrt{n}\mathcal{O}U^Te_i-x\|_2^2\cdot 1(\|\sqrt{n}\mathcal{O}U^Te_i\|\geq 2r, \|x\|\leq r) \\
& \leq \ell_n(\Xi^*) +\frac{5}{2n}\sum_{i=1}^n\|\sqrt{n}\mathcal{O}U^Te_i\|_2^2 \cdot 1(\|\sqrt{n}\mathcal{O}U^Te_i\|\geq 2r). 
\end{align*}
The upper bound converges in probability to $\min_{\Xi\in \mathcal{C}_K}\ell(\Xi)+\frac{5}{2}E[\|z\|_2^2\cdot 1(\|z\|_2\geq 2r)]$ by Lemma \ref{lln} (i)(ii). The convergence of the lower bound $\min_{\Xi \in \mathcal{C}_{K-1}}\ell_n(\Xi)\overset{P}{\rightarrow} \min_{\Xi\in \mathcal{C}_{K-1}}\ell(\Xi)$ is obtained by the same inductive argument from \citet{pollard1981strong}. Since $\min_{\Xi\in \mathcal{C}_{K-1}}\ell(\Xi)> \min_{\Xi\in \mathcal{C}_{K}}\ell(\Xi)$, choosing $r$ large enough makes the upper bound limit smaller than the lower bound limit, which implies 
\[
{\rm pr}(\hat{\Xi}_n \not\subseteq \mathcal{B}_{5r}) \leq {\rm pr}(\hat{\Xi}_n \cap \mathcal{B}_r= \emptyset) + {\rm pr}(\mathcal{E})\rightarrow 0.
\]
Having the convergence results shown in Lemma \ref{lln} (i), (b)(c) can be proved by following exactly the same lines in \citet{pollard1981strong}. We thus do not repeat the arguments. 

It remains to verify that the $K$ limits in \eqref{center:limits} form the set $\Xi^*$ which minimizes $\ell(\Xi)$. Because the first coordinate of $z$ equals the constant 1, the first coordinates of points in $\Xi^*$ will be all equal to it. Regarding the other $K-1$ coordinates, it is equivalent to show that for the following objective function
\[
g(\mu)\equiv E\min_{1\leq k\leq K}\|\mu_k-Y\|_2^2,
\]
where $Y\sim  \frac{1}{K}\sum_{k=1}^K\mathcal{N}(\tilde{\nu}_k,K/\tau^{\overrightarrow{w}}_{\infty}\cdot I_{K-1}), \tilde{\nu}_k=\sqrt{K(\tau^{\overrightarrow{w}}_{\infty}-K)/\tau^{\overrightarrow{w}}_{\infty}} \nu_k$, there exists a constant $t>0$ such that $\mu^*_k=t \tilde{\nu}_k, k\in [K]$ minimize $g(\mu)$. It is clear that the constant $t>0$ has to minimize 
$E\min_{1\leq k\leq K}\|t\tilde{\nu}_k-Y\|_2^2$, thus 
\begin{align}
\label{t:formula}
t=-E\min_{ k}-\tilde{\nu}_k^TY/\|\tilde{\nu}_1\|_2^2. 
\end{align}
We next compute the gradient of $g(\mu)$. Define $f_{\beta}(\mu)=E\frac{1}{-\beta}\log\sum_ke^{-\beta\|\mu_k-Y\|_2^2}$ for some $\beta>0$. It is easy to check that $f_{\beta}(\mu)\leq g(\mu)\leq f_{\beta}(\mu)+\frac{\log K}{\beta}$. Hence for any given direction $h$, 
\[
\Big|\frac{g(\mu+\Delta \cdot h)-g(\mu)}{\Delta}   -\frac{f_{\beta}(\mu+\Delta \cdot h)-f_{\beta}(\mu)-\beta^{-1}\log K}{\Delta} \Bigg |\leq  \frac{\log K}{\beta \Delta}.
\]
Choosing $\beta=\Delta^{-2}$ and applying Dominated Convergence Theorem, it is straightforward to confirm that taking $\Delta \searrow 0$ on both sides of the above inequality yields 
\[
h^T\nabla g(\mu)=\sum_{k} 2h_k^T E[(\mu_k-Y)\cdot 1(\|\mu_k-Y\|_2\leq \|\mu_j-Y\|_2, j \neq k)]. 
\]
We verify that the derivative $\frac{\partial g(\mu)}{\partial \mu_k}$ evaluated at $\{\mu^{*}_k\}_{k\in [K]}$ is equal to zero. We have
\begin{align}
\label{gradient:eqc}
\frac{\partial g(\mu^*)}{\partial \mu_k}&=2 E[(t\tilde{\nu}_k-Y)\cdot 1((\tilde{\nu}_k-\tilde{\nu}_j)^TY\geq 0, j \neq k)] \nonumber \\
&=2K^{-1}t\tilde{\nu}_k-2E[Y\cdot 1((\tilde{\nu}_k-\tilde{\nu}_j)^TY\geq 0, j \neq k)],
\end{align}
where the second equality is due to the symmetry of $\{\tilde{\nu}_k\}$. To evaluate the second term on the right-hand side of above equation, denote $\mathcal{V}=(\tilde{\nu}_1,\ldots,\tilde{\nu}_{k-1},\tilde{\nu}_{k+1},\ldots, \tilde{\nu}_K)^T\in \mathcal{R}^{(K-1)\times (K-1)}$. With the change of variables $Y=\mathcal{V}^T\tilde{Y}$, we get
\begin{align*}
&E[Y\cdot 1((\tilde{\nu}_k-\tilde{\nu}_j)^TY\geq 0, j \neq k)] \\
=&(2\pi K/\tau^{\overrightarrow{w}}_{\infty})^{(1-K)/2}\int_{\tilde{Y}_i \geq 0, i\in [K-1]} \mathcal{V}^T\tilde{Y}\sum_{k=1}^Ke^{-\frac{\tau^{\overrightarrow{w}}_{\infty}\|\mathcal{V}^T\tilde{Y}-\tilde{\nu}_k\|_2^2}{2K}}d\tilde{Y}.
\end{align*}
Moreover, the summation in the above integrand is invariant under the permutation of the rows of $\mathcal{V}$. Let $\mathcal{V}_{\varpi}$ be the resulting matrix after a row permutation $\varpi$ of $\mathcal{V}$. We use the change of variables $Y=\mathcal{V}_{\varpi}^T\tilde{Y}$ to derive 
\begin{align}
\label{symmetry:use}
&E[Y\cdot 1((\tilde{\nu}_k-\tilde{\nu}_j)^TY\geq 0, j \neq k)] \nonumber \\
=&(2\pi K/\tau^{\overrightarrow{w}}_{\infty})^{(1-K)/2} (K!)^{-1} \int_{\tilde{Y}_i \leq 0, i\in [K-1]} \sum_{\varpi }\mathcal{V}^T_{\varpi}\tilde{Y}\cdot \sum_{k=1}^Ke^{-\frac{\tau^{\overrightarrow{w}}_{\infty}\|\mathcal{V}^T_{\varpi}\tilde{Y}-\tilde{\nu}_k\|_2^2}{2K}}d\tilde{Y} \nonumber \\
=&-\tilde{v}_k(2\pi K/\tau^{\overrightarrow{w}}_{\infty})^{(1-K)/2} K^{-1} \int_{\tilde{Y}_i \leq 0, i\in [K-1]} 1^T\tilde{Y}\cdot \sum_{k=1}^Ke^{-\frac{\tau^{\overrightarrow{w}}_{\infty}\|\mathcal{V}^T\tilde{Y}-\tilde{\nu}_k\|_2^2}{2K}}d\tilde{Y}.
\end{align}
We have used the fact $\sum_k\tilde{\nu}_k=0$ in the last equality. In a similar way, we can obtain 
\begin{align*}
E\min_{ k}-\tilde{\nu}_k^TY &=\sum_{k}E[-\tilde{\nu}_k^TY\cdot 1((\tilde{\nu}_k-\tilde{\nu}_j)^TY\geq 0, j \neq k)]  \\
&=\|\tilde{v}_1\|_2^2(2\pi K/\tau^{\overrightarrow{w}}_{\infty})^{(1-K)/2} \int_{\tilde{Y}_i \leq 0, i\in [K-1]} 1^T\tilde{Y}\cdot \sum_{k=1}^Ke^{-\frac{\tau^{\overrightarrow{w}}_{\infty}\|\mathcal{V}^T\tilde{Y}-\tilde{\nu}_k\|_2^2}{2K}}d\tilde{Y}.
\end{align*}
This result combined with \eqref{t:formula}\eqref{gradient:eqc}\eqref{symmetry:use} shows that the derivatives are zero. 
\end{proof}

\begin{lemma}\label{lln}
Define $x_i\equiv \sqrt{n}\mathcal{O}U^Te_i, i=1,2,\ldots, n$, and the Gaussian mixture distribution $z\sim \frac{1}{K}\sum_{k=1}^K\mathcal{N}(\mu^{(k)},\Theta)$, where $\{\mu^{(k)}\}_{k=1}^K$ and $\Theta$ are specified in Proposition \ref{gaussian:embed}. The following convergence results hold:
\begin{itemize}
\item[(i)] For any constant $\delta \in \mathcal{R}$ and any given set $\Xi$ that contains a finite number of points,
\[
\frac{1}{n}\sum_{i=1}^n\min_{\xi\in \Xi}(\|x_i-\xi\|_2- \delta)_+^2 \overset{P}{\rightarrow} E\min_{\xi\in \Xi}(\|z-\xi\|_2-\delta)_+^2.
\]
\item[(ii)] $\frac{1}{n}\sum_{i=1}^n\|x_i\|_2^2\cdot 1(\|x_i\|_2> r) \overset{P}{\rightarrow} E[\|z\|_2^2\cdot 1(\|z\|_2> r)], ~~\forall r \in [-\infty, \infty)$.
\item[(iii)] $\frac{1}{n}\sum_{i=1}^n1(x_i\in \mathcal{B}_r) \overset{P}{\rightarrow} {\rm pr}(z \in \mathcal{B}_r), ~~\forall r \in (0,\infty)$.
\end{itemize}
\end{lemma}
\begin{proof}
We first prove Part (i). Write
\begin{align*}
&\frac{1}{n}\sum_{i=1}^n\min_{\xi\in \Xi}(\|x_i-\xi\|_2- \delta)_+^2=\frac{1}{n}\sum_{i=1}^n\min_{\xi\in \Xi}\Big((\|x_i-\xi\|_2- \delta)^2-(\|x_i-\xi\|_2- \delta)_{-}^2\Big) \\
=&\delta^2+\frac{1}{n}\sum_{i=1}^n\|x_i\|_2^2+\frac{1}{n}\sum_{i=1}^n\min_{\xi\in \Xi}\Big(\underbrace{-2\delta\|x_i-\xi\|_2+\|\xi\|_2^2-2\xi^Tx_i-(\|x_i-\xi\|_2- \delta)_{-}^2}_{\equiv y_i(\xi)}\Big) \\
=&\delta^2+ K+\frac{1}{n}\sum_{i=1}^n\min_{\xi\in \Xi}y_i(\xi),
\end{align*}
where the last equality holds since $U\in \mathcal{R}^{n\times K}$ is the eigenvector matrix. With a similar decomposition for $E\min_{\xi\in \Xi}(\|z-\xi\|_2-\delta)_+^2$, it is direct to verify that (i) is equivalent to 
\[
\frac{1}{n}\sum_{i=1}^n\min_{\xi\in \Xi}y_i(\xi)\overset{P}{\rightarrow} E\min_{\xi\in \Xi}\Big(-2\delta\|z-\xi\|_2+\|\xi\|_2^2-2\xi^Tz-(\|z-\xi\|_2- \delta)_{-}^2\Big).
\]
It is then sufficient to prove for $k\in [K]$
\begin{equation}
\label{target:one}
\frac{K}{n}\sum_{c_i=k}\min_{\xi\in \Xi} y_i(\xi)\overset{P}{\rightarrow} E\min_{\xi\in \Xi}\Big(-2\delta\|z_k-\xi\|_2+\|\xi\|_2^2-2\xi^Tz_k-(\|z_k-\xi\|_2- \delta)_{-}^2\Big),
\end{equation}
where $z_k$ is the $k^{th}$ Gaussian component of $z$. Using the exchangeability of $\{x_i\}_{c_i=k}$, \eqref{target:one} holds if we can show (suppose $c_1=c_2=k$ without loss of generality)
\begin{align*}
&(i1)~ \mbox{var}\big(\min_{\xi\in \Xi} y_1(\xi)\big)=O(1)\quad (i2)~\mbox{Cov}\big(\min_{\xi\in \Xi} y_1(\xi), \min_{\xi\in \Xi} y_2(\xi)\big)=o(1) \\
&(i3)~E\min_{\xi\in \Xi}y_1(\xi)\rightarrow E\min_{\xi\in \Xi}\Big(-2\delta\|z_k-\xi\|_2+\|\xi\|_2^2-2\xi^Tz_k-(\|z_k-\xi\|_2- \delta)_{-}^2\Big).
\end{align*}
Given the convergence result in Proposition 3, $(i1)(i3)$ can be obtained by proving the following uniform integrability
\begin{equation*}
\limsup_{n\rightarrow \infty} E\big[\min_{\xi\in \Xi}y_1(\xi)\big]^2<\infty.
\end{equation*}
Towards this goal, simple use of Jensen's and Cauchy-Schwarz inequality yields 
\begin{equation}
\label{uniform:ing}
\big[\min_{\xi\in \Xi}y_1(\xi)\big]^2 \leq \max_{\xi\in \Xi}[y_1(\xi)]^2 \leq C\Big[\delta^4+(\delta^2+\max_{\xi\in \Xi}\|\xi\|_2^2)\|x_1\|_2^2+\max_{\xi\in \Xi}(\delta^2\|\xi\|_2^2+\|\xi\|_2^4)\Big],
\end{equation}
for some absolute constant $C>0$. It is thus sufficient to show $\limsup_{n\rightarrow \infty}E\|x_1\|^2_2<\infty$. This holds because of the exchangeability of $\{x_i\}_{c_i=k}$:
\[
nK=E\sum_{i=1}^n\|x_i\|_2^2\geq E\sum_{c_i=k}^n\|x_i\|_2^2=\frac{n}{K}E\|x_1\|^2_2.
\]
Regarding $(i2)$, the continuous mapping theorem together with Proposition 3 reveals that $\min_{\xi\in \Xi} y_1(\xi)$ and $\min_{\xi\in \Xi} y_2(\xi)$ are asymptotically independent. Therefore (i2) is proved if the following uniform integrability holds
\[
\limsup_{n\rightarrow \infty} E\big[\min_{\xi\in \Xi}y_1(\xi)\cdot \min_{\xi\in \Xi}y_2(\xi)\big]^2<\infty.
\]
Having \eqref{uniform:ing} enables us to prove the above by showing $\limsup_{n\rightarrow \infty}E[\|x_1\|^2_2\cdot \|x_2\|_2^2]<\infty$ which holds again due to the exchangeability of $\{x_i\}_{c_i=k}$:
\begin{align*}
&E[\|x_1\|^2_2\cdot \|x_2\|_2^2]=\frac{1}{n/K-1}E\Big[\|x_1\|_2^2\cdot \sum_{i:c_i=k,i\neq 1}\|x_i\|_2^2\Big] \\
\leq &\frac{1}{n/K-1}E\Big[\|x_1\|_2^2\cdot \sum_{i=1}^n\|x_i\|_2^2\Big]=\frac{nK}{n/K-1}E\|x_1\|_2^2.
\end{align*}

\vspace{0.3cm}

We now prove Part (ii). The case $r=-\infty$ is trivial since $\frac{1}{n}\sum_{i=1}^n\|x_i\|_2^2=K$. For $r\in (-\infty,\infty)$, it is equivalent to prove 
\[
\frac{1}{n}\sum_{i=1}^n\|x_i\|_2^2\cdot 1(\|x_i\|_2\leq r) \overset{P}{\rightarrow} E[\|z\|_2^2\cdot 1(\|z\|_2\leq r)].
\]
The idea is similar to the proof of (i). The easier part is that the boundedness of $\|x_i\|_2^2\cdot 1(\|x_i\|_2\leq r)$ readily leads to the uniform integrability results. The harder part is that $\|x_i\|_2^2\cdot 1(\|x_i\|_2\leq r)$ is not continuous in $x_i$ so continuous mapping theorem is not directly applicable. To resolve this issue, we construct $g_{\epsilon}=\|x_i\|_2^2\cdot 1(\|x_i\|_2\leq r)+(\frac{-r^2}{\epsilon}\|x_i\|_2+\frac{r^2(r+\epsilon)}{\epsilon})\cdot1(r<\|x_i\|_2\leq r+\epsilon)$ which is bounded and continuous in $x_i$ and approximates $\|x_i\|_2^2\cdot 1(\|x_i\|_2\leq r)$ with arbitrary accuracy for small enough $\epsilon$. The remaining arguments are standard. We skip them for simplicity. 

\vspace{0.3cm}
Part (iii) can be proved in a similar but easier way. This is because the indicator function $1(x_i\in \mathcal{B}_r)$ is bounded, and the convergence for all continuity sets from the portmanteau lemma can be directly applied. We do not repeat the arguments here. 
\end{proof}

\subsection{Proof of Proposition \ref{thm:main2}} \label{sec:prop2}

We first have the following decomposition 
\begin{align*}
\frac{A}{d_n}=\frac{A-E(A)}{d_n}+\frac{E(A)}{d_n}\equiv W+W^*,
\end{align*}
where $d_n$ is defined in \eqref{notation:use}. Throughout this proof, we use $\lambda_j(\cdot)$ to denote the $j^{th}$ largest eigenvalue of a given matrix. 
From the spectral decomposition \eqref{svd:one} it is clear that 
\begin{align}
\label{first:K}
\lambda_1(W^*)\geq d_n,~~~ \lambda_j(W^*)= \sqrt{\tau^{\overrightarrow{w}}_{n}/K}, ~j=2,3,\ldots, K,
\end{align}
and all the other eigenvalues of $W^*$ are zero. Under the conditions 
\begin{align}
\label{conditions:use}
d_n=\Omega(\log^2 n), ~~~\sum_{\ell=1}^Lw^2_{\ell}p^{(\ell)}(1-p^{(\ell)})\approx  \sum_{\ell=1}^Lw^2_{\ell}q^{(\ell)}(1-q^{(\ell)}), 
\end{align}
a direct use of the sharp spectral norm bound (Theorem 1.4 in \citet{vu2005spectral}) shows that almost surely
\begin{align}
\label{spectral:norm:b}
\|W\|_2 \leq 2+o(1).
\end{align}
Therefore, almost surely for all $j=K+1,\ldots, n$
\begin{align}
\label{range:remain}
|\lambda_j(A/d_n)|=|\lambda_j(A/d_n)-\lambda_j(W^*)|\leq \|W\|_2\leq 2 +o(1). 
\end{align}
Moreover, combining \eqref{first:K} and \eqref{spectral:norm:b} it is easy to verify that 
\begin{align}
\label{first:one}
\lambda_1(A/d_n)/\lambda_1(W^*)\overset{a.s.}{\rightarrow} 1.
\end{align}
We now focus on analyzing $\lambda_j(A/d_n), j=2,\ldots, K$. With the same conditions in \eqref{conditions:use}, it is straightforward to confirm (see Proposition 2 and Corollary 3 in \citet{avrachenkov2015spectral}) that the empirical eigenvalue distribution of $W$ converges almost surely weakly to the semi-circle distribution with density $\mu(x)=\frac{\sqrt{4-x^2}}{2\pi}$. Furthermore, by a direct inspection of the proof of Theorem 1.1 in \citet{bai2012limiting}, we can conclude that the asymptotic Haar property continues to hold for the eigenvectors of $W$. Specifically, it holds that $s_n^T(zI-W)^{-1}s_n\overset{a.s.}{\rightarrow}\frac{z-{\rm sign}(z)\sqrt{z^2-4}}{2}\equiv G(z), s_n^T(zI-W)^{-1}t_n\overset{a.s.}{\rightarrow} 0$, where $s_n,t_n$ are non-random unit vectors and $s_n^Tt_n=0$. Having these results, the eigenvalue phase transition property revealed in Theorem 2.1 of \citet{benaych2011eigenvalues} can be readily extended to $\lambda_j(A/d_n), j=K, K+1$. We thus have $\lambda_{K+1}(A/d_n)\overset{a.s.}{\rightarrow} 2$ and 
\begin{align*}
\lambda_K(A/d_n)\overset{a.s.}{\rightarrow} 
\begin{cases}
\lim_{n\rightarrow \infty} G^{-1}(1/\lambda_K(W^*))=\sqrt{\tau^{\overrightarrow{w}}_{\infty}/K}+ \sqrt{K/\tau^{\overrightarrow{w}}_{\infty}} & \text{if}~~\tau^{\overrightarrow{w}}_{\infty}>K \\
2 & \text{otherwise}
\end{cases}
\end{align*}
The above result together with \eqref{range:remain} and \eqref{first:one} implies that the ratio between the $K^{th}$ and $(K+1)^{th}$ largest eigenvalues of $A/d_n$ ordered in magnitude is asymptotically equivalent to $\lambda_K(A/d_n)/\lambda_{K+1}(A/d_n)$. This completes the proof. 

\subsection{Proof of Proposition \ref{gaussian:embed}} \label{sec:prop3}
Define a $K\times K$ orthogonal matrix $S=(s_1,s_2,\ldots, s_K)^T$ of which the first column is $(1/\sqrt{K},\ldots, 1/\sqrt{K})^T$. We further denote $s_i=(1/\sqrt{K},\nu_i^T)^T\in \mathcal{R}^{K}, 1\leq i \leq K$.

For the weighted adjacency matrix $A$, conditioning on the community labels $\overrightarrow{c}$, we have the following spectral decomposition for $E(A)$:
\begin{align}
\label{svd:one}
E(A)=U^*\Lambda^*(U^*)^T, \quad U^* \in \mathcal{R}^{n\times K}, \Lambda^* \in \mathcal{R}^{K\times K},
\end{align}
where
\[
 \Lambda^*=\mbox{diag}\Big(n\sum_{l=1}^Lw_lq^{(l)}+\frac{n}{K}\sum_{l=1}^L w_l(p^{(l)}-q^{(l)}), \frac{n}{K}\sum_{l=1}^L w_l(p^{(l)}-q^{(l)}), \ldots, \frac{n}{K}\sum_{l=1}^L w_l(p^{(l)}-q^{(l)})\Big),
\]
and the $i^{th}$ row of $U^*$ equals to $s_{c_i}^T\sqrt{K/n}$. Define the modularity matrix 
\[
M=A-\Big(\sum_{l=1}^Lw_lq^{(l)}+\sum_{l=1}^L w_l(p^{(l)}-q^{(l)})/K \Big)J_n.
\]
It is straightforward to verify the following spectral decomposition for $E(M)$:
\begin{align}
\label{svd:two}
M^*\equiv E(M)=E^*V^*(E^*)^T, \quad V^*=I_{K-1}\cdot \sum_{l=1}^L w_l(p^{(l)}-q^{(l)})n/K,
\end{align}
where the $i^{th}$ row of $E^*$ is $\nu_{c_i}^T\sqrt{K/n}$. To fix the idea, we first analyze the spectral embedding of $M$. The results for $A$ will follow after some direct modifications. Our proof is highly motivated by the analysis of large Wigner matrices in \citet{bai2012limiting}. Define the following notations that will be constantly used
\begin{align}
\label{notation:use}
&d_n=\sqrt{\frac{n}{K}\sum_{\ell=1}^Lw^2_{\ell}\big(p^{(\ell)}(1-p^{(\ell)})+(K-1)q^{(\ell)}(1-q^{(\ell)})\big)} \nonumber \\
&b_n=\sum_{l=1}^L w_l(p^{(l)}-q^{(l)})n/K, ~~z_n=\frac{d_n}{b_n}+\frac{b_n}{d_n}.
\end{align}

\subsubsection{Spectral embedding of $M$}

Let the $(K-1)$ largest (in magnitude) eigenvalues of $M$ form the diagonal matrix $V \in \mathcal{R}^{(K-1)\times (K-1)}$ and the associated eigenvectors be the columns of $E \in \mathcal{R}^{n\times (K-1)}$. 

\begin{lemma}
\label{embedding:m}
Let the SVD of $E^TE^*$ be $E^TE^*=U_e\Sigma_eV^T_e$. Define the orthogonal matrix $P=U_eV_e^T$. For any pair $i, j\in [n], i \neq j$, conditioning on the community labels $c_i,c_j$, as $n\rightarrow \infty$
\begin{align*}
\begin{pmatrix}
\sqrt{n}P^TE^Te_i \\
\sqrt{n}P^TE^Te_j
\end{pmatrix}
\overset{d}{\rightarrow} \mathcal{N}(\mu, \Sigma), \quad 
\mu=
\begin{pmatrix}
\sqrt{\frac{K(\tau^{\overrightarrow{w}}_{\infty}-K)}{\tau^{\overrightarrow{w}}_{\infty}}}v_{c_i} \\
\sqrt{\frac{K(\tau^{\overrightarrow{w}}_{\infty}-K)}{\tau^{\overrightarrow{w}}_{\infty}}}v_{c_j}
\end{pmatrix}
, \quad \Sigma =\frac{K}{\tau^{\overrightarrow{w}}_{\infty}}\cdot I_{2K-2}.
\end{align*}
\end{lemma}

\begin{proof}
We consider the case when $i=1, j=2$ without loss of generality. We adopt the notions from Lemmas \ref{pair:convergence}. According to the decomposition results in Lemma \ref{pair:convergence}, we aim to establish the joint multivariate Gaussianity of $\sqrt{n}\omega_1^TH_{(12)}^{-1}E^*$ and $\sqrt{n}\omega_2^TH_{(12)}^{-1}E^*$. Introduce the notation
\begin{align*}
H_{(12)}^{-1}E^*=
\begin{pmatrix}
r^T_1 \\
\vdots \\
r^T_n
\end{pmatrix}
, X_1=
\begin{pmatrix}
\sqrt{n}r_2 \omega_{12} \\
\sqrt{n}r_1 \omega_{21}
\end{pmatrix}
, X_{i-1}=
\begin{pmatrix}
\sqrt{n}r_i \omega_{1i} \\
\sqrt{n}r_i \omega_{2i}
\end{pmatrix}
\in \mathcal{R}^{2(K-1)}, i=3,\ldots, n,
\end{align*}
where $\omega_{1i}, \omega_{2i}$ are the $i^{th}$ entries of $\omega_1,\omega_2$ respectively. It is direct to verify that
\begin{align*}
\begin{pmatrix}
\sqrt{n}(E^*)^TH_{(12)}^{-1}\omega_1 \\
\sqrt{n}(E^*)^TH_{(12)}^{-1}\omega_2
\end{pmatrix}
=\sum_{i=1}^{n-1}X_i, 
~~ X_1=o_p(1).
\end{align*}
Thus the goal is to show $\sum_{i=2}^{n-1}X_i\sim \mathcal{N}(0, K/(\tau^{\overrightarrow{w}}_{\infty}-K)I_{2K-2})$. Observe that $\{X_i\}_{i=2}^{n-1}$ are mutually independent conditioning on $H_{(12)}$, and $E(X_i\mid H_{(12)})=0, ~2\leq i \leq n-1$. Define
\begin{align*}
\Sigma_{H_{(12)}}=\sum_{2=1}^{n-1}\mbox{Cov}(X_i\mid H_{(12)}), \quad Z\sim \mathcal{N}(0,I_{2K-2}).
\end{align*}
We shall prove
\begin{itemize}
\item[(a)] The normalized summation $\Sigma^{-1/2}_{H_{(12)}}\sum_{i=2}^{n-1}X_i \overset{d}{\rightarrow } Z$.
\item[(b)] $\Sigma_{H_{(12)}}\overset{a.s.}{\rightarrow} K/(\tau^{\overrightarrow{w}}_{\infty}-K)I_{2K-2}$.
\end{itemize}
Regarding (b), first note that
\begin{align*}
\Sigma_{H_{(12)}}=
\begin{pmatrix}
\Sigma_1 & 0\\
0 & \Sigma_2
\end{pmatrix}
, ~~ \Sigma_1\equiv \sum_{i=3}^{n}nr_ir_i^T\mbox{var}(\omega_{1i}),~~ \Sigma_2 \equiv \sum_{i=3}^{n}nr_ir_i^T\mbox{var}(\omega_{2i}).
\end{align*}
Then (b) holds because
\begin{align}
\label{use:limit}
\Sigma_1 \approx \Sigma_2 \approx (E^*)^TH^{-2}_{(12)}E^* \approx (E^*)^TH^{-2}E^* \overset{a.s.}{\rightarrow} K/(\tau^{\overrightarrow{w}}_{\infty}-K) I_{K-1},
\end{align}
where the convergence is shown in Lemma \ref{com:integral} (i). To obtain (a), according to the multivariate Berry-Esseen theorem (Theorem 1.1 in \citet{raivc2019multivariate}), for all convex sets $Q\in \mathcal{R}^{2K-2}$,
\begin{align}
\label{berry:essen}
Y_n&\equiv \big|{\rm pr}\Big(\Sigma^{-1/2}_{H_{(12)}}\sum_{i=2}^{n-1}X_i\in Q\mid H_{(12)}\Big)-{\rm pr}(Z \in Q)\big| \nonumber \\
&\leq C_K \sum_{i=2}^{n-1}E\Big(\|\Sigma^{-1/2}_{H_{(12)}}X_i\|^3_2\mid H_{(12)}\Big),
\end{align}
where $C_K>0$ is a constant only depending on $K$. Thus for any convex set $Q\in \mathcal{R}^{2K-2}$,
\begin{align*}
&\big|{\rm pr}\Big(\Sigma^{-1/2}_{H_{(12)}}\sum_{i=2}^{n-1}X_i \in Q \Big)-{\rm pr}(Z \in Q)\big| \nonumber \\
\leq& E\big|{\rm pr}\Big(\Sigma^{-1/2}_{H_{(12)}}\sum_{i=2}^{n-1}X_i\in Q\mid H_{(12)}\Big)-{\rm pr}(Z \in Q)\big| =E(Y_n).
\end{align*}
The rest of the proof is to show $E(Y_n)\rightarrow 0$. Since $|Y_n|\leq 2$, it suffices to prove $Y_n\overset{P}{\rightarrow} 0$ which follows if the upper bound in \eqref{berry:essen} converges to zero in probability. It is direct to see that
\begin{align*}
\Sigma_{H_{(12)}}^{-1/2}X_{i-1}=
\begin{pmatrix}
\Sigma_1^{-1/2}\sqrt{n}r_i \omega_{1i} \\
\Sigma_2^{-1/2}\sqrt{n}r_i \omega_{2i}
\end{pmatrix}
,~~i=3,\ldots, n.
\end{align*}
We therefore have
\begin{align*}
&\sum_{i=2}^{n-1}E\Big(\|\Sigma^{-1/2}_{H_{(12)}}X_i\|^3_2\mid H_{(12)}\Big)\\
=&\sum_{i=3}^{n}E\Big((\|\Sigma_1^{-1/2}\sqrt{n}r_i \omega_{1i}\|_2^2+\|\Sigma_2^{-1/2}\sqrt{n}r_i \omega_{2i}\|_2^2)^{3/2}  \mid H_{(12)}\Big)\\
\leq &
\sqrt{2} \sum_{i=3}^{n}\Big[\|\Sigma_1^{-1/2}r_i\|_2^3 \cdot E|\sqrt{n}\omega_{1i}|^3+\|\Sigma_2^{-1/2}r_i \|_2^3\cdot E|\sqrt{n}\omega_{2i}|^3\Big] \\
\lesssim &\frac{\sqrt{n}(\|\Sigma_1^{-1}\|_2^{3/2}+\|\Sigma_2^{-1}\|_2^{3/2})}{d_n} \cdot \sum_{i=3}^n\|r_i\|_2^3
\end{align*}
To show the right-hand side term above is $o_p(1)$, it is sufficient to prove
\begin{align*}
\|\Sigma_1^{-1}\|_2=O_p(1),~~~\|\Sigma_2^{-1}\|_2=O_p(1), ~~~\frac{\sqrt{n}}{d_n}\sum_{i=3}^{n}E\|r_i\|_2^3=o(1).
\end{align*}
The first two results hold from \eqref{use:limit}. To obtain the third one, since the random vectors $\{r_i\}_{i\geq 3: c_i=k}$ are exchangeable for given $k\in [K]$,
\begin{align*}
\frac{\sqrt{n}}{d_n}\sum_{i=3}^{n}E\|r_i\|_2^3 \lesssim \frac{n^{3/2}}{d_n}\sum_{k=1}^K \sum_{1^{st} i:c_i=k} E\|r_i\|_2^3.
\end{align*}
From \eqref{use:limit} we expect to have $\|r_i\|_2\propto n^{-1/2}$. In fact a direct adaptation of the arguments in Theorem 1 of \citet{bai2012limiting} enables us to obtain $\frac{n^{3/2}}{d_n}E\|r_i\|_2^3=O(\frac{1}{d_n})=o(1)$.
\end{proof}

\begin{lemma} 
\label{pair:convergence}
Adopt the notions from Lemma \ref{lemma:two}. Given any pair $i,j \in [n], i \neq j$, let $W_{(ij)}$ be the resulting matrix by replacing the $i^{th}, j^{th}$ rows and columns of $W$ with zeros, and  $H_{(ij)}=W_{(ij)}-z_nI_n$. It holds that as $n\rightarrow \infty$

\begin{align*}
\sqrt{n}e_i^TEP&=\sqrt{\frac{K(\tau^{\overrightarrow{w}}_{\infty}-K)}{\tau^{\overrightarrow{w}}_{\infty}}}v^T_{c_i}-\sqrt{\frac{\tau^{\overrightarrow{w}}_{\infty}-K}{\tau^{\overrightarrow{w}}_{\infty}}}\cdot \sqrt{n}\omega_i^TH_{(ij)}^{-1}E^*+o_p(1) \\
\sqrt{n}e_j^TEP&=\sqrt{\frac{K(\tau^{\overrightarrow{w}}_{\infty}-K)}{\tau^{\overrightarrow{w}}_{\infty}}}v^T_{c_j}-\sqrt{\frac{\tau^{\overrightarrow{w}}_{\infty}-K}{\tau^{\overrightarrow{w}}_{\infty}}}\cdot \sqrt{n}\omega_j^TH_{(ij)}^{-1}E^*+o_p(1),
\end{align*}
where $P$ is the orthogonal matrix defined in Lemma \ref{embedding:m}. 
\end{lemma}

\begin{proof}
We prove the result for $i=1, j=2$. According to Lemma \ref{lemma:two}, 
\begin{align}
&\sqrt{n}e_1^TEP=\sqrt{\frac{K(\tau^{\overrightarrow{w}}_{\infty}-K)}{\tau^{\overrightarrow{w}}_{\infty}}}v^T_{c_1}-\sqrt{\frac{\tau^{\overrightarrow{w}}_{\infty}-K}{\tau^{\overrightarrow{w}}_{\infty}}}\cdot \sqrt{n}\omega_1^TH_{(1)}^{-1}E^*+o_p(1) \label{eq:one}\\
&\sqrt{n}e_2^TEP=\sqrt{\frac{K(\tau^{\overrightarrow{w}}_{\infty}-K)}{\tau^{\overrightarrow{w}}_{\infty}}}v^T_{c_2}-\sqrt{\frac{\tau^{\overrightarrow{w}}_{\infty}-K}{\tau^{\overrightarrow{w}}_{\infty}}}\cdot \sqrt{n}\omega_2^TH_{(2)}^{-1}E^*+o_p(1) \label{eq:two}
\end{align}
We replace the first entry of $\omega_2$ by zero and denote the new vector by $\tilde{\omega}_2$. It is clear that
\[
H_{(1)}=H_{(12)}+e_2\tilde{\omega}^T_2+\tilde{\omega}_2e_2^T
\]
We then have
\begin{align*}
&\sqrt{n}\omega_1^TH^{-1}_{(12)}E^*-\sqrt{n}\omega_1^TH^{-1}_{(1)}E^*=\sqrt{n}\omega_1^TH^{-1}_{(12)}(H_{(1)}-H_{(12)})H^{-1}_{(1)}E^* \\
=& \sqrt{n}\omega_1^TH^{-1}_{(12)}e_2\tilde{\omega}^T_2H^{-1}_{(1)}E^*-\sqrt{n}\omega_1^TH^{-1}_{(12)}\tilde{\omega}_2e_2^TH^{-1}_{(1)}E^*.
\end{align*}
To prove \eqref{eq:one} it is sufficient to show the above difference is $o_p(1)$. Using the mutual independence between $\omega_1, \tilde{\omega}_2, H_{(12)}$ together with Lemma \ref{com:integral} (vi), it is straightforward to confirm that
\begin{align*}
\sqrt{n}\omega_1^TH^{-1}_{(12)}\tilde{\omega}_2=O_p(1), ~~\|e_2^TH^{-1}_{(1)}E^*\|_2=o_p(1),~~\sqrt{n}\omega_1^TH^{-1}_{(12)}e_2=O_p(1)
\end{align*}
It remains to analyze the term $\tilde{\omega}^T_2H^{-1}_{(1)}E^*$. We use the Woodbury formula to obtain
\begin{align*}
H_{(1)}^{-1}\tilde{\omega}_2&=(H_{(12)}+e_2\tilde{\omega}^T_2+\tilde{\omega}_2e_2^T)^{-1}\tilde{\omega}_2=\frac{(H_{(12)}+e_2\tilde{\omega}^T_2)^{-1}\tilde{\omega}_2}{1+e_2^T(H_{(12)}+e_2\tilde{\omega}^T_2)^{-1}\tilde{\omega}_2} \\
&=\frac{H^{-1}_{(12)}\tilde{\omega}_2-H^{-1}_{(12)}e_2\tilde{\omega}_2^TH^{-1}_{(12)}\tilde{\omega}_2}{1+z_n^{-1}\tilde{\omega}_2^TH_{(12)}^{-1}\tilde{\omega}_2}
\end{align*}
Accordingly,
\begin{align*}
\tilde{\omega}^T_2H^{-1}_{(1)}E^*=\frac{\tilde{\omega}_2^TH^{-1}_{(12)}E^*-\tilde{\omega}_2^TH^{-1}_{(12)}\tilde{\omega}_2\cdot e_2^TH^{-1}_{(12)}E^*}{1+z_n^{-1}\tilde{\omega}_2^TH_{(12)}^{-1}\tilde{\omega}_2}=o_p(1).
\end{align*}
The last equality relies on the following results
\[
\tilde{\omega}_2^TH^{-1}_{(12)}E^*=o_p(1), ~~\tilde{\omega}_2^TH^{-1}_{(12)}\tilde{\omega}_2=O_p(1),~~\|e_2^TH^{-1}_{(12)}E^*\|_2=o_p(1),
\]
which hold due to the independence between $\tilde{\omega}_2, H_{(12)}$ and Lemma \ref{com:integral} (vi). We thus have proved \eqref{eq:one}. Replace the second entry of $\omega_1$ by zero and denote the new vector by $\tilde{\omega}_1$. Then 
\[
H_{(2)}=H_{(12)}+e_1\tilde{\omega}^T_1+\tilde{\omega}_1e_1^T
\]
Using the mutual independence between $\omega_2, \tilde{\omega}_1, H_{(12)}$ and the Woodbury formula, \eqref{eq:two} can be proved in a similar way. We do not repeat the arguments. 
\end{proof}

\begin{lemma}
\label{lemma:two}
Denote $W=\frac{M-M^*}{d_n}, H=W-z_n I_n$, and the $i^{th}$ column of $W$ by $\omega_i$. Let $W_{(i)}$ be the resulting matrix by replacing the $i^{th}$ row and column of $W$ with zeros, and $H_{(i)}=W_{(i)}-z_nI_n$. Given any $i\in [n]$, it holds that as $n\rightarrow \infty$
\begin{align*}
\sqrt{n}e_i^TEP&=\sqrt{\frac{K(\tau^{\overrightarrow{w}}_{\infty}-K)}{\tau^{\overrightarrow{w}}_{\infty}}}v^T_{c_i}-\sqrt{\frac{\tau^{\overrightarrow{w}}_{\infty}-K}{\tau^{\overrightarrow{w}}_{\infty}}}\cdot \sqrt{n}\omega_i^TH_{(i)}^{-1}E^*+o_p(1),
\end{align*}
where $P$ is the orthogonal matrix defined in Lemma \ref{embedding:m}. 
\end{lemma}

\begin{proof}
We prove the result for $i=1$. Other cases follow the same arguments. Without loss of generality, we assume the diagonal elements of $W$ are all zero as they relate to the self-edges. The eigenvalue equation $ME=EV$ leads to the following decomposition for $E$:
\begin{align*}
E=(d_nz_n)^{-1}(M-M^*)E+(d_nz_n)^{-1}M^*E+E(I-(d_nz_n)^{-1}V)
\end{align*}
The eigenvector matrix $E$ can be then rewritten as
\begin{align*}
E=&(d_nz_n)^{-1}(I-(d_nz_n)^{-1}(M-M^*))^{-1}M^*E \\
&+(I-(d_nz_n)^{-1}(M-M^*))^{-1}E(I-(d_nz_n)^{-1}V) \equiv E(1)+E(2).
\end{align*}
We first analyze $E(1)$. Using the spectral decomposition $M^*=b_nE^*(E^*)^T$ in \eqref{svd:two} we have
\begin{align*}
E(1)=b_n(d_nz_n)^{-1}\big(I-(d_nz_n)^{-1}(M-M^*)\big)^{-1}E^*\big((E^*)^TE\big) =\frac{-b_n}{d_n}\cdot H^{-1}E^*((E^*)^TE).
\end{align*}
The Woodbury formula enables us to obtain
\begin{align}
\label{woodbury:res}
&\sqrt{n}e_1^TH^{-1}E^*((E^*)^TE) =\sqrt{n}e_1^T(W_{(1)}+e_1w^T_1+\omega_1e_1^T-z_nI)^{-1}E^*((E^*)^TE) \nonumber \\
 =&\frac{\sqrt{n}e_1^T(W_{(1)}-z_nI+e_1\omega_1^T)^{-1}E^*((E^*)^TE)}{1+e_1^T(W_{(1)}-z_nI+e_1\omega_1^T)^{-1}\omega_1} \nonumber \\
=&\frac{\sqrt{n}e_1^T(W_{(1)}-z_nI)^{-1}E^*((E^*)^TE)+\sqrt{n}z_n^{-1}\omega_1^T(W_{(1)}-z_nI)^{-1}E^*((E^*)^TE)}{1+z_n^{-1}\omega_1^T(W_{(1)}-z_nI)^{-1}\omega_1} \nonumber \\
=&\frac{-\sqrt{n}z_n^{-1}e_1^TE^*((E^*)^TE)+\sqrt{n}z_n^{-1}\omega_1^TH_{(1)}^{-1}E^*((E^*)^TE)}{1+z_n^{-1}\omega_1^TH_{(1)}^{-1}\omega_1}.
\end{align}
Here we have used the following identities 
\begin{align*}
&(W_{(1)}-z_nI+e_1\omega_1^T)^{-1}=(W_{(1)}-z_nI)^{-1}-\frac{(W_{(1)}-z_nI)^{-1}e_1\omega_1^T(W_{(1)}-z_nI)^{-1}}{1+\omega_1^T(W_{(1)}-z_nI)^{-1}e_1}, \\
&e_1^T(W_{(1)}-z_nI)^{-1}=-z^{-1}_ne_1^T, ~~\omega_1^T(W_{(1)}-z_nI)^{-1}e_1=0, ~~e_1^T(W_{(1)}-z_nI)^{-1}e_1=-z^{-1}_n.
\end{align*}
Lemma \ref{com:integral} (i)(ii) together with the independence between $H_{(1)}$ and $\omega_1$ implies that 
\begin{align*}
&\sqrt{n}e_1^TE^*((E^*)^TE)=O_p(1), ~~\sqrt{n}\omega_1^TH_{(1)}^{-1}E^*((E^*)^TE)=O_p(1),\\
& (1+z_n^{-1}\omega_1^TH_{(1)}^{-1}\omega_1)^{-1}-z_nd_nb_n^{-1}=o_p(1).
\end{align*}
The above results combined with \eqref{woodbury:res} yield
\begin{align*}
&\sqrt{n}e_1^TH^{-1}E^*((E^*)^TE) \\
=&-\sqrt{n}d_nb_n^{-1}e_1^TE^*((E^*)^TE)+\sqrt{n}d_nb_n^{-1}\omega_1^TH_{(1)}^{-1}E^*((E^*)^TE)+o_p(1).
\end{align*}
Hence we can conclude 
\begin{align*}
\sqrt{n}e_1^TE(1)P =&\sqrt{n}e_1^TE^*((E^*)^TE)U_eV^T_e-\sqrt{n}\omega_1^TH_{(1)}^{-1}E^*((E^*)^TE)U_eV_e^T+o_p(1) \\
=&\sqrt{n}e_1^TE^*V_e\Sigma_eV_e^T-\sqrt{n}\omega_1^TH_{(1)}^{-1}E^*V_e\Sigma_eV_e^T+o_p(1) \\
=&\sqrt{\frac{(\tau^{\overrightarrow{w}}_{\infty}-K)K}{\tau^{\overrightarrow{w}}_{\infty}}}v_{c_1}^T-\sqrt{\frac{\tau^{\overrightarrow{w}}_{\infty}-K}{\tau^{\overrightarrow{w}}_{\infty}}}\sqrt{n}\omega_1^TH_{(1)}^{-1}E^*+o_p(1),
\end{align*}
where in the last equality we have used Lemma \ref{com:integral} (iii). The proof will be completed if we can further show $\sqrt{n}e_1^TE(2)P=o_p(1)$. This is proved in Lemma \ref{small:term}.
\end{proof}

\begin{lemma} \label{small:term}
For each $i\in [n]$, it holds that
\begin{align*}
\sqrt{n}e_i^T(I-(d_nz_n)^{-1}(M-M^*))^{-1}E(I-(d_nz_n)^{-1}V)P=o_p(1).
\end{align*}
\end{lemma}

\begin{proof}
We consider the case $i=1$. We first rewrite the objective term as 
\begin{align*}
&\sqrt{n}e_1^T(I-(d_nz_n)^{-1}(M-M^*))^{-1}EP\cdot \Big(P^T(I-(d_nz_n)^{-1}V)P\Big) \\
=&-z_n\cdot \sqrt{n}e_1^TH^{-1}EP\cdot \Big(P^T(I-(d_nz_n)^{-1}V)P\Big),
\end{align*}
where the notations in Lemma \ref{lemma:two} are adopted. Lemma \ref{com:integral} (iv) implies $\|P^T(I-d_nz_n)^{-1}V)P\|_F=o_p(1)$. Hence we can complete the proof by proving that $\sqrt{n}e_1^TH^{-1}EP=O_p(1)$. Using similar arguments as in \eqref{woodbury:res} we can obtain
\begin{align}
\label{bound:one}
\sqrt{n}e_1^TH^{-1}EP=\frac{-\sqrt{n}z_n^{-1}e_1^TEP+\sqrt{n}z_n^{-1}\omega_1^TH_{(1)}^{-1}EP}{1+z_n^{-1}\omega_1^TH_{(1)}^{-1}\omega_1}
\end{align}
We first use the leave-one-out trick \citep{abbe2020entrywise} to analyze the term $\sqrt{n}\omega_1^TH_{(1)}^{-1}EP$. Let $\tilde{M}\in \mathcal{R}^{n\times n}$ be the resulting matrix by replacing the first row and column of $M$ with zeros. Let $\tilde{V}\in \mathcal{R}^{(K-1)\times (K-1)}$ be the diagonal matrix having the $K-1$ largest (in magnitude) eigenvalues of $\tilde{M}$ on the diagonal, and $\tilde{E}\in \mathcal{R}^{n\times (K-1)}$ be the corresponding eigenvector matrix. Combining the Davis-Kahan theorem and Cauchy interlacing theorem, there exists an orthogonal matrix $\tilde{P}$ depending on $E, \tilde{E}$ such that 
\begin{align*}
\|E-\tilde{E}\tilde{P}\|_F\leq \frac{\sqrt{2}\|(M-\tilde{M})E\|_F}{\min\{\sigma_{K-1}-\sigma_K, \max\{0,\sigma_{K-1}\}\}},
\end{align*}
where $\sigma_{K-1}, \sigma_K$ are the $(K-1)^{th}, K^{th}$ eigenvalues of $M$ respectively. Observe that the matrix $M-\tilde{M}$ only has non-zero entries in the first row and column and those non-zero values are equal to the entries of $M$ on the same locations. Hence, 
\[
e_1^T(M-\tilde{M})E=e_1^TME=e_1^TEV, \quad e_i^T(M-\tilde{M})E=M_{i1}e_1^TE,~~ 2\leq i \leq n,
\]
leading to
\begin{align*}
\|(M-\tilde{M})E\|^2_F\leq  \|e_1^TEV\|^2_2 + \|e_1^TE\|_2^2\sum_{i\neq 1} M^2_{i1} \lesssim d^2_n \|e_1^TE\|_2^2(1+O_p(1)).
\end{align*}
This together with Lemma \ref{com:integral} (v) enables us to obtain
\[
\|E-\tilde{E}\tilde{P}\|_F \lesssim O_p(\|e^T_1E\|_2)
\]
Making use of the independence between $\omega_1$ and $(H_{(1)}, \tilde{E})$, we proceed to derive
\begin{align*}
&\|\sqrt{n}\omega_1^TH_{(1)}^{-1}(E-\tilde{E}\tilde{P})P\|_2\leq \|\sqrt{n}\omega_1^TH_{(1)}^{-1}\|_2 \cdot \|E-\tilde{E}\tilde{P}\|_F \lesssim O_p(\sqrt{n}\|e^T_1E\|_2)\\
& \|\sqrt{n}\omega_1^TH_{(1)}^{-1}\tilde{E}\tilde{P}P\|_2 \leq \|\sqrt{n}\omega_1^TH_{(1)}^{-1}\tilde{E}\|_2 \lesssim {\rm tr}(\tilde{E}^TH_{(1)}^{-2}\tilde{E})\lesssim O_p(1),
\end{align*}
where we have used the fact that $\|H^{-1}_{(1)}\|=O_p(1)$ from Lemma \ref{com:integral} (vi). We therefore can conclude that
\begin{align*}
\|\sqrt{n}\omega_1^TH_{(1)}^{-1}EP\|_2&\leq \|\sqrt{n}\omega_1^TH_{(1)}^{-1}(E-\tilde{E}\tilde{P})P\|_2+\|\sqrt{n}\omega_1^TH_{(1)}^{-1}\tilde{E}\tilde{P}P\|_2 \\
&\lesssim O_p(1+\sqrt{n}\|e^T_1E\|_2).
\end{align*}
This combined with \eqref{bound:one} shows that the proof will be finished if 
\[
\sqrt{n}\|e^T_1E\|_2=O_p(1). 
\]
This holds by utilizing the exchangeability of $\{e_i^TE\}_{c_i=k}$:
\begin{align*}
n(K-1)=E\|\sqrt{n}E\|_F^2=\sum_{k=1}^K\Bigg(\sum_{c_i=k}E\|\sqrt{n}e_i^TE\|^2_2\Bigg)\geq \frac{n}{K}E\|\sqrt{n}e_1^TE\|_2^2.
\end{align*}
\end{proof}

\begin{lemma} 
\label{com:integral}
Adopt the notations from Lemmas \ref{pair:convergence},\ref{lemma:two} and \ref{small:term}. It holds that 
\begin{itemize}
\item[(i)] $(E^*)^TH^{-2}E^* \rightarrow K/(\tau^{\overrightarrow{w}}_{\infty}-K) I_{K-1}$ both almost surely and in expectation.
\item[(ii)] $\omega_i^TH_{(i)}^{-1}\omega_i+\sqrt{K/\tau^{\overrightarrow{w}}_{\infty}}=o_p(1), ~~i\in [n]$.
\item[(iii)] $\Sigma_e \overset{a.s.}{\rightarrow } \sqrt{(\tau^{\overrightarrow{w}}_{\infty}-K)/\tau^{\overrightarrow{w}}_{\infty}}$.
\item[(iv)] $(d_nz_n)^{-1} V\overset{a.s.}{\rightarrow} I_{K-1}$.
\item[(v)] $(\sigma_{K-1}-\sigma_K)/d_n$ and $\sigma_{K-1}/d_n$ are bounded away from zero almost surely
\item[(vi)] $\|H^{-1}_{(1)}\|=O_p(1), |H^{-1}_{(12)}\|=O_p(1), \tilde{\omega}_2^TH^{-1}_{(12)}\tilde{\omega}_2=O_p(1)$
\end{itemize}
\end{lemma}

\begin{proof}
Part (i). We prove the almost sure convergence. Convergence in expectation follows the same arguments. As explained in the proof of Proposition \ref{thm:main2}, Theorem 1 in \cite{bai2012limiting} implies that $(A-E(A))/d_n$ has Haar-like eigenvectors and the empirical distribution $F_n(x)\equiv \sum_{i=1}^n|u_i^Ts_n|^21(\lambda_i\leq x)$ converges almost surely weakly to the semi-circle distribution with density $\mu(x)=\frac{\sqrt{4-x^2}}{2\pi}$, where $\lambda_1\geq \cdots \geq \lambda_n$ are eigenvalues of $(A-E(A))/d_n$, $u_i$'s are the associated orthonormal eigenvectors, and $s_n$ is a unit vector. Moreover, the $z_n$ defined in \eqref{notation:use} converges to $z\equiv \sqrt{\tau^{\overrightarrow{w}}_{\infty}/K}+\sqrt{K/\tau^{\overrightarrow{w}}_{\infty}}$, so $z_n$ is outside the support of $\mu(x)$ for large $n$ when $\tau^{\overrightarrow{w}}_{\infty}>K$. We therefore have
\begin{align}
\label{limit:for:all}
s_n^TH^{-2}s_n \overset{a.s.}{\rightarrow} \int \frac{\mu(x)}{(x-z)^2}dx =\frac{1}{2\pi}\int_{-2}^2 \frac{\sqrt{4-x^2}}{(x-z)^2}dx.
\end{align}
The above integral can be reformulated as a complex integral by setting $x=e^{i\theta}$:
\begin{align*}
\frac{1}{2\pi}\int_{-2}^2\frac{\sqrt{4-x^2}}{(x-z)^2}dx&=\frac{1}{\pi}\int_{0}^{2\pi} \frac{sin^2\theta}{(2cos\theta-z)^2}d\theta=\frac{-1}{4i\pi}\oint_{|y|=1}\frac{(y^2-1)^2}{y(y^2-yz+1)^2}dy \\
&=\frac{-1}{4i\pi}\oint_{|y|=1}\frac{(y^2-1)^2}{y(y-y_1)^2(y-y_2)^2}dy=\frac{K}{(\tau^{\overrightarrow{w}}_{\infty}-K)},
\end{align*}
where $y_1=\sqrt{\tau^{\overrightarrow{w}}_{\infty}/K}, y_2=\sqrt{K/\tau^{\overrightarrow{w}}_{\infty}}$. The last equality is due to Cauchy's residue theorem: 
\begin{align*}
\oint_{|y|=1}\underbrace{\frac{(y^2-1)^2}{y(y-y_1)^2(y-y_2)^2}}_{\equiv f(y)}dy=2\pi i(\mbox{Res}(f, 0)+\mbox{Res}(f, y_2))
\end{align*}
It is straightforward to verify that $\mbox{Res}(f, 0)=\frac{1}{y_1y_2}=1$, and
\begin{align*}
\mbox{Res}(f, y_2)&=\left. \bigg[\frac{(y^2-1)^2}{y(y-y_1)^2}\bigg]' \right \vert_{y=y_2}=\frac{4(y_2^2-1)}{(y_2-y_1)^2}-\frac{(y_2^2-1)^2}{y_2^2(y_2-y_1)^2}
-\frac{2(y_2^2-1)^2}{y_2(y_2-y_1)^3} \\
&=\frac{4K}{K-\tau^{\overrightarrow{w}}_{\infty}}-1-\frac{2K}{K-\tau^{\overrightarrow{w}}_{\infty}}=\frac{K+\tau^{\overrightarrow{w}}_{\infty}}{K-\tau^{\overrightarrow{w}}_{\infty}}
\end{align*}
Denote $E^*=(e_1^*,\ldots, e_{K-1}^*)$. Choosing $s_n=e_i^*$ in \eqref{limit:for:all} yields $(e_i^*)^TH^{-2}e_i^*\overset{a.s.}{\rightarrow} K/(\tau^{\overrightarrow{w}}_{\infty}-K)$. For $1\leq i \neq j\leq K-1$, using \eqref{limit:for:all} for $s_n=e_i^*, e_j^*, 2^{-1/2}(e_i^*+e_j^*)$ leads to 
\begin{align*}
(e_i^*)^TH^{-2}e_j^*=&(2^{-1/2}(e_i^*+e_j^*))^TH^{-2}(2^{-1/2}(e_i^*+e_j^*)) \\
&-(e_i^*)^TH^{-2}e_i^*/2-(e_j^*)^TH^{-2}e_j^*/2 \overset{a.s.}{\rightarrow }0.
\end{align*}

Part (ii). As shown in Section \ref{sec:prop2}, the empirical eigenvalue distribution of $(A-E(A))/d_n$ converges almost surely weakly to the semi-circle distribution with density $\mu(x)=\frac{\sqrt{4-x^2}}{2\pi}$. Hence the stability of empirical spectral distribution (Exercise 2.4.3 in \citet{tao2012topics}) enables us to obtain 
\[
\frac{1}{n}{\rm tr}(H_{(i)}^{-1})\approx \frac{1}{n}{\rm tr}(H^{-1})\overset{a.s.}{\rightarrow} \int\frac{\mu(x)}{x-z}dx =-\sqrt{K/\tau^{\overrightarrow{w}}_{\infty}}.
\]
Given the independence between $H_{(i)}$ and $\omega_i$, we can apply the moment inequality on quadratic forms (Lemma B.26) in \citet{bai2010spectral} to obtain
\begin{align*}
E\Big(\omega_i^TH_{(i)}^{-1}\omega_i-\frac{1}{n}{\rm tr}(H_{(i)}^{-1})\Big)^2 \lesssim \frac{1}{d_n^2} E\frac{1}{n}{\rm tr}(H_{(i)}^{-2})=o(1).
\end{align*}
The last equality holds because $d_n=\Omega(\log^2 n)$, and $E\frac{1}{n}{\rm tr}(H_{(i)}^{-2})\approx E\frac{1}{n}{\rm tr}(H^{-2})=O(1)$. \\

Part (iii). Recall $E^TE^*(E^*)^TE=U_e\Sigma_e^2U_e^T$. The proof is completed if we can show $E^TE^*(E^*)^TE \overset{a.s.}{\rightarrow} (\tau^{\overrightarrow{w}}_{\infty}-K)/\tau^{\overrightarrow{w}}_{\infty}\cdot I_{K-1}$. Towards this goal, denote $E^*=(e_1^*,\ldots, e_{K-1}^*), E=(\epsilon_1,\ldots, \epsilon_{K-1})$. For each $j=1,2,\ldots, K-1$, it is direct to verify the following identity
\begin{align}
\label{start:former}
\epsilon_j=\frac{-b_n}{d_n}\sum_{k=1}^{K-1}\Big(\frac{A-E(A)}{d_n}-\frac{\lambda_j}{d_n}I_n \Big)^{-1}e_k^*(e_k^*)^T\epsilon_j  
\end{align}
Part (iv) shows $\frac{\lambda_j}{d_n}-z_n\overset{a.s.}{\rightarrow} 0$, which combined with the arguments in (i) implies that $(e_j^*)^T(\frac{A-E(A)}{d_n}-\frac{\lambda_j}{d_n}I_n)^{-2}e_k^*\rightarrow K/(\tau^{\overrightarrow{w}}_{\infty}-K)$ if $j=k$ and converges to zero if $j\neq k$. Hence squaring both sides of \eqref{start:former} and letting $n\rightarrow \infty$ yields $e_j^TE^*(E^*)^Te_j\overset{a.s.}{\rightarrow} (\tau^{\overrightarrow{w}}_{\infty}-K)/\tau^{\overrightarrow{w}}_{\infty}$. For $1\leq i \neq j \leq K-1$, using the orthogonality $\epsilon_i^T\epsilon_j=0$ and \eqref{start:former} in a similar way gives $e_i^TE^*(E^*)^Te_j\overset{a.s.}{\rightarrow}0$.

Parts (iv) and (v). It is sufficient to derive the limits of the first $K$ largest eigenvalues (in magnitude) of $M$. An almost identical derivation for eigenvalues of $A$ has been presented in Section \ref{sec:prop2}. We thus do not repeat the arguments.  

Part (vi). According to the spectral norm bound obtained in \eqref{spectral:norm:b}, we have $\|H^{-1}\|_2\leq (z_n-\|(A-E(A))/d_n\|_2)_+^{-1}\lesssim (z-2)^{-1}_+<\infty$ when $\tau^{\overrightarrow{w}}_{\infty}>K$. Note that $H_{(1)}, H_{(12)}$ are obtained by adding small perturbations to $H$. The same spectral norm bound and arguments work for them as well. Lastly, $|\tilde{\omega}_2^TH^{-1}_{(12)}\tilde{\omega}_2|\leq \|H_{(12)}^{-1}\|_2\cdot \|\tilde{\omega}_2\|_2^2=O_p(1)$.

\end{proof}

\subsubsection{Spectral embedding of $A$}

We are in the position to prove the multivariate Gaussian spectral embedding of $A$ as described in Proposition \ref{gaussian:embed}. Let the $K$ largest (in magnitude) eigenvalues of $A$ form the diagonal matrix $\Lambda \in \mathcal{R}^{K\times K}$ and the associated eigenvectors be the columns of $U \in \mathcal{R}^{n\times K}$. Further denote
\[
\Lambda=
\begin{pmatrix}
\lambda_1 & 0 \\
0 & \bar{\Lambda}
\end{pmatrix}
, ~~ U=
\begin{pmatrix}
u_1 & \bar{U}
\end{pmatrix}
,
\]
where $\lambda_1$ is the largest eigenvalue and $u_1$ is the corresponding eigenvector. The result in Proposition \ref{gaussian:embed} is implied by the following two lemmas.
\begin{lemma}
For any given $i\in [n]$, as $n\rightarrow \infty$
\[
{\rm sign}(1^Tu_1)\cdot \sqrt{n}e_i^Tu_1 \overset{P}{\rightarrow } 1.
\]
\end{lemma}
\begin{proof}
Let the largest eigenvalue of $E(A)$ be $\lambda_1^*$ and the associated eigenvector be $u_1^*$. From the spectral decomposition \eqref{svd:one} we know that $\lambda_1^*\propto 
d_n^2$ and $u_1^*=(1/\sqrt{n},\ldots, 1/\sqrt{n})^T$. Denote $\delta={\rm sign}(1^Tu_1)$. We have the following decomposition
\begin{align}
\label{decomp:one}
\sqrt{n}e_i^T\delta u_1-\frac{\sqrt{n}e_i^TAu_1^*}{\lambda^*_1}=\frac{\lambda_1^*-\lambda_1}{\lambda_1}\cdot \frac{\sqrt{n}e_i^TAu_1^*}{\lambda^*_1}+\frac{1}{\lambda_1}\sqrt{n}e_i^TA(\delta u_1-u_1^*).
\end{align}
We first bound $\sqrt{n}e_i^TA(\delta u_1-u_1^*)$. We have
\begin{align*}
\|A(\delta u_1-u_1^*)\|_2 &\leq \|A-E(A)\|_2\cdot \|\delta u_1-u_1^*\|_2+\|E(A)\|_2\cdot  \|\delta u_1-u_1^*\|_2 \lesssim \|A-E(A)\|_2,
\end{align*}
where in the last inequality we have used the variant of David-Kahan theorem (Corollary 1 in \citet{yu2015useful}). With a standard symmetrization trick, we have $E\|A-E(A)\|_2 \leq 2 E\|\tilde{A}\|_2$, where $\tilde{A}_{ij}=A_{ij}\cdot R_{ij}, 1\leq j\leq i \leq n$ with $R_{ij}$'s being independent Rademacher variables. We can then apply Corollary 3.6 in \citet{bandeira2016sharp} to obtain $E\|A-E(A)\|_2 \lesssim d_n$. Further note that $\|A-E(A)\|_2$, as a function of the independent entries in $A-E(A)$, is separately convex and Lipschitz continuous. Thus ${\rm var}(\|A-E(A)\|_2)\lesssim 1$ via Efron-Stein inequality. These results lead to $E\|A(\delta u_1-u_1^*)\|_2^2 \lesssim d^2_n$. This combined with the exchangeability of $\{e_i^TA(\delta u_1-u_1^*)\}_{c_i=k}$ further implies 
\begin{align}
\label{term:two}
E|e_i^TA(\delta u_1-u_1^*)|^2\lesssim d_n^2/n, ~~ \sqrt{n}e_i^TA(\delta u_1-u_1^*)=O_p(d_n).
\end{align}
We next bound $\sqrt{n}e_i^TAu_1^*/\lambda^*_1$. With the independence of entries in $A$, it is straightforward to verify that 
\begin{align}
\label{conver:one}
\sqrt{n}e_i^TAu_1^*/\lambda^*_1\overset{P}{\rightarrow} 1. 
\end{align}
Finally, using Weyl's inequality and the fact $\|A-E(A)\|_2=O_p(d_n)$ it is direct to obtain
\begin{align}
\label{eigen:value}
\lambda_1^*/\lambda_1\overset{P}{\rightarrow} 1.
\end{align}
Putting together \eqref{decomp:one}\eqref{term:two}\eqref{conver:one}\eqref{eigen:value} completes the proof.
\end{proof}

\begin{lemma}
Let the SVD of $\bar{U}^TE^*$ be $\bar{U}^TE^*=U_e\Sigma_eV^T_e$. Define the orthogonal matrix $P=U_eV_e^T$. For any given $i, j\in [n], i\neq j$, conditioning on the community labels $c_i,c_j$, as $n\rightarrow \infty$
\begin{align*}
\begin{pmatrix}
\sqrt{n}P^T\bar{U}^Te_i \\
\sqrt{n}P^T\bar{U}^Te_j
\end{pmatrix}
\overset{d}{\rightarrow} \mathcal{N}(\mu, \Sigma), \quad 
\mu=
\begin{pmatrix}
\sqrt{\frac{K(\tau^{\overrightarrow{w}}_{\infty}-K)}{\tau^{\overrightarrow{w}}_{\infty}}}v_{c_i} \\
\sqrt{\frac{K(\tau^{\overrightarrow{w}}_{\infty}-K)}{\tau^{\overrightarrow{w}}_{\infty}}}v_{c_j}
\end{pmatrix}
, \quad \Sigma =\frac{K}{\tau^{\overrightarrow{w}}_{\infty}}\cdot I_{2K-2}.
\end{align*}
\end{lemma}

\begin{proof}
Lemma \ref{embedding:m} has showed that the same convergence result holds for the spectral embedding of the modularity matrix $M$ . The proof will be a direct modification of the one of Lemma \ref{embedding:m}. Referring to the proof of Lemma \ref{embedding:m}, we need additionally prove the following results
\begin{itemize}
\item[(a)] $\sqrt{n}b_n^{-1}\lambda_1^*e_i^TH^{-1}u_1^*(u_1^*)^T\bar{U}=o_p(1)$.
\item[(b)] $\bar{U}^TE^*(E^*)^T\bar{U}\overset{a.s.}{\rightarrow} (\tau^{\overrightarrow{w}}_{\infty}-K)/\tau^{\overrightarrow{w}}_{\infty}\cdot I_{K-1}$.
\item[(c)] $(d_nz_n)^{-1}\bar{\Lambda}\overset{a.s.}{\rightarrow} I_{K-1}$.
\item[(d)] $\sqrt{n}e_i^T\bar{U}=O_p(1)$.
\end{itemize}
Part (c) can be verified as we did for Lemma \ref{com:integral} (iv). To prove others, denote $E^*=(u_2^*,\ldots, u_K^*), \bar{U}=(u_2,\ldots, u_{K}), \bar{\Lambda}={\rm diag}(\lambda_2,\ldots, \lambda_K)$. The eigenvalue $\lambda_j^*$ associated with $u_j^*$ is $b_n$ for all $j=2,\ldots, K$ according to the spectral decomposition \eqref{svd:one}. Using eigenvalue equations $Au_j=\lambda_j u_j, E(A)u^*_j=b_nu^*_j$, it is straightforward to confirm the following identity for $j=2, \ldots, K$,
\begin{align}
\label{start}
u_j=&\frac{-b_n}{d_n}\sum_{k=2}^{K}\Big(\frac{A-E(A)}{d_n}-\frac{\lambda_j}{d_n}I_n \Big)^{-1}u_k^*(u_k^*)^Tu_j  \nonumber \\
&+\frac{-b_n}{d_n}\Big(\frac{A-E(A)}{d_n}-\frac{\lambda_j}{d_n}I_n \Big)^{-1}u_1^*\Big(\frac{\lambda_1^*}{b_n}(u_1^*)^Tu_j\Big)
\end{align}
Multiplying both sides of \eqref{start} by $u^*_1$ yields
\begin{align*}
\frac{\lambda_1^*}{b_n}(u^*_1)^Tu_j=\frac{\frac{-b_n}{d_n}\sum_{k=2}^{K}(u^*_1)^T\Big(\frac{A-E(A)}{d_n}-\frac{\lambda_j}{d_n}I_n \Big)^{-1}u_k^*(u_k^*)^Tu_j}{\frac{b_n}{\lambda^*_1}+\frac{b_n}{d_n}(u^*_1)^T\Big(\frac{A-E(A)}{d_n}-\frac{\lambda_j}{d_n}I_n\Big)^{-1}u_1^*}
\end{align*}
Like in the proof of Lemma \ref{com:integral} (iii), letting $n\rightarrow \infty$ above gives $b_n^{-1}\lambda_1^*(u_1^*)^Tu_j=o_p(1)$ which proves (a). Similarly multiplying \eqref{start} by $\sqrt{n}e_i^T$ proves (d). In a similar way, using the facts $\|u_j\|_2=1, u_j^Tu_{j'}=0, j \neq j'$ together with \eqref{start} can prove (b).
\end{proof}

\bibliographystyle{JASA}

\bibliography{reference}
\end{document}